\documentclass{article} %
\usepackage{iclr2022_conference,times}

\usepackage{amsmath,amsfonts,bm}

\def\eqref#1{equation~\ref{#1}}

\def\ceil#1{\lceil #1 \rceil}

\def\1{\bm{1}}

\DeclareMathAlphabet{\mathsfit}{\encodingdefault}{\sfdefault}{m}{sl}
\SetMathAlphabet{\mathsfit}{bold}{\encodingdefault}{\sfdefault}{bx}{n}

\def\gB{{\mathcal{B}}}

\def\gD{{\mathcal{D}}}

\def\gS{{\mathcal{S}}}

\def\gX{{\mathcal{X}}}
\def\gY{{\mathcal{Y}}}
\def\gZ{{\mathcal{Z}}}

\DeclareMathOperator*{\argmax}{arg\,max}
\DeclareMathOperator*{\argmin}{arg\,min}

\usepackage{hyperref}
\usepackage{url}
\usepackage{graphicx}
\usepackage{pifont}%
\usepackage{comment}
\newcommand{\cmark}{\ding{51}}%
\newcommand{\xmark}{\ding{55}}%
\usepackage{amsfonts}       %
\usepackage{amsmath}
\usepackage{amssymb}
\usepackage{amsthm}
\usepackage{booktabs}
\usepackage{xcolor}
\usepackage{subcaption}

\title{Addressing Missing Sources with Adversarial Support-Matching}

\author{Thomas Kehrenberg$^1$, Myles Bartlett$^1$, Viktoriia Sharmanska$^1$ \& Novi Quadrianto$^{1,2}$ \\
$^1$Predictive Analytics Lab (PAL), University of Sussex, Brighton, United Kingdom \\
$^2$BCAM Severo Ochoa Strategic Lab on Trustworthy Machine Learning, Bilbao, Spain \\
\texttt{t.kehrenberg@sussex.ac.uk}
}

\newtheorem{theorem}{Proposition}

\theoremstyle{definition}

\iclrfinalcopy %
\begin{document}

\maketitle

\begin{abstract}
When trained on diverse labeled data, machine learning models have proven themselves to be a powerful tool in all facets of society.
However, due to budget limitations, deliberate or non-deliberate censorship, and other problems during data collection and curation,
the labeled training set might exhibit a systematic shortage of data for certain groups.
We investigate a scenario in which the absence of certain data is linked to the second level of a two-level hierarchy in the data.
Inspired by the idea of protected groups from algorithmic fairness, we refer to the partitions carved by this second level as ``subgroups''; we refer to combinations of subgroups and classes, or leaves of the hierarchy, as ``sources''.
To characterize the problem, we introduce the concept of classes with incomplete subgroup support.
The representational bias in the training set can give rise to spurious correlations between the classes and the subgroups which render standard classification models ungeneralizable to unseen sources.
To overcome this bias,
we make use of an additional, diverse but unlabeled dataset, called the  ``deployment set'', to learn a representation that is invariant to subgroup. This is done by adversarially matching the support of the training and deployment sets in representation space.
In order to learn the desired invariance,
it is paramount that the sets of samples observed by the discriminator are balanced by class;
this is easily achieved for the training set, but requires using semi-supervised clustering for the deployment set.
We demonstrate the effectiveness of our method with experiments on several datasets and variants of the problem.
\end{abstract}

\section{Introduction}
\label{sec:intro}
\begin{figure}[tb]
    \centering
    \includegraphics[width=0.9\columnwidth]{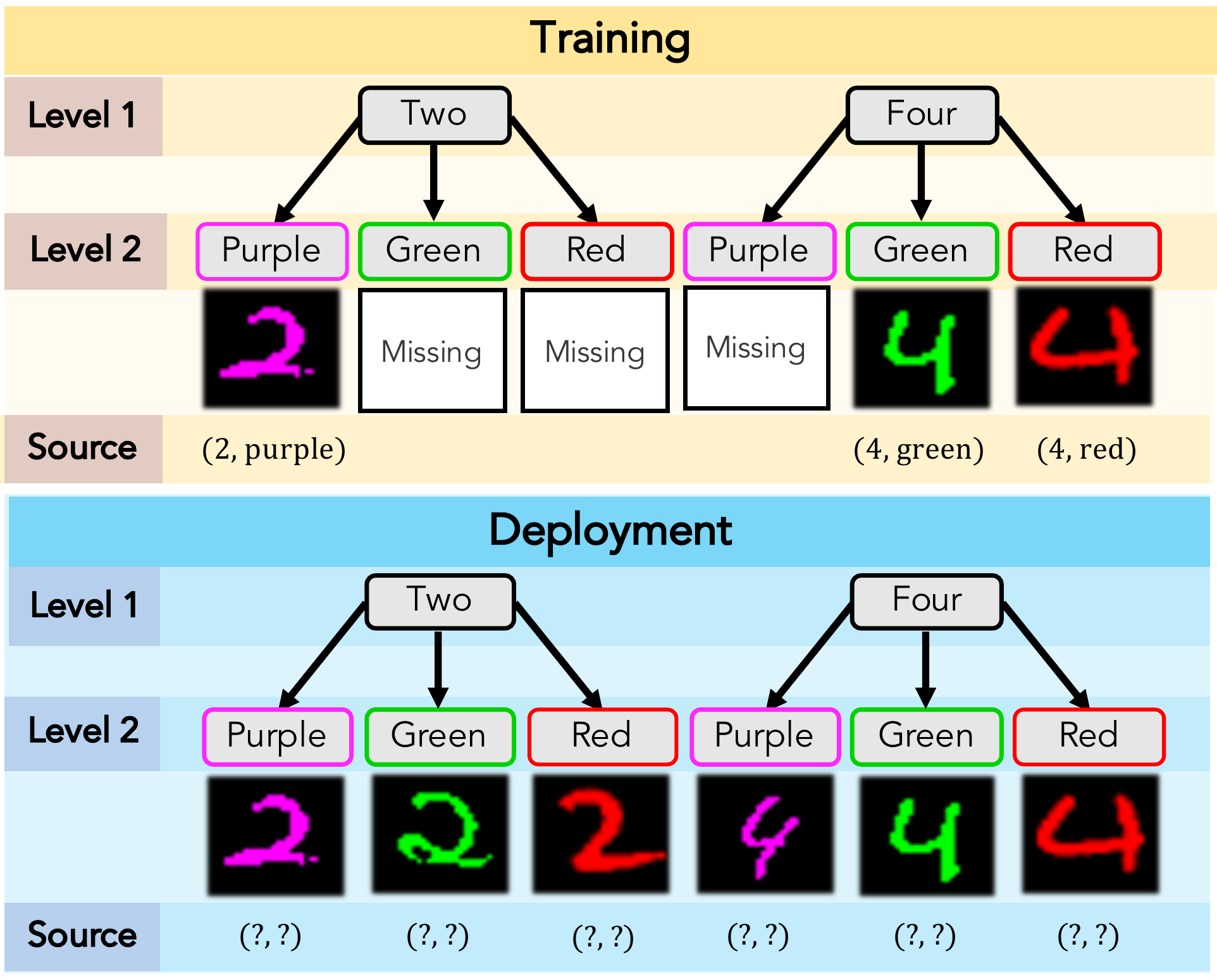}
    \caption{%
      Illustration of our general problem setup. 
      We assume the data follows a two-level hierarchy in which the first level corresponds to the class-level information (digit) and the second level corresponds to subgroup-level information (color).
      While all digits appear in the training set (Top), not all digit-color combinations (sources) do; these gaps in conditional support give rise to a spurious correlation between digit and color, where the former is completely determined by the latter in the training set (giving the mappings $\textrm{{\color{purple}purple}} \rightarrow \texttt{2}$ and $\textrm{{\color{green}green}} \lor \textrm{{\color{red}red}} \rightarrow \texttt{4}$ as degenerate solutions to the classification problem), yet this same correlation does not hold for the deployment set (Bottom) which contains samples from the missing combinations.
      To disentangle the (spurious) subgroup- and class-related information, we make use of an additional dataset that is representative of the data the model is expected to encounter at deployment time, in terms of the sources present.
    }%
    \label{fig:problem-setup}
\end{figure}
Machine learning has burgeoned in the last decade, showing the ability to solve a wide variety of tasks with unprecedented accuracy and efficiency.
These tasks range from image classification \citep{krizhevsky2012imagenet} and object detection \citep{ren2015faster}, to recommender systems \citep{ying2018graph} and the modeling of complex physical systems such as precipitation \citep{ravuri2021skillful} and protein folding \citep{jumper2021highly}.
In the shadow of this success, however, one finds less cause for optimism in frequent failure in equitability and generalization-capability, failure which can have serious repercussions in high-stakes applications such as self-driving cars \citep{sun2019unsupervised}, judicial decision-making \citep{mayson2018bias}, and medical diagnoses \citep{albadawy2018deep}.
ML's data-driven nature is a double-edged sword: while it opens up the ability to learn patterns that are infeasibly complex for a practitioner to encode by hand, the quality of the solutions learned by these models depends primarily on the quality of the data with which they were trained. If the practitioner does not properly account for this, models ingesting data ridden with biases will assimilate, and sometimes even amplify, those biases.
The problem boils down to not having sufficiently diverse annotated data, however collecting more labeled data is not always feasible due to temporal, monetary, legal, regulatory, or physical constraints.

While data can be intrinsically biased (such as in the case of bail records), \emph{representational bias} is more often to blame, where socioeconomic or regulatory factors resulting in certain demographics being under- (or even un-) represented. 
Clinical datasets are particularly problematic for ML due to the frequency of the different outcomes being naturally highly imbalanced, with the number of negative cases (\texttt{healthy}) typically greatly outweighing the number of positive cases (\texttt{diseased}); even if a subgroup is well-represented overall, that may well not be the case when conditioned on the outcome. Equally, it is entirely possible that certain subgroups may be completely absent.
For example, pregnant women are often excluded from clinical trials due to safety concerns, and if they do participate it is often at too low a rate to be meaningful \citep{afrose2021overcoming}.

Like many prior works \citep{SohDunAngGuetal20,kim2019learning,creager2021environment,SagRagKohLia20}, we consider settings where there is a two-level hierarchy, with the second level partitioning the data into \emph{subgroups} that are causally independent of the class (constituting the first level) which is being predicted.
This second level of the data is assumed to be predictable by the classifiers in the considered hypothesis class. In both \citet{SohDunAngGuetal20} and \citet{creager2019flexibly} the entailed subgroups are unobserved and need to be inferred in a semi-supervised fashion. We consider a similar problem but one where the second level is partially observed.
Specifically, we focus on problems where some outcomes are available for some subgroups and not for others. 
This particular form of the problem has -- so far as we are aware -- been hitherto overlooked despite pertaining to a number of real-world problems.

If the labeled training set is sufficiently balanced in terms of classes and subgroups,
a standard ERM (empirical risk minimization) classifier can achieve good performance.
However, we consider the added difficulty that, in the labeled training set,
some outcomes (classes) are not observed for all subgroups,
meaning some of the classes do not overlap with all the subgroups.
In other words, in the training set, some of the classes have \emph{incomplete support} with respect to the subgroup partition, while in the deployment setting we expect all possible combinations of subgroup and class to appear.
We illustrate our problem setup in Fig.~\ref{fig:problem-setup}, using Colored MNIST digits as examples; here, the first level of the hierarchy captures digit class, the second level, color. 
While the (unlabeled) deployment set contains all digit-color combinations (or \emph{sources}), half of these combinations are missing from the (labeled) training set. 
A classifier trained using only this labeled data would wrongly learn to classify \texttt{2}s based on their being {\color{purple}purple} and \texttt{4}s, based on their being {\color{green}green} or {\color{red}red} (instead of based on shape) and when deployed would perform no better than random due to the new sources being colored contrary to their class (relative to the training set).

We address this problem by learning representations that are invariant to subgroups and that thus enable the model to ignore the subgroup partition and to predict only the class labels.
In order to train an encoder capable of producing these representations,
the information contained in the labeled training set alone is not sufficient to break the \emph{spurious correlations}.
To learn the ``correct'' representations, we make use of an additional unlabeled dataset with support equivalent to that of the deployment set (which includes the possibility of it being the actual deployment set).
We do not consider this a significant drawback as such data is almost always far cheaper and less labor-intensive to procure than \emph{labeled} data (which may require expert knowledge).

This additional dataset serves as the inductive bias needed by the encoder to disentangle class- and subgroup-related factors.
The encoder is trained adversarially to produce representations whose source (\texttt{training} or \texttt{deployment}) is indeterminable to a set-classifier.
To ensure subgroup- (not class-) invariance is learned,
the batches fed to the discriminator need to be approximately balanced, such that they reflect the support, and not the shape, of the distributions.
We propose a practical way of achieving this based on semi-supervised clustering.

We empirically show that our proposed method can effectively disentangle subgroup and semantic factors on a range of classification datasets and is robust to noise in the bag-balancing, to the degree of outperforming the baseline methods even when no balancing of bags from the deployment set is performed.
Furthermore, we prove that the entailed objective is theoretically guaranteed to yield representations that are invariant to subgroups and that we can bound the error incurred due to imperfect clustering.

\section{Related work }%
\paragraph{Invariant learning.} \citet{SohDunAngGuetal20} and \citet{creager2021environment} both consider a similar problem,
where the data also exhibits a two-level hierarchy formed by classes and subgroups.
In contrast to our work, however, there is no additional bias in the data; while they may be unobserved, the labeled data is assumed to have complete class-conditional support over the subgroups.
As such, these methods are not directly applicable to the particular form of the problem we consider.
Like us, \citet{SohDunAngGuetal20} uses semi-supervised clustering to uncover the hidden subgroups, however their particular clustering method requires access to the class labels not afforded by the deployment set, as does the training of the robust classifier.

\paragraph{Unsupervised domain adaptation.}
In unsupervised domain adaptation (UDA), there are typically one or more source domains, for which training labels are available, and one or more unlabeled target domains to which we hope to generalize the classifier.
A popular approach for solving this problem is to learn a representation that is invariant to the domain using adversarial networks \citep{ganin2016domain} or non-parametric discrepancy measures such as MMD \citep{gretton2012kernel}.

There are two ways in which one can compare UDA to our setting:
1) by treating the subgroups as domains;
and 2) by treating the training and the deployment set as ``source'' and ``target'' domains, respectively.
The first comparison is exploited in algorithm fairness, yet does not carry over to our setting in which the labeled data contains \emph{incomplete} domains. 
When all sources from a given domain are missing then there are no domains to be matched, and even when this is not the case, matching will result in misalignment due to differences in class-conditional support.
The second comparison is more germane
but ignores an important aspect of our problem: the presence of spurious correlations.

\paragraph{Multiple instance learning.}
Multiple instance learning \citep{maron1998framework} is a form of weakly-supervised learning in which samples are not labeled individually part as part of a set or \emph{bag} of samples.
In the simplest (binary) case, a bag is labeled as positive if there is a single instance of a positive class contained within it, and negative otherwise.
In our case, we can view the missing sources as constituting the positive classes, which leads to all bags (a term we will use throughout the paper distinctly from ``batches'') from the deployment set being labeled as positive, and all bags from the training set being labeled as negative.
Given this labeling scheme, we make use of an adversarial set-classifier to align the supports of the training and deployment sets in the representation space of an encoder network.

\section{Problem setup}\label{sec:problem-setup}
In this section, we formalize the problem of classes with incomplete subgroup-support.
Let $s\in\gS$ refer to discrete-valued subgroup labels with the associated domains $\gS$,
typically given by integers, e.g., $\gS=\{0, 1, 2\}$.
$x$, with the associated domain $\gX$, represents other attributes\slash features of the data.
Let $\gY$ denote the space of class labels for a classification task; $\gY = \{0,1\}$ for binary classification or $\gY = \{1,2,\ldots,C_{\text{cls}}\}$ for multi-class classification.

Let \(\mathcal{S}_{tr}(y=y')\subset\gS\) refer to the set of subgroups that the class \(y'\) has overlap with in the training set.
Thus, a class \(y'\) has full \(s\)-support in the training set if \(\mathcal{S}_{tr}(y=y')=\mathcal{S}\).
As in a standard supervised learning task, we have access to a labeled training set $\gD_{tr} =\{(x_i, s_i, y_i)\}$, that is used to learn a model $M:\gX \rightarrow \gY$. 
The distinguishing aspect of our setup is
that not all classes have full \(s\)-support:
\begin{align}
\exists y'\in \gY: \mathcal{S}_{tr}(y=y')\neq \mathcal{S}~.
\end{align}
For example, we might have \(\mathcal{S}_{tr}(y=1)=\{1\}\),
while \(\mathcal{S}=\{0, 1\}\),
meaning class $y=1$ has no overlap with subgroup $s=0$.
Assuming binary $y$, we are thus observing a one-sided (negative) outcome for the subgroup $s=0$,
giving rise to a setting we refer to as \emph{subgroup bias} (SB).

Once the model $M$ is trained, we deploy it on diverse real-world data.
That is, it will encounter data where the classes have overlap with all subgroups.
If the model relies only on the incomplete training set, it is to be expected that the model will misclassify the subgroups with reduced presence in the training set.
The model becomes biased against those subgroups, leading to unexpectedly poor performance when it is deployed.
We propose to alleviate the issue of subgroup bias by mixing labeled data with unlabeled data that is usually much cheaper to obtain \citep{ChaSchZie06}. 
In this paper, we refer to this set of \emph{unlabeled} data as the
deployment set\footnote{In our experiments, we report accuracy and bias metrics on another independent test set instead of on the unlabeled data that is available at training time.} $\gD_{dep} =\{(x_i)\}$.
This deployment set has full support in all classes:
$ \gS_{dep}(y=y') = \gS ~,~\forall y'\in\gY$
where $\gS_{dep}$ is the analogue of $\gS_{tr}$ for the deployment set.
Importantly, the deployment set comes with neither class nor subgroup labels.

\begin{figure*}[tbp]
  \centering
  \includegraphics[width=0.8\textwidth]{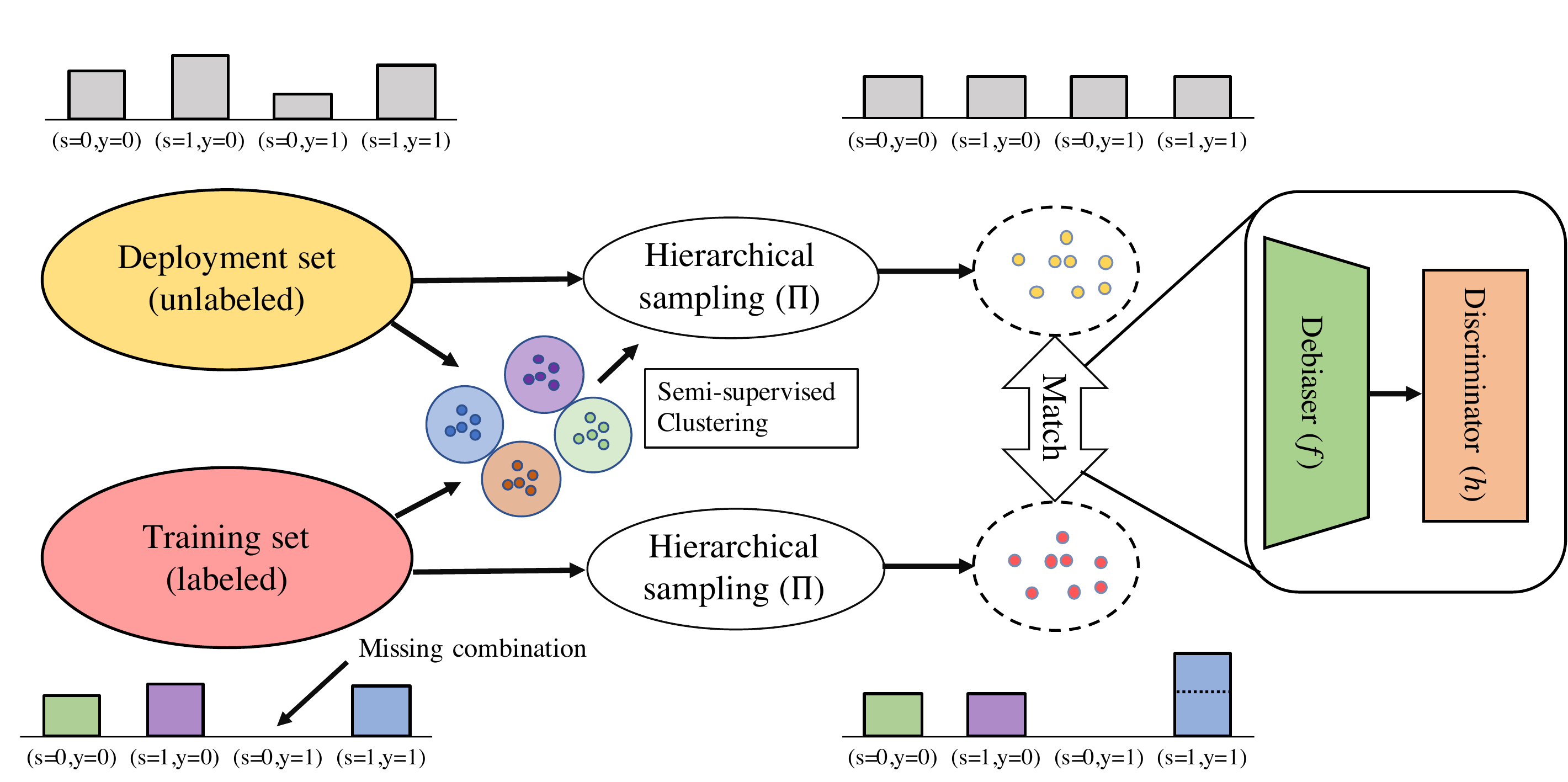}
  \caption{%
    Visualization of our support-matching pipeline.
    Bags are sampled from the training and deployment sets using the hierarchical sampling procedure described in Sec.~\ref{sec:adversarialsm} and defined functionally in Eq.~\ref{eq:functional-sampling}. 
    Since we cannot use ground-truth labels for hierarchical sampling of the deployment set, we use a semi-supervised clustering algorithm to produce balanced batches.
    In the event that certain combinations are missing, as shown here for $(s=0,y=1)$,
    the sampling on the training set substitutes the missing combinations with combinations that ensure equal representation of the target classes. 
    The debiaser is adversarially trained to produce representations from which the source dataset cannot be reliably inferred by the discriminator. 
    Assuming the bags are sufficiently balanced and $S_{tr} \times Y_{tr} \subset S_{dep} \times Y_{dep}$, the optimal debiaser is one that produces a representation $z$ that is invariant to $s$, which we prove in Appendix~\ref{implication-of-the-objective}.
    }%
  \label{fig:pipeline}
\end{figure*}

\section{Adversarial Support-Matching}
\label{sec:adversarialsm}
We cast the problem of learning a subgroup-invariant representation as one of \emph{support-matching} between a dataset that is \emph{labeled} but has \emph{incomplete} support over $\gS \times \gY$, and one, conversely, that has \emph{complete} support over $\gS \times \gY$, but is \emph{unlabeled}.
The idea is to produce a representation that is invariant to this difference in support, and thus invariant to the subgroup.
However, it is easy to learn the wrong invariance if one is not careful.
To measure the discrepancy in support between the two distributions, we adopt an adversarial approach, but one where the adversary is operating on small sets -- which we call \emph{bags} -- instead of individual samples.
These bags need to be balanced with respect to ($s$, $y$),
such that we can interpret them as approximating ($\mathrm{\gS \times \gY}$) as opposed to the joint probability distribution, $P(\gS, \gY)$.
More details on how these bags are constructed will be discussed in Secs. \ref{ssec:realization} and \ref{sec:implementation}.

\subsection{Objective}
We now present our overall objective.
The goal is to learn an encoder $f: \gX\to\gZ$,
which preserves all information relating to $y$,
but is invariant to $s$. %
Let \(P_\mathit{train}(f(x)=z', s=s',y=y')\) be the joint
probability that a data point \(x\) drawn from \(P_\mathit{train}(x)\) -- the training set --
results in the encoding \(z'\) and is at the same time labeled as
subgroup \(s'\) and class \(y'\).
We also define the following shorthand: $p_f(z=z')=P_\mathit{train}(f(x)=z')$,
the distribution resulting from sampling \(x\)
from \(P_\mathit{train}\) and then transforming \(x\) with \(f\).
Analogously for the deployment set: $q_f(z=z')=P_\mathit{dep}(f(x)=z')$.
For the conditioned distributions, we write $p_f|_{s=s',y=y'}$.

The objective makes a distinction between those classes, \(y\in\gY\), for which there is overlap with all subgroups \(s\in\gS\) in the training set and those classes for which there is not.
To formalize this,
we define the following helper function $\Pi$ which maps \((s',y')\)
to a set of subgroup identifiers depending on whether the class \(y\)
has full \(s\)-support:
\begin{align}
\Pi(s',y') = \begin{cases}
\{s'\}&\text{if }\,\mathcal{S}_{tr}(y=y')=\mathcal{S}\\
\mathcal{S}_{tr}(y=y')&\text{otherwise.}
\end{cases}
\end{align}
$\Pi(s,y)$ ensures that the correct invariance is learned and is discussed in more detail further below.
Our objective is then
\begin{align}
\mathcal{L}(f)=\sum\limits_{s'\in\mathcal{S}}\sum\limits_{y'\in\mathcal{Y}} d(p_f|_{s\in \Pi(s',y'),y=y'}, q_f|_{s=s',y=y'})
\label{eq:objective}
\end{align}
where \(d(\cdot, \cdot)\) is a distance measure for probability distributions.
The optimal encoder $f^*$ is found by solving the following optimization problem:
\begin{align}
f^*=\argmin\limits_{f\in\mathcal{F}} \mathcal{L}(f)\quad
\text{s.t.} \quad\mathcal{I}(X;Y)=\mathcal{I}(f(X);Y)
\end{align}
where $\mathcal{I}$ denotes the mutual information between two random variables. As written, Eq.~\ref{eq:objective} requires knowledge of \(s\) and \(y\) on the deployment set for conditioning.
That is why, in practice, the distribution matching is not done separately for all combinations of \(s'\in\gS\) and \(y'\in\gY\).
Instead, we compare \emph{bags} that contain samples from all combinations in the right proportions.
For the deployment set, Eq.~\ref{eq:objective} implies that all \(s\)-\(y\)-combinations have to be present at the same rate in the bags, but for the training set, we need to implement \(\Pi(s',y')\) with hierarchical balancing.

As the implications of the given objective might not be immediately clear,
we provide the following proposition. The proof can be found in Appendix~\ref{sec:theoretical-analysis}.
\begin{theorem}
If \(f\) is such that
\begin{align}
p_f|_{s\in \Pi(s',y'),y=y'} = q_f|_{s=s',y=y'}\quad\forall s'\in\mathcal{S}, y' \in\mathcal{Y}
\end{align}
and \(P_\mathit{train}\) and \(P_\mathit{dep}\) are data
distributions that correspond to the real data distribution \(P\),
except that some \(s\)-\(y\)-combinations are less prevalent, or, in the
case of \(P_\mathit{train}\), missing entirely, then, for every
\(y'\in\mathcal{Y}\), there is either full coverage of \(s\) for \(y'\)
in the training set (\(\mathcal{S}_{tr}(y=y')=\mathcal{S}\)), or the
following holds:
\begin{align}
P(s=s'|f(x)=z', y=y')=\frac{1}{n_s}~.
\end{align}
In other words: for \(y=y'\), \(f(x)\) is not predictive of \(s\).
\end{theorem}

\subsection{Realization}\label{ssec:realization}
The implementation of the objective combines elements from unsupervised representation-learning and adversarial learning.
In addition to the invariant representation $z$,
our model also outputs $\tilde{s}$, in a similar fashion to \citet{KehBarThoQua20} and \citet{creager2019flexibly}. 
This can be understood as a reconstruction of the subgroup information from the input $x$ and is necessary to prevent $z$ from being forced to encode $s$ by the reconstruction loss.
We note that this need could potentially be obviated through use of non-reconstruction-based encoders, such as those based on contrastive learning \citep{chen2020simple}.

The model is made up of four core modules:
1) two \emph{encoder} functions, $f$ (which we refer to as the ``debiaser'') and $t$, which share weights and map $x$ to $z$ and $\tilde{s}$, respectively;
2) a \emph{decoder} function $g$ that learns to invert $f$ and $t$: $g: (z, \tilde{s}) \rightarrow \tilde{x}$;
3) \emph{predictor} functions $\ell_y$ and $\ell_s$ that predict $y$ and $s$ from $z$ and $\tilde{s}$ respectively,
and 4) a \emph{discriminator} function $h$ that predicts which dataset a bag of samples embedded in $z$ was sampled from.
The predictor $\ell_s$ is usually the identity function,
and is primarily listed here for notational symmetry.
The basic idea is that $f$ produces a representation $z$ for which the adversary $h$ cannot tell whether it originated from the training set or deployment set.
Formally, given bags $\gB_{tr}$, sampled according to $\Pi$ from the training set, and balanced bags from the deployment set, $\gB_\mathit{dep}$, we first define, for notational convenience, the loss with respect to the encoder networks, $f$ and $t$ as
\begin{align}
&\mathcal{L}_{\text{enc}}(f, t, h) = \sum_{b \in \{\gB_\mathit{dep}, \mathcal{B}_{tr}\}}\Bigg[ 
   \,\sum_{x \in b} L_{\text{recon}} (x,g(f(x),t(x)))
    \nonumber\\
   &\quad\quad\quad\quad\quad\quad\quad\quad\quad\quad\quad- \lambda_\mathit{adv} \log h(\{f(x)\ | x \in b\})\Bigg] \nonumber\\
   &\quad+ \!\!\sum_{x\in \gB_{tr}}
   \lambda_y L_{\text{sup}} (
   y, \ell_y(f(x))) + \lambda_s L_{\text{sup}} (s, \ell_s(t(x)))
\label{eq:disentangling}
\end{align}
where $L_{\text{recon}}$ and $L_{\text{sup}}$ denote the reconstruction loss, and supervised loss, respectively, and $\lambda_y$, $\lambda_s$ and $\lambda_\mathit{adv}$ are positive pre-factors.
The overall objective, encompassing $f$, $t$, and $h$ can then be formulated in terms of $\mathcal{L}_{\text{enc}}$ as
\begin{align}
    \underset{f, t}{\textrm{min}}\; \underset{h}{\textrm{max}}\,\mathcal{L}_{enc}(f, t, h)~.
    \label{eq:disentangling_total}
\end{align}

Aside from being defined bag-wise, our adversarial loss is unusual in the respect that both of its constituent terms are dependent on $f$ ($f$ generates both the ``real'' and the ``fake'' samples).
We allow the gradient to flow through both $f(x) \in \gB_{tr}$ and $f(x) \in \gB_{dep}$, finding that adding a stop-gradient to $\log h(\{f(x) | x \in \gB_{dep}\})$ has a negative effect on the stability and convergence rate of training.

Eq.~(\ref{eq:disentangling_total}) is computed over batches of bags and the discriminator is trained to map a bag of samples from the training set and the deployment set to a binary label:
$1$ if the bag is judged to have been sampled from the deployment set, $0$ if from the training set.
For the discriminator to be able to classify sets of samples, it needs to be permutation-invariant along the bag dimension -- that is, its predictions should take into account dependencies between samples in a bag while being invariant to the order in which they appear.
We experiment with two different types of attention mechanism for the bag-wise pooling layer of our discriminator, finding them both to work well.
For more details see Appendix~\ref{ssec:attention-mechanism}.
Furthermore, in Appendix~\ref{ssec:no-mil},
we validate that using sets (``bags'') as input to the discriminator improves performance compared to using individual samples.

Our goal is to disentangle $x$ into a part $z$, representing the class, and a part $\tilde{s}$, representing the subgroup,
but for this to be well-posed, it is crucial that the bags differ only in terms of which sources are present and not in terms of other aspects.
We thus sample the bags according to the following set of rules which operationalize $\Pi$.
Please refer to Fig.~\ref{fig:pipeline} for a visualization of the effect of these rules.
1) Bags of the deployment set are sampled so as to be approximately balanced with respect to $s$ and $y$ (all combinations of $s$ and $y$ should appear in equal number).
2) For bags from the training set, all possible values of $y$ should appear with equal frequency. 
Without this constraint, there is the risk of $y$ being encoded in $\tilde{s}$ instead of $s$.
3) Bags of the training set should furthermore exhibit equal representation of each subgroup within classes so long as 2) would not be violated. 
For classes that do not have complete $s$-support, the missing combinations of $(s, y)$ need to be substituted with a sample from the same class -- i.e., if $s \notin \mathcal{S}_{tr}(y)$ we instead sample randomly from a uniform distribution over $\mathcal{S}_{tr}(y$).

We supplement the implicit constraints carried by the balancing of the bags with the explicit constraint that $z$ be predictive of $y$, which we achieve using a linear predictor $l_y$. Whenever we have $\textrm{dim}(\mathcal{S}_{tr}) > 1$, %
we can also impose the same constraint on $\tilde{s}$, but with respect to $s$.

\subsection{Implementation}\label{sec:implementation}
A visual overview of our pipeline is given in Fig.~\ref{fig:pipeline}.
Borrowing from the literature on algorithmic fairness \citep{chouldechova17,KleMulRag16}, we refer to a bag in which all elements of $\gY \times \gS$ appear in equal proportions as a ``perfect bag'' (even if the balancing is only approximate).
Our pipeline entails two steps: 1) sample perfect bags from an unlabeled deployment set, and 2) produce disentangled representations using the perfect bags via adversarial support-matching as described in Sec.~\ref{ssec:realization}.

\begin{figure}[tbp]
  \centering
   \includegraphics[width=1.0\columnwidth]{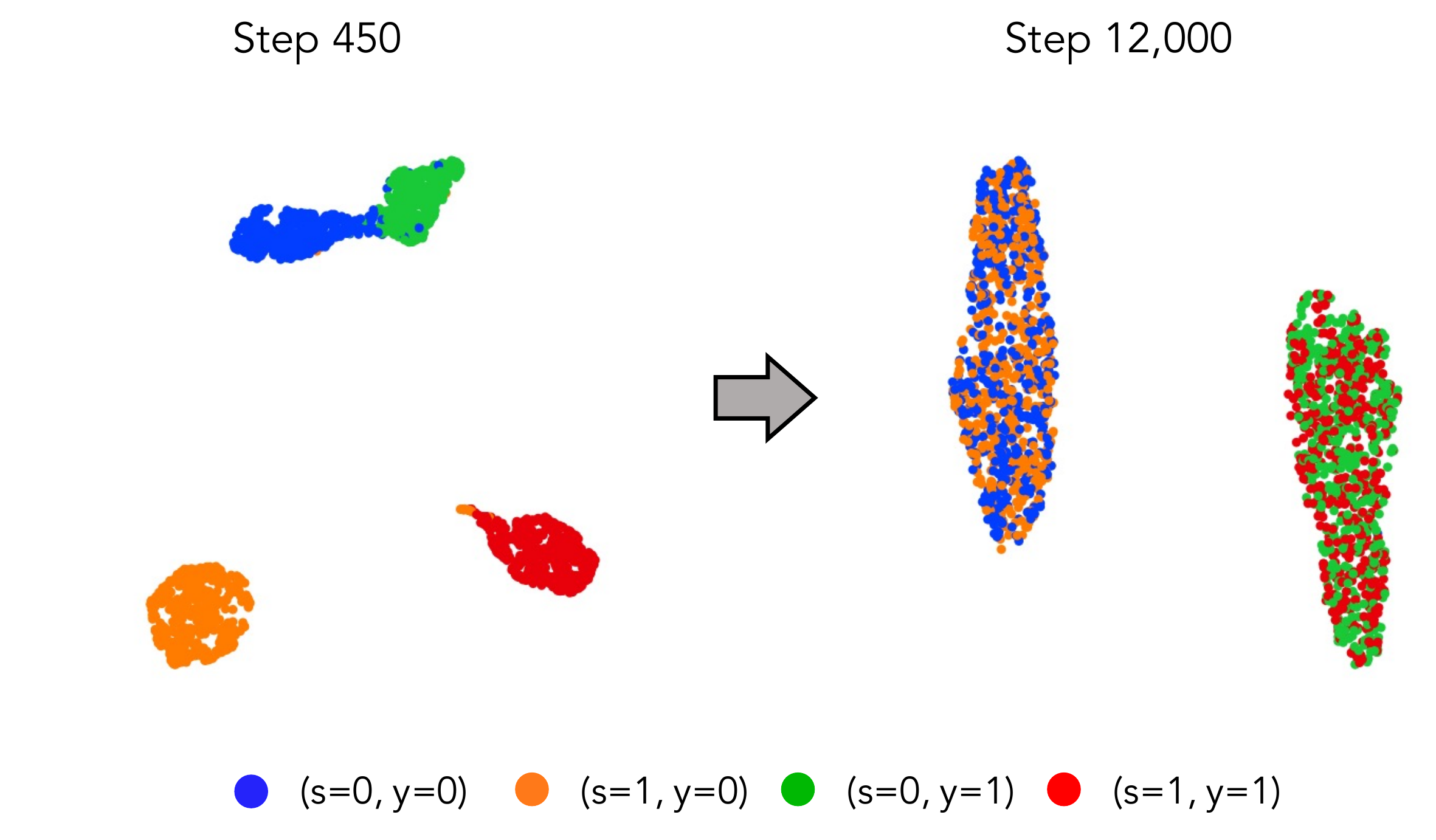}
  \caption{%
    UMAP visualizations of the representations learned by the debiaser for the Colored MNIST dataset. \textbf{Left}: After 450 training-steps each source forms a distinct cluster. \textbf{Right}: After 12,000 training-steps sources with the same $y$-value have merged, thus eliminating the spurious correlation between digit and color.
  }%
  \label{fig:umap}
\end{figure}

\textbf{Constructing perfect bags via clustering.}
We cluster the data points from the deployment set into $K=\textrm{dim}(\mathcal{Y})\cdot \textrm{dim}(\mathcal{S})$ number of clusters (i.e. the number of $s$-$y$-combinations) using a  recently-proposed supervised clustering method based on rank statistics \citep{HanRebEhrVedetal20}.
The cluster assignments can then be used to evenly stratify the deployment set into perfect bags, to be used by the subsequent disentangling phase.

As a result of clustering, the data points in the deployment set $\gD_\mathit{dep}$ are labeled with cluster assignments 
$\gD_\mathit{dep}=\{(x_i, c_i)\}$, $c_i = C(z_i)$,
so that we can form perfect bags from $\gD_\mathit{dep}$ by sampling all clusters at equal rates.
We do \emph{not} have to associate the clusters with specific $s$ or $y$ labels.

Balancing bags based on clusters instead of the true labels introduces an error, which we can try to bound.
For this error-bounding, we assume that the probability distribution distance measure used in Eq.~\ref{eq:objective} is the \emph{total variation distance} $TV$.
The proof can be found in Appendix~\ref{sec:theoretical-analysis}.
\begin{theorem}
If \(q_f(z)\) is a data distribution on
\(\mathcal{Z}\) that is a mixture of \(n_y\cdot n_s\) Gaussians, which
correspond to all the unique combinations of \(y\in\mathcal{Y}\) and
\(s\in\mathcal{S}\), and \(p_f(z)\) is any data distribution
on \(\mathcal{Z}\), then without knowing \(y\) and \(s\) on \(q_f\), it is possible to estimate
\begin{align}
\sum\limits_{s'\in\mathcal{S}}\sum\limits_{y'\in\mathcal{Y}} TV(p_f|_{s\in g(s',y'),y=y'}, q_f|_{s=s',y=y'})
\end{align}
with an error that is bounded by \(\tilde{O}(\sqrt{1/N})\) with high
probability, where \(N\) is the number of samples drawn from \(q_f\) for
learning.
\end{theorem}

\subsection{Limitation and intended use}
\label{sec:limitations}
Although having zero labeled examples for some subgroups is not uncommon due to the effects of systematic bias or dataset curation, we should make a value-judgment on the efficacy of the dataset with respect to a task.
We can then decide whether or not to take corrective action as described in this paper.
A limitation of the presented approach is that, for constructing the perfect bags used to train the disentangling algorithm, we have relied on knowing the number of clusters \emph{a priori}, something that, in practice, is perhaps not the case.
However, for person-related data, such information can, for example, be gleaned from recent census data.
(see also Appendix~\ref{sec:overclustering}
for results with misspecified numbers of clusters.)
One difficulty with automatic determination of the number of clusters is the need to ensure that the small
clusters are correctly identified. 
A cluster formed by an underrepresented subgroup can be easily overlooked by a clustering algorithm in favor of larger but less meaningful clusters.

\begin{figure*}[t]
  \centering
  \begin{subfigure}[b]{\textwidth}
  \includegraphics[width=0.49\textwidth]{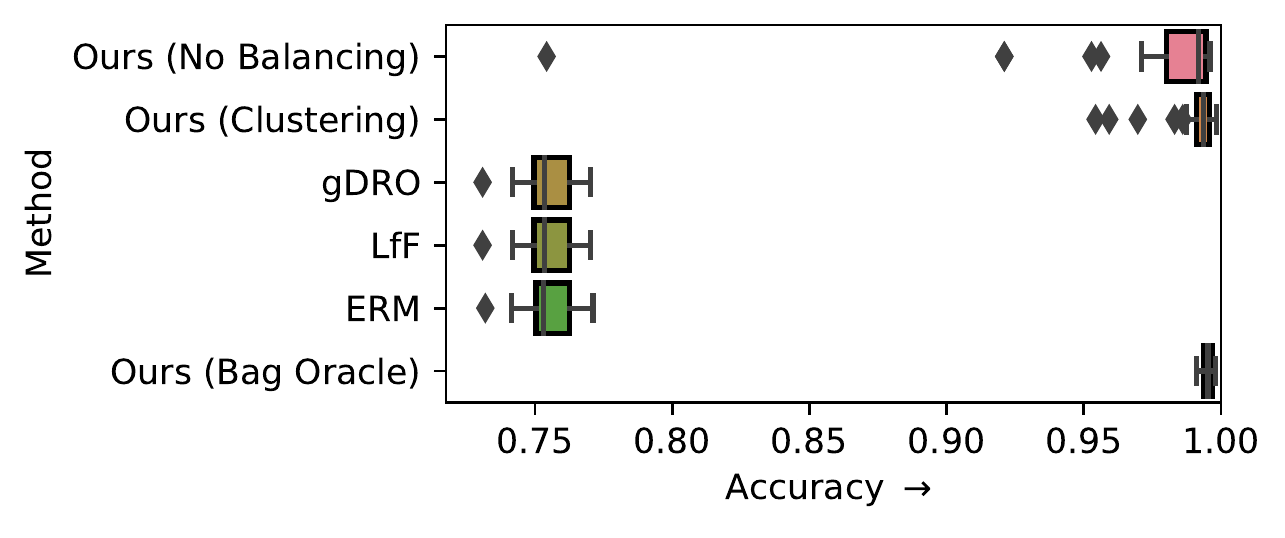}
  \includegraphics[width=0.49\textwidth]{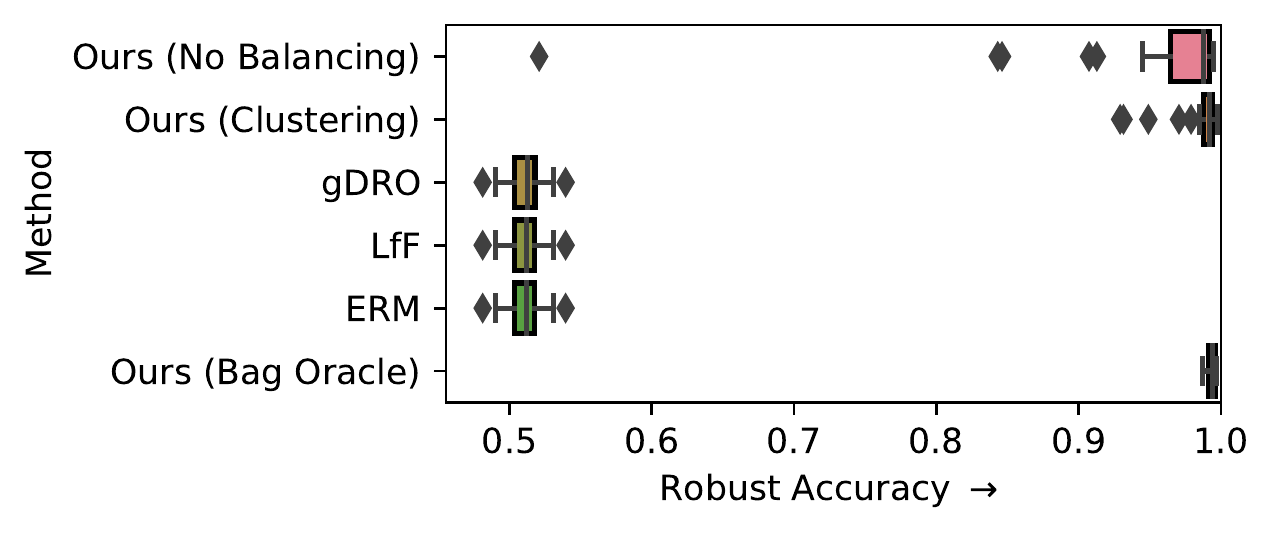}
  \caption{
    Results for the \emph{subgroup-bias} scenario where {\color{purple}purple} fours constitute the missing source.
    The clustering accuracy for \texttt{Ours (No Balancing)} was 96\% $\pm$ 6\%.
    Our method consistently outperforms the baselines, which fare no better than random on the subgroup with the missing source. 
    As we would expect, the median and IQR of our method are positively- and negative- correlated, respectively, with how well the bags of the deployment set are balanced, with \texttt{Ours (Bag Oracle)} providing an upper bound for this.
    Indeed, in one case \texttt{Ours (No Clustering)} failed to surpass the baselines, but through use of clustering the \texttt{Robust Accuracy} is kept within the region of $95\%$ at worst.
  }%
  \label{fig:cmnist-2v4-partial}
  \end{subfigure}
  
  \begin{subfigure}[b]{\textwidth}
  \centering
  \includegraphics[width=0.49\textwidth]{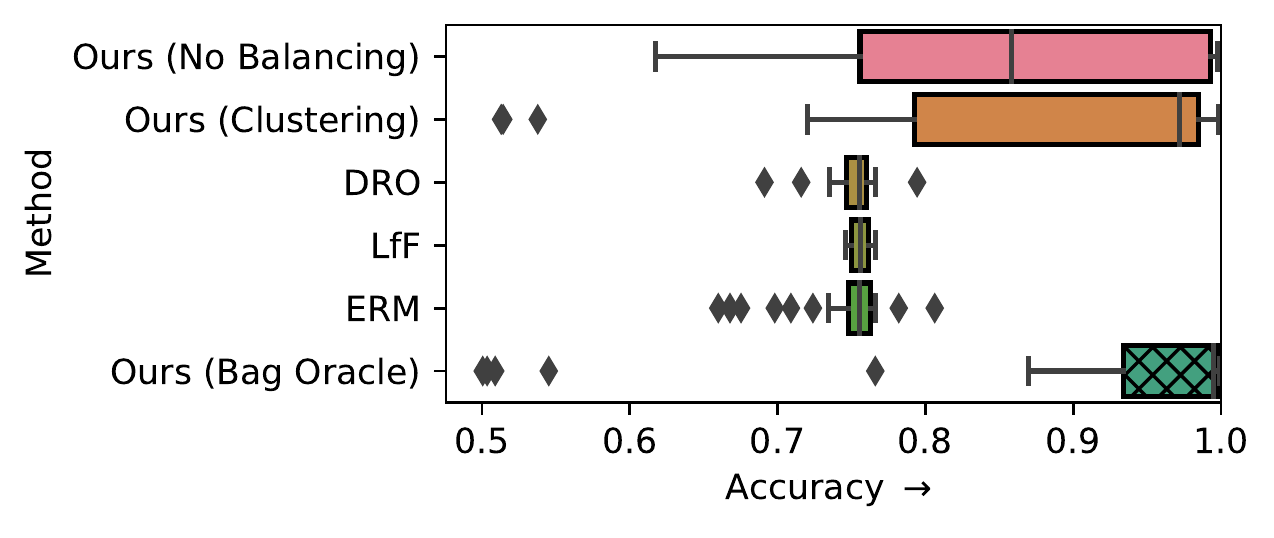}
  \includegraphics[width=0.49\textwidth]{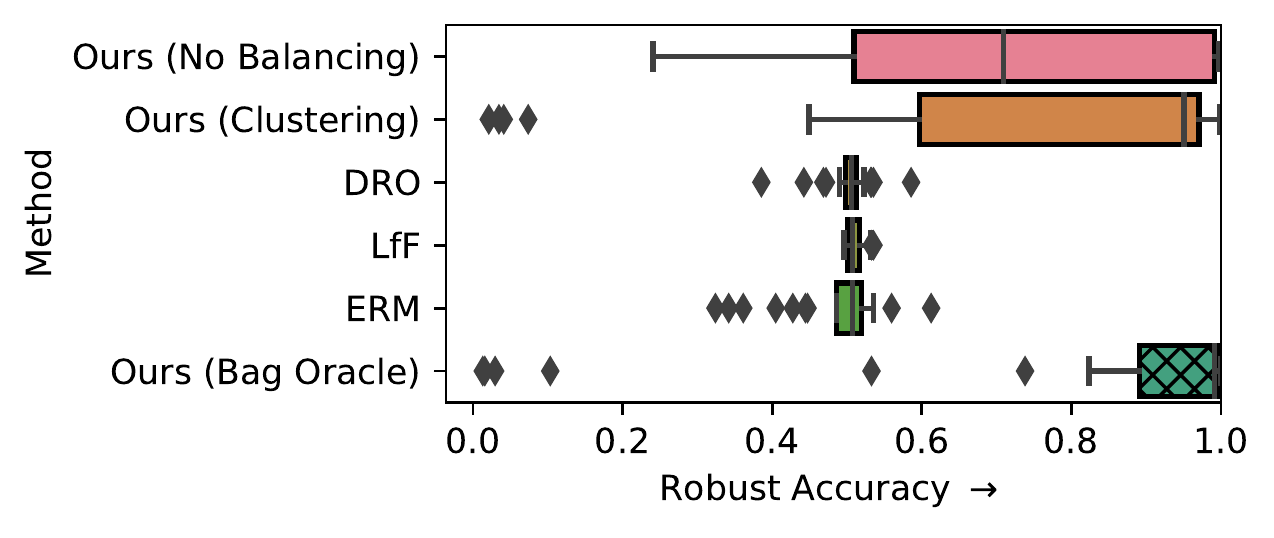}
  \caption{
    Results for the \emph{missing-subgroup scenario} where {\color{purple}purple} digits constitute the missing subgroup.
    The clustering accuracy for \texttt{Our (No Balancing)} was 88\% $\pm$ 5\%. 
    This scenario is significantly more difficult to solve than the subgroup-bias as there is insufficient inductive bias in the labels and the deployment set for the support matching to be well-posed. 
    This is reflected in the high variance of our method, variance, however, which can be drastically reduced by improving the quality of balancing.
    Nevertheless, all variants of our method perform significantly better than the baselines in terms of the median \texttt{Robust Accuracy}, and the rate at which they produce degenerate solutions (marked by performance worse than \texttt{ERM}'s) relatively low.
  }%
  \label{fig:cmnist-2v4-miss-s}
  \end{subfigure}
  \caption{
  Results for two-digit Colored MNIST for two different scenarios (subgroup bias (Top) and missing subgroup (Bottom)) in the form of box plots of the \texttt{Robust Accuracy} (the minimum accuracy computed over the subgroups) over \textbf{30 repeats}.
  }
\end{figure*}

\section{Experiments}
We perform experiments on a combination of image and tabular datasets that are publicly available -- Colored MNIST \citep{arjovsky2019invariant}, CelebA \citep{liu2015celeba} and Adult Income \citep{Dua:2017} (the latter with results in Appendix~\ref{sec:adult-results}).
For all experiments, we report the \texttt{Robust Accuracy}, defined as the minimum accuracy over the subgroups. Additional plots showing the Accuracy, Positive Rate, TPR, and TNR ratios can be found in Appendix~\ref{sec:additional-metrics}.

We compare the performance of our disentangling model when paired with each of three different bag balancing methods:
1) with clustering via rank statistics (\texttt{Clustering});
2) without balancing, when the deployment set $\gD_{dep}$ is used as is (\texttt{No Balancing});
3) with balancing using the ground-truth class and subgroup labels (\texttt{Oracle Bag}) that would in practice be unobservable; this provides insight into the performance under ideal conditions and how sensitive the method is to bag imbalance.
 
\textbf{Colored MNIST.}
Appendix~\ref{sec:dataset-construction}
provides description of the dataset and the settings used for $D_{dep}$ and $D_{tr}$. Each source is then a combination of digit-class (class label) and color (subgroup label). 
We begin by considering a binary, 2-digit/2-color, variant of the dataset with $\mathcal{Y} = \{2, 4\}$ and $\mathcal{S} = \{\text{\color{green}green}, \text{\color{purple}purple}\}$.
(Appendix~\ref{ssec:3-digit-3-color}
provides results for 3-digit/3-color variant.)
For this variant we explore both the SB (subgroup bias) setting and a more extreme \emph{missing subgroup} setting.
To simulate the SB setting, we set \(\mathcal{S}_{tr}(y=4)=\{\text{\color{green}{green}}\}\).
In the \emph{missing subgroup} setting,
$s=\text{\color{purple}{purple}}$ is missing from \(\mathcal{S}_{tr}(y=2)\) as well,
so that all classes only have support in \(\{\text{\color{green}{green}}\}\).
However, for this scenario, the disentangling procedure has more than one possible solution --
apart from the natural solution, it is also possible to consider ($y=2$, $s=\text{\color{green}{green}}$) and ($y=4$, $s=\text{\color{purple}{purple}}$)
as forming one factor in the disentangling,
with the other factor comprising the two remaining $s$-$y$-combinations.
Such an ``unnatural'' disentangling (spanning digit class \emph{and} color) is avoided only by the tendency of neural networks to prefer simpler solutions (Occam's razor) and in general we cannot guarantee that this pathological case be avoided based only on the information provided by the training labels and deployment set.

To establish the effectiveness of our method, we compare with four baselines.
The first is \texttt{ERM}, a classifier trained with cross-entropy loss on this data;
the second is \texttt{DRO} \citep{HasSriNamLia18}, which functions without subgroup labels by minimizing the worst-case training loss over all possible groups that are above a certain minimum size; the third is \texttt{gDRO} \citep{sagawa2019distributionally}, which minimizes the worst-case training loss over predefined subgroups but is only applicable when $\text{dim}(\mathcal{S}_{tr}) > 1$; the fourth is \texttt{LfF} \citep{NamChaAhnLeeetal20} which reweights the cross-entropy loss using the predictions of a purposely biased sister network.
For fair comparison, the training set is balanced according to the rules defined in Sec.~\ref{sec:adversarialsm} for all baselines.

Fig.~\ref{fig:cmnist-2v4-partial} shows the results for the SB setting. 
We see that the performance of our method directly correlates with how balanced the bags are, with the ranking of the different balancing methods being \texttt{Oracle} $>$ \texttt{Clustering}$>$ \texttt{No Balancing}. 
Even without balancing, our method greatly outperforms the baselines, which all perform similarly.

Fig.~\ref{fig:cmnist-2v4-miss-s} shows that the problem of \emph{missing subgroups} is harder to solve.
For all balancing strategies, the IQR is significantly higher than observed in the SB setting, with the latter also giving rise to a large number of extreme outliers. 
The median, however, remains high, reflecting a ``hit-or-miss'' performance of the method, but with the number of hits far outweighing the misses. 
Visualizations of the reconstructions (Appendix~\ref{sec:qual-results})
suggest that the extreme outliers correspond to the degenerate solution mentioned above.

We visualize the learned representation -- from a successful run -- using UMAP \citep{mcinnes2018umap} in Fig.~\ref{fig:umap}.
Here, we see that at the beginning of training,
all four sources are distinct, and the two sources with $s=0$
(from which one is missing in the training set)
are closer to each other than to their respective classes.
At the end of training,
the representations clearly separate into two clusters corresponding to the two classes,
while the subgroups are distributed evenly therein.

\begin{figure*}[t]
  \centering
 \includegraphics[width=0.49\textwidth]{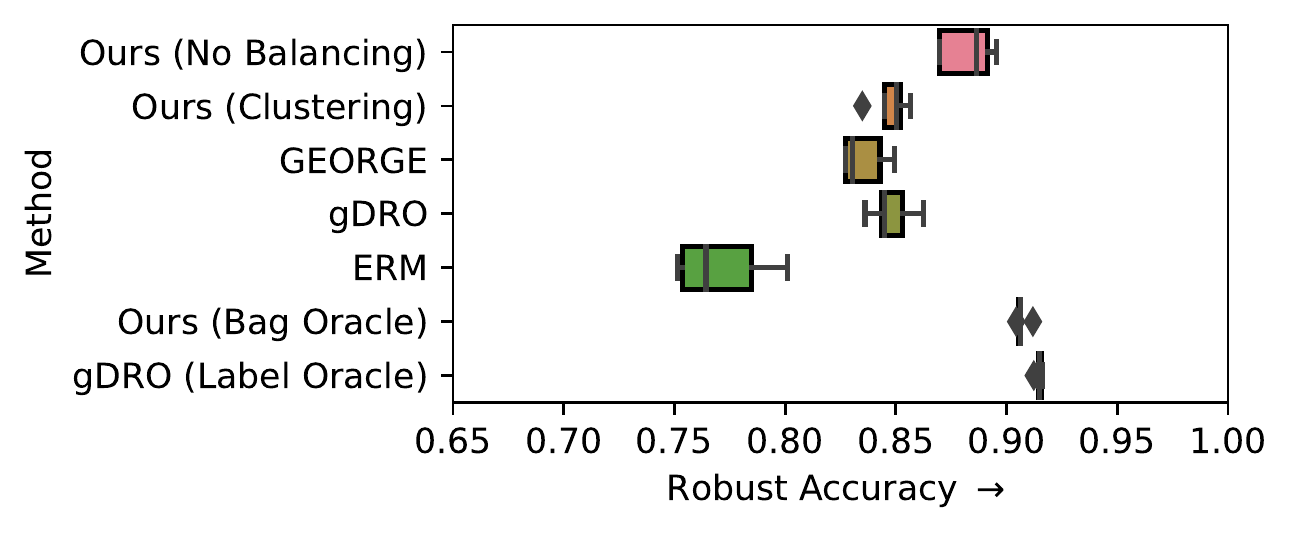}
 \includegraphics[width=0.49\textwidth]{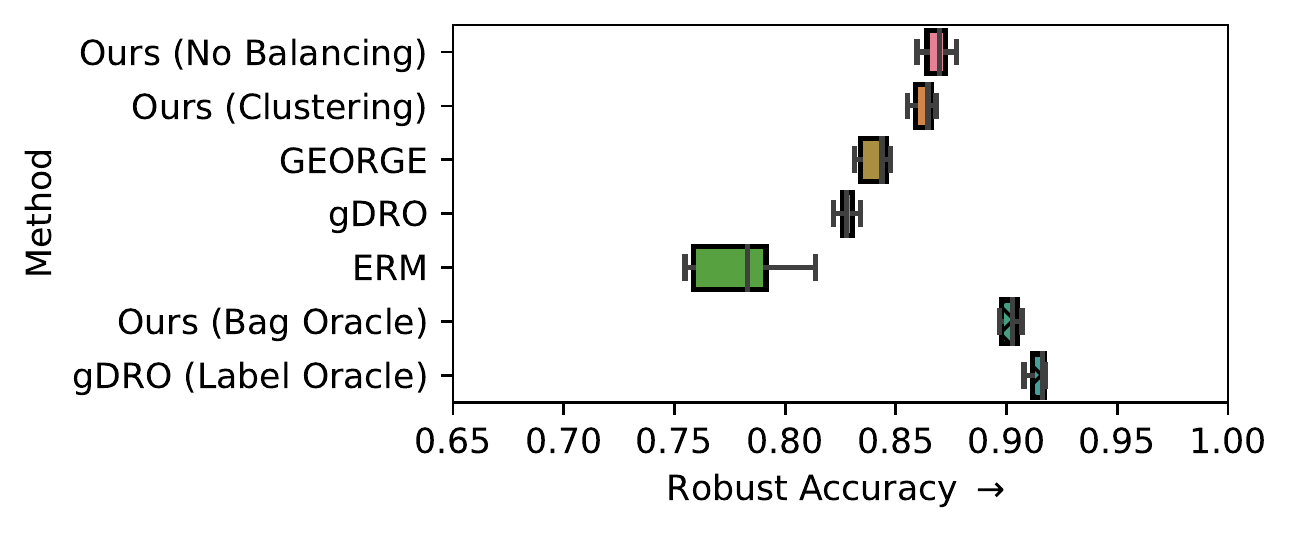}
 \includegraphics[width=0.49\textwidth]{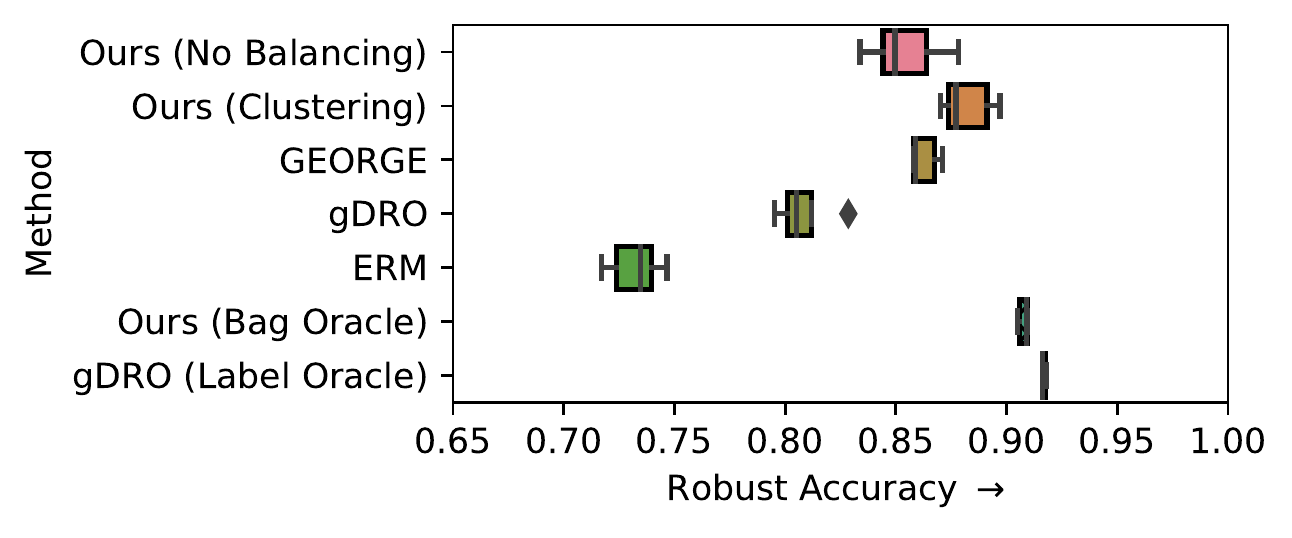}
 \includegraphics[width=0.49\textwidth]{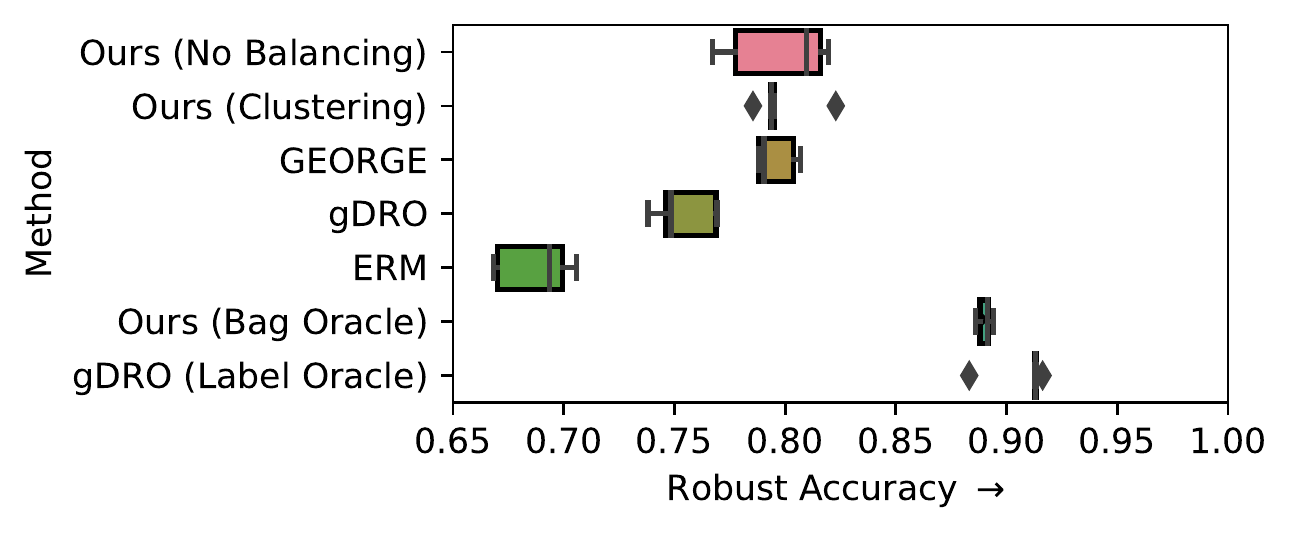}

  \caption{
    Results from \textbf{5 repeats} for the CelebA dataset for the \emph{subgroup bias} scenario.
    The task is to predict ``smiling'' vs ``non-smiling'' and the subgroups are based on gender. 
    The four sources are dropped one at a time from the training set
    (\textbf{Top Left}: smiling females; \textbf{Top Right}: smiling males; \textbf{Bottom Left}: non-smiling females;
    \textbf{Bottom Right}: non-smiling males), while the deployment set is kept fixed.
    \texttt{Robust Accuracy} refers to the minimum accuracy computed over the subgroups.
    Our method consistently performs on par with or outperforms \texttt{GEORGE} (which in turn outperforms \texttt{ERM}).
    We note that in some of the runs, \texttt{GEORGE} performed no better than random -- these results were truncated for visibility but can be found in
    Fig.~\ref{fig:celeba-gender-smiling-full}.
    Given \emph{indirect} supervision from the deployment set in the form of oracle-balancing, our method performs similarly to \texttt{gDro (Label Oracle)} that receives \emph{direct} supervision.
  }%
  \label{fig:celeba-gender-smiling}
\end{figure*}

\textbf{CelebA.} To demonstrate our method generalizes to real-world computer vision problems, we consider the CelebA dataset~\cite{liu2015celeba} comprising over 200,000 images of different celebrities.
The dataset comes with per-image annotations of physical and affective attributes such as smiling, gender, hair color, and age.
Since the dataset exhibits natural imbalance with respect to $S \times Y$, we perform no additional sub-sampling of either the training set or the deployment set.
We predict ``smiling'' as the class label and use the binary attribute, ``gender'', as the subgroup label.
Here, we consider the SB setting but rather than just designating one missing source, we repeat our experiments with each source being dropped in turn.
As before, we evaluate our method under three balancing schemes and compare with ERM and gDRO trained on only the labeled training data.
We also compare with two other variants of gDRO: 1) \texttt{gDRO (Label Oracle)}, a variant that is trained with access to the ground-truth labels of the deployment set, thus providing an upper-bound on the downstream classification performance;
2) \texttt{GEORGE} \citep{SohDunAngGuetal20}, which follows a two-step procedure of first clustering to obtain the labels for hidden subgroups, and then using these labels to train a robust classifier using gDRO.
\citet{SohDunAngGuetal20} consider a different version of the problem (termed \emph{hidden-stratification}) in which the class labels are known for all samples but the subgroup-labels are missing entirely.
We adapt \texttt{GEORGE} to our setting by modifying the semi-supervised clustering algorithm to predict the marginal distributions ($P(Y|X)$ and $P(S|X)$ instead of the joint distribution $P(Y, S| X$), allowing us to propagate the class labels from $D_{tr}$ to $D_{dep}$ (see Appendix~\ref{adapting_g}
for details).

Fig.~\ref{fig:celeba-gender-smiling} shows the results for experiments for each missing source,
showing similar trends across all instantiations of the SB scenario.
\texttt{gDRO (Label Oracle)} consistently achieves the best performance according to both metrics, with \texttt{Our Method (Bag Oracle)} consistently coming in second.
We note that while both methods use some kind of oracle, the \emph{label oracle} provides \emph{all} class/subgroup labels to its algorithm, whereas the \emph{bag oracle} only balances the bags.
Despite the large difference in the level of supervision, the margin between the two oracle methods is slim.
We observe that clustering in many cases impairs performance which can be explained by poor clustering of the missing source ($\sim$60\% accuracy).
CelebA exhibits a natural imbalance with respect to gender/smiling but not a significant one, allowing for random sampling to yield a reasonable approximation to the desired perfect bags.
We believe adjustments to the clustering algorithm -- e.g., using a self-supervised loss instead of a reconstruction loss for the encoder -- could close the gap between clustering-based and oracle-based balancing.
Nonetheless, among the non-oracle methods, variants of our method consistently match or exceed the performance of the baselines.
While the plots show \texttt{GEORGE} can perform strongly in this SB scenario, we note that for several of the missing sources, the method failed catastrophically in one out of the five runs.
We have cut off those data points here so as not to compromise the visibility of the other results; the full versions of the plots can be found in Appendix~\ref{ssec:extended-results-celeba}.
The fact that \texttt{GEORGE} leverages both the training and deployment sets in a semi-supervised way with clustering makes it the baseline most comparable to our method.
However, its performance is much more dependent on the clustering step than our method.

\section{Conclusion}
We have highlighted the problem that systematic bias or dataset curation can result in one or more subgroups having zero labeled data, and by doing so, hope to have stimulated serious consideration for it (even if to be dismissed) when planning, building, and evaluating %
ML
systems.
We proposed a two-step approach for addressing the resulting spurious correlations.
First, we construct perfect bags from an unlabeled deployment set via semi-su\-per\-vised clustering. 
Second, by matching the support of the training and deployment sets in rep\-res\-en\-ta\-tion space, we learn rep\-res\-en\-ta\-tions with subgroup-invariance.
We empirically validate our frame\-work on the Colored MNIST, CelebA and Adult Income datasets, proving it is possible to maintain high performance on subgroups with incomplete support.
Furthermore, we bound the error in the objective due to imperfect clustering.
Future work includes exploring other unsupervised-learning methods (e.g. con\-tras\-tive learning) and addressing the limitations raised in Sec.~\ref{sec:limitations}.

\section*{Acknowledgments}
This research was supported in part by a European Research Council (ERC) Starting Grant for the project ``Bayesian Models and Algorithms for Fairness and Transparency'',
funded under the European
Union's Horizon 2020 Framework Programme
(grant agreement no. 851538).
Novi Quadrianto is also supported by the Basque Government
through the BERC 2018-2021 program and by Spanish Ministry of Sciences, Innovation and Universities:
BCAM Severo Ochoa accreditation SEV-2017-0718.
\bibliography{bibfile}
\bibliographystyle{iclr2022_conference}

\newpage

\appendix
\appendix

\section{Additional experiments}\label{sec:additional-results}
\subsection{Results for Adult Income}\label{sec:adult-results}
\begin{figure*}[htp]
    \centering
    \includegraphics[width=0.49\textwidth]{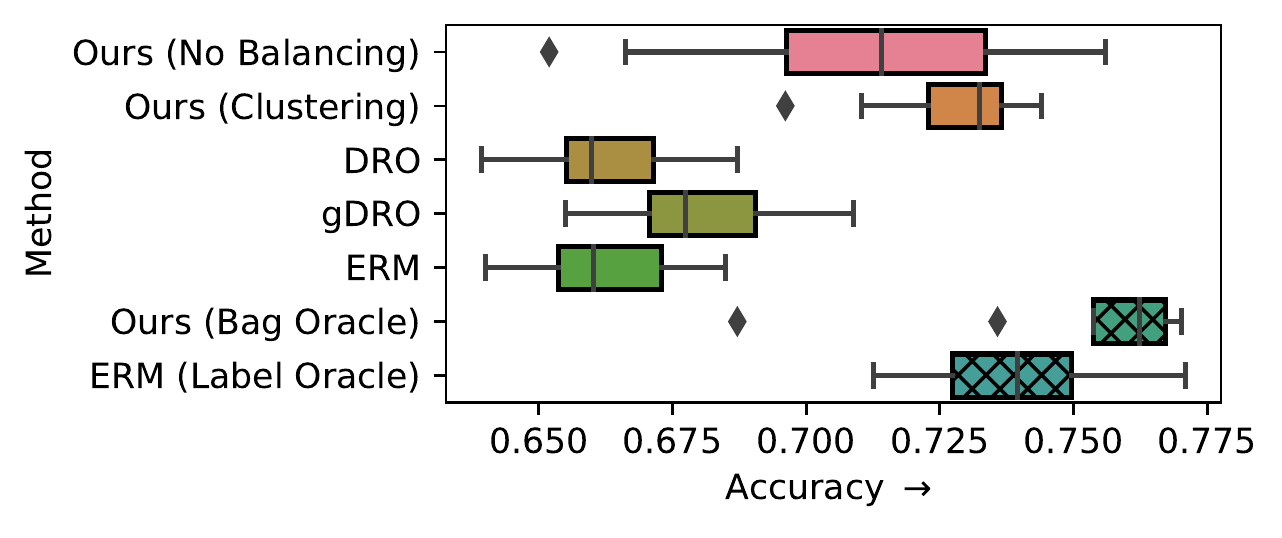}
    \includegraphics[width=0.49\textwidth]{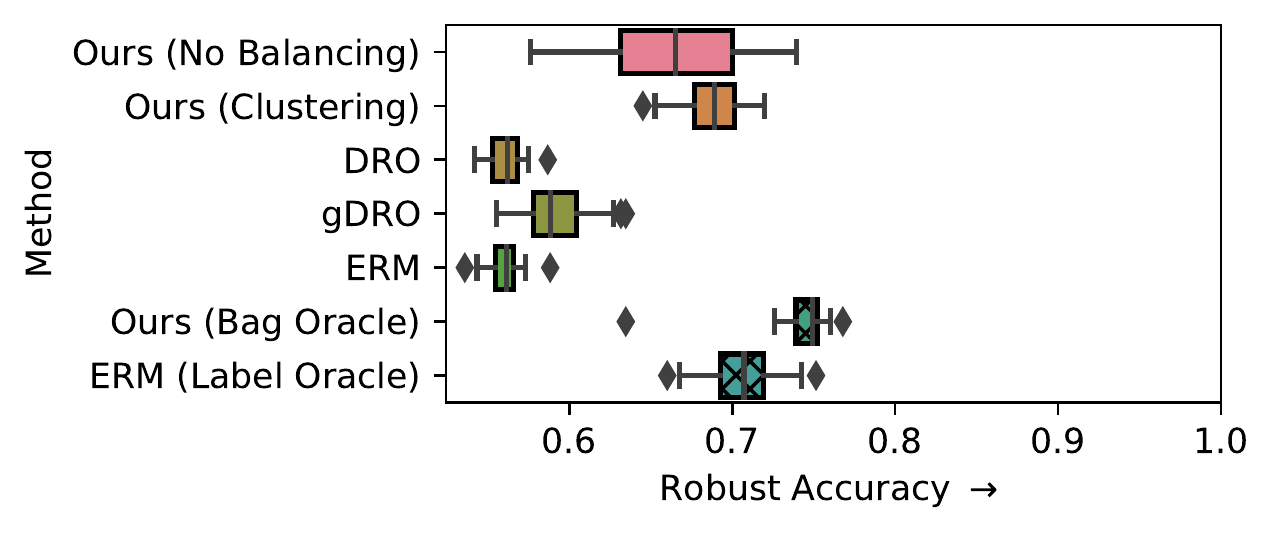}
    \includegraphics[width=0.49\textwidth]{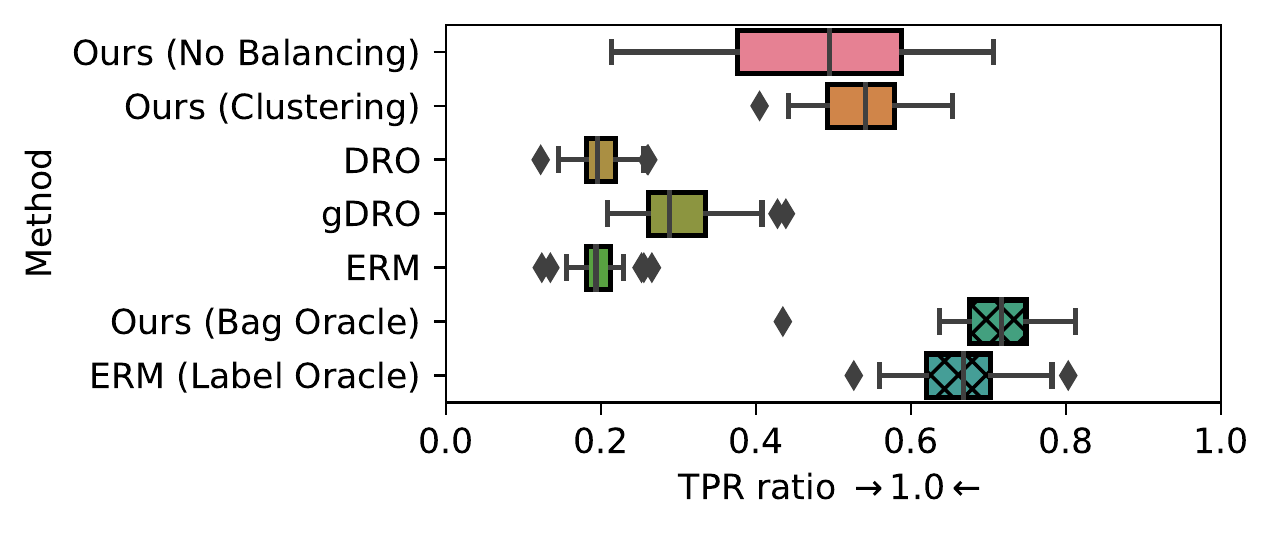}
    \includegraphics[width=0.49\textwidth]{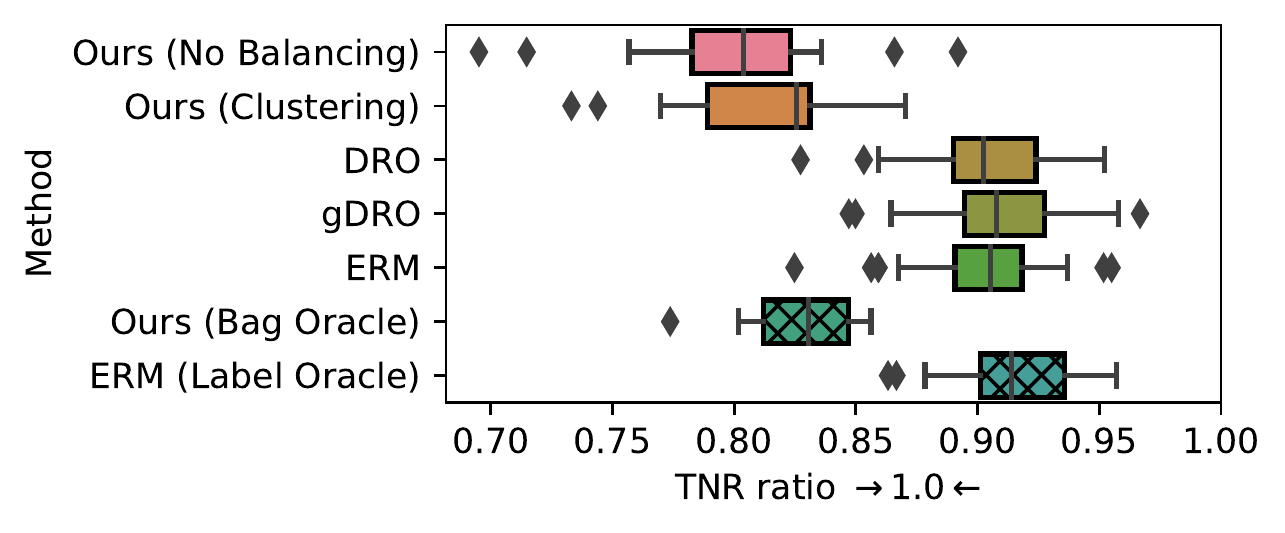}
    \caption{%
    Results for the Adult Income dataset with \emph{subgroup bias},
    for the binary classification task of predicting whether an individual earns $>$\$50,000 with a binary subgrouping based on \emph{gender}.
    \texttt{ERM (Label Oracle)} refers to a model based on ERM (empirical risk minimization),
    trained on a labeled deployment set and as such not suffering from bias present in the training set.
    \textbf{Top left}: Accuracy.
    \textbf{Top right}: Robust Accuracy.
    \textbf{Bottom left}: True positive rate ratio.
    \textbf{Bottom right}: True negative rate ratio.
    For \texttt{Ours (Clustering)}, the clustering accuracy was 69.7\% $\pm$ 0.3\%;
    }%
    \label{fig:adult-subgroup-bias}
\end{figure*}
\begin{figure*}[htp]
    \centering
    \includegraphics[width=0.49\textwidth]{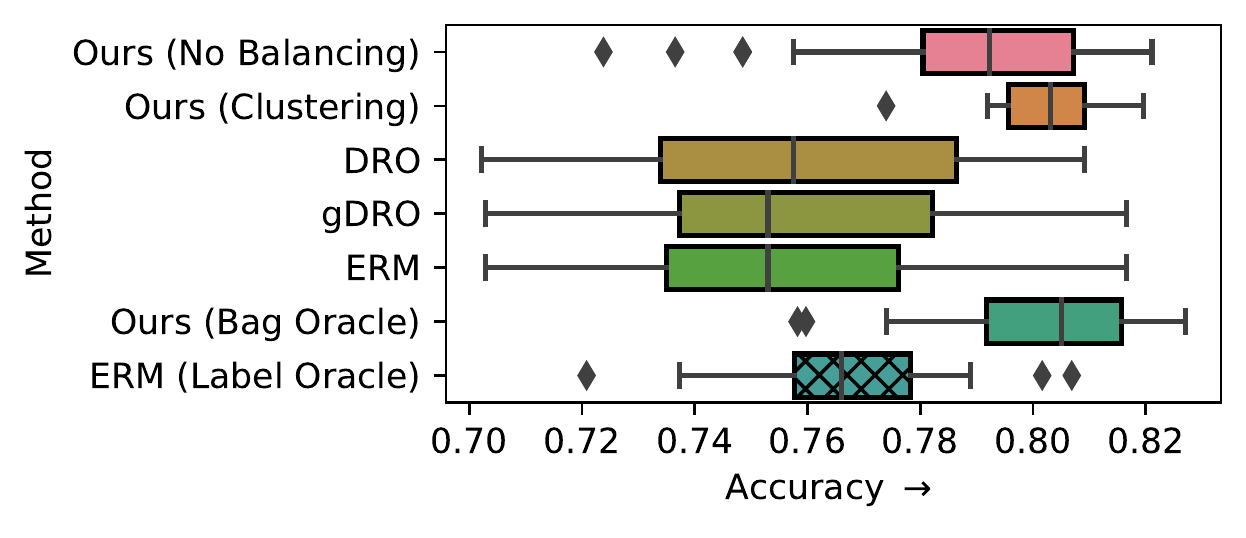}
    \includegraphics[width=0.49\textwidth]{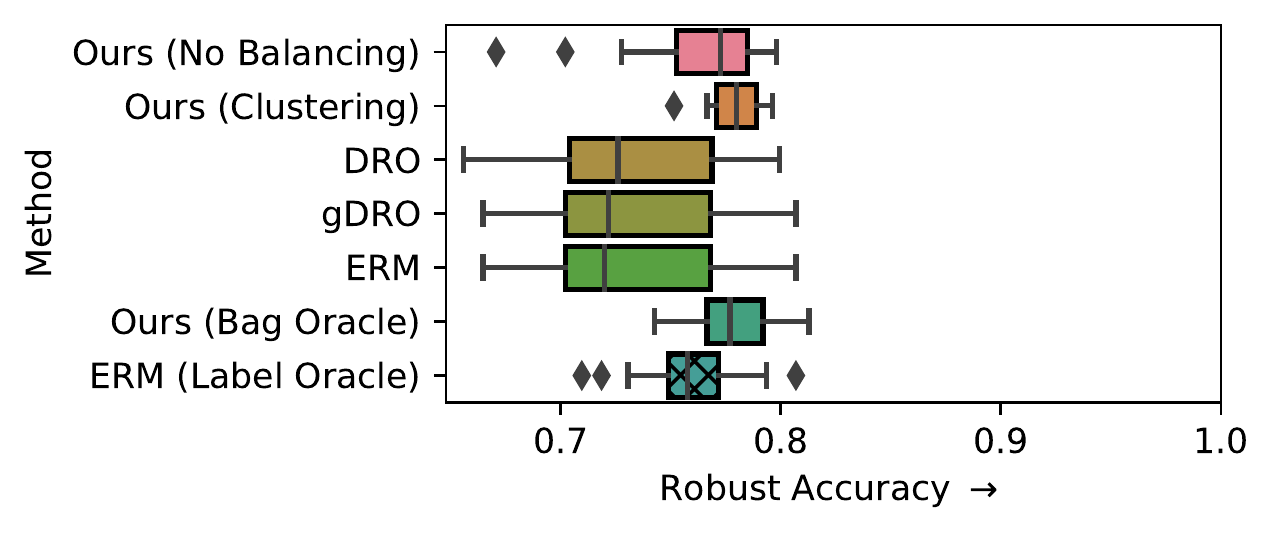}
    \includegraphics[width=0.49\textwidth]{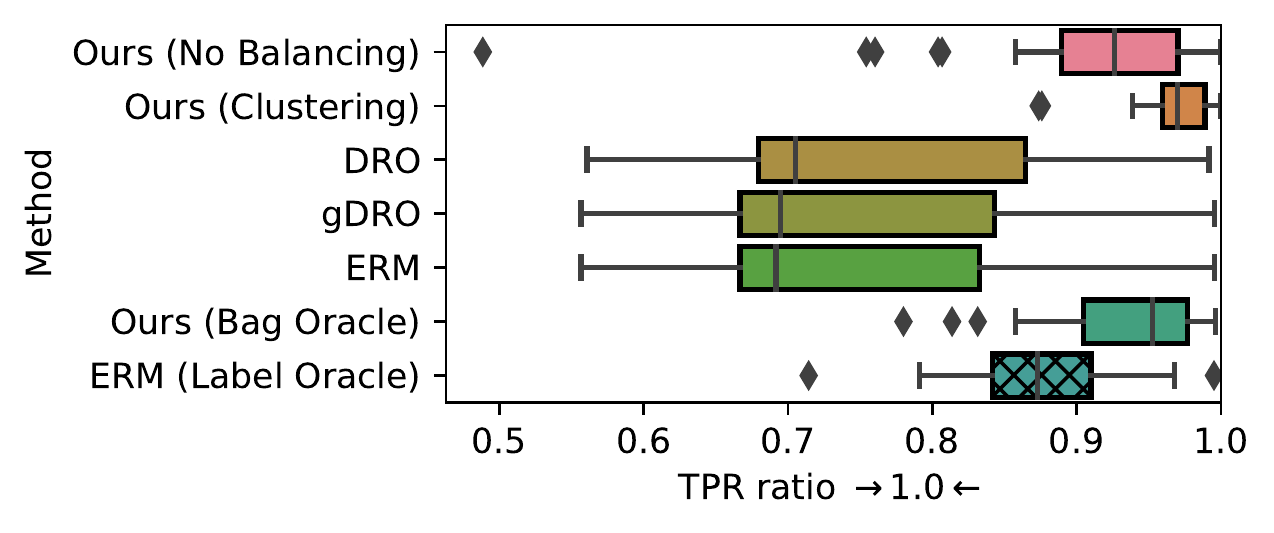}
    \includegraphics[width=0.49\textwidth]{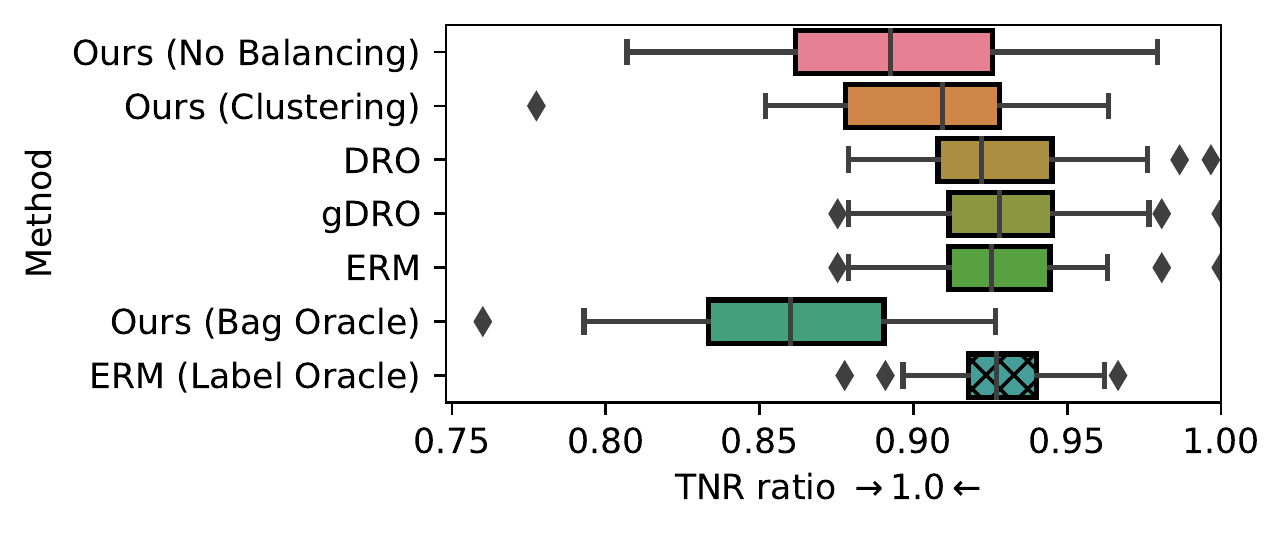}
    \caption{%
    Results for the Adult Income dataset with a \emph{missing subgroup},
    for the binary classification task of predicting whether an individual earns $>$\$50,000 with a binary subgrouping based on \emph{gender}.
    \texttt{ERM (Label Oracle)} refers to a model based on ERM (empirical risk minimization),
    trained on a \textbf{l}abeled \textbf{d}eployment set; thus not suffering from bias in the training set.
    \textbf{Top left}: Accuracy.
    \textbf{Top right}: Robust Accuracy.
    \textbf{Bottom left}: True positive rate ratio.
    \textbf{Bottom right}: True negative rate ratio.
    For\texttt{Ours (Clustering)}, the clustering accuracy was 60.4\% $\pm$ 0.8\%.
    }%
    \label{fig:adult-missing-subgroup}
\end{figure*}
Figures~\ref{fig:adult-subgroup-bias} and \ref{fig:adult-missing-subgroup} show results from our method on the Adult Income dataset \cite{Dua:2017}.
This dataset is a common dataset for evaluating fair machine learning models. 
Each instance in the dataset is described by $14$ characteristics including gender, education, marital status, number of work hours per week among others, along with a label denoting income level ($\geq$\$50K or not). 
We transform the representation into $62$ real and binary features along with the subgroup label $s$. %
The dataset is naturally imbalanced with respect to gender: 30\% of the males are labeled as earning more than \$50K per year (high income), while only 11\% of females are labeled as such.
For further details on the dataset construction, see section~\ref{ssec:dataset-construction-adult}.
Following standard practice in algorithmic fairness  e.g. \cite{ZemeWuSwePitetal13}, we consider gender to be the subgroup label $s$.
A \emph{source} is defined as certain combinations of the subgroup, $s$, and target class, $y$.

For the Adult Income dataset, we study the following two settings of missing sources, a subgroup bias setting and a more extreme missing subgroup setting:
1) \emph{subgroup bias}: we have labeled training data for males ($s=1$) with both positive and negative outcomes, but for the group of females ($s=0$), we only observe the one-sided negative outcome: \(\mathcal{S}_{tr}(y=1)=\{1\}\);
2) \emph{missing subgroup}: we have training data for males with positive and negative outcomes, but do not have labeled data for females, i.e.\ \(\mathcal{S}_{tr}(y=0)=\{1\}\) and \(\mathcal{S}_{tr}(y=1)=\{1\}\).

As before, \texttt{Ours (Clustering)}, \texttt{Ours (No Balancing)}\ and \texttt{Ours (Bag Oracle)} denote variants of our method with different deployment-set balancing strategies.
As baseline methods, we have \texttt{ERM} (standard empirical risk minimization with balanced batches), \texttt{DRO} \cite{HasSriNamLia18}, \texttt{gDRO} \cite{sagawa2019distributionally}
and \texttt{ERM (Label Oracle)} which is the same model as \texttt{ERM}, but trained with access to the ground-truth labels of the deployment set.

In both settings, we observe the same order as for the other dataset in terms of accuracy: \texttt{Ours (Bag Oracle)} achieves the highest performance, followed by \texttt{Ours (Clustering)}, then \texttt{Ours (No Balancing)}.
However, for the \emph{missing subgroup} setting, \texttt{Ours (Clustering)} and \texttt{Ours (Bag Oracle)} perform almost identically, with the former outstripping the latter slightly in terms of de-biasing metrics.
This reduced reliance on balancing can be explained by the additional supervision that comes with having two sources missing instead of one -- in order for the discriminator to distinguish between bags from the deployment set and bags from the training set, the former need only contain \emph{one} of the two missing sources.

Generally, we observe a high variance in the results. This is not attributable to our method, however, with the baselines exhibiting the same behavior, but rather to the fact that the Adult Income dataset is a very noisy dataset which, at the best of times, allows only about 85\% accuracy to be attained (see also \cite{agrawal2020debiasing}). The problem is that samples vary widely in how informative they are. This, coupled with us artificially biasing the dataset to be even more biased (as \emph{subgroup bias} and \emph{missing subgroup}), makes the attainable performance very dependent on which samples the classifier gets to see, which varies according to the random seed used for the data set split.

\subsection{Results for 3-digit 3-color variant of Colored MNIST}\label{ssec:3-digit-3-color}
To investigate how our method scales with the number of sources, we look to a 3-digit, 3-color variant of the dataset in the \emph{subgroup bias} setting where four sources are missing from $D_{tr}$.
Results for this configuration are shown in fig.~\ref{fig:cmnist-3dig-4miss}.
We see that the performance of \texttt{Ours (No Balancing)} is quite close to that of \texttt{Ours (Bag Oracle)}. We suspect this is because balancing is less critical with the increased number of subgroups strengthening the training signal.
As inter-subgroup ratios do not make for suitable metric for non-binary $S$,
we instead quantify the invariance of the predictions to the subgroup with
the HGR maximal correlation \cite{renyi1959measures}.

\begin{figure*}[htp]
  \centering
  \includegraphics[width=0.49\textwidth]{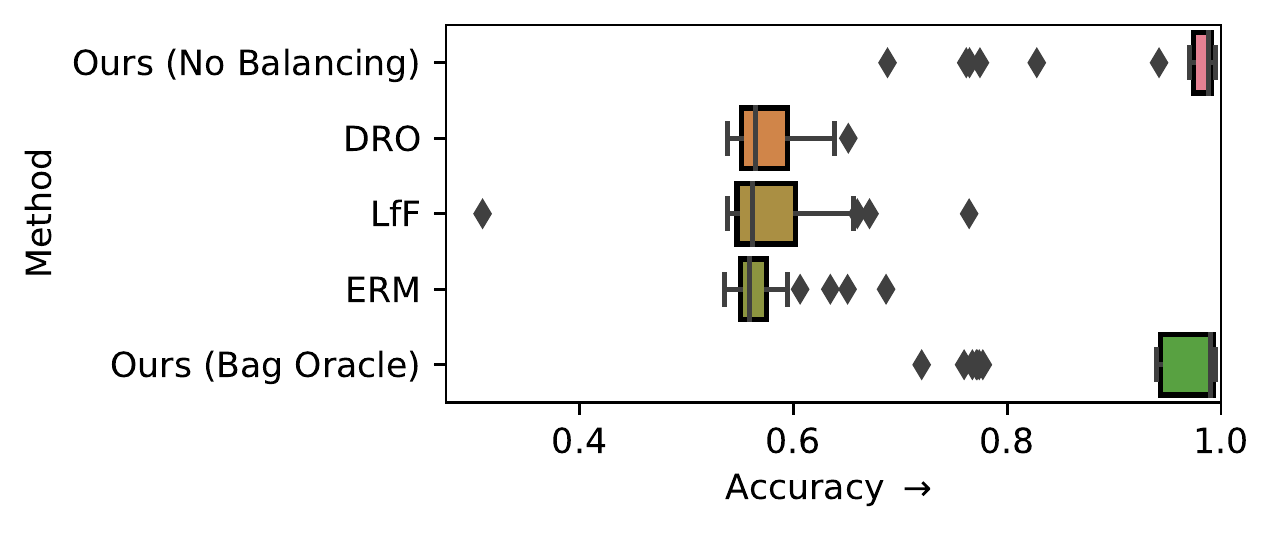}
  \includegraphics[width=0.49\textwidth]{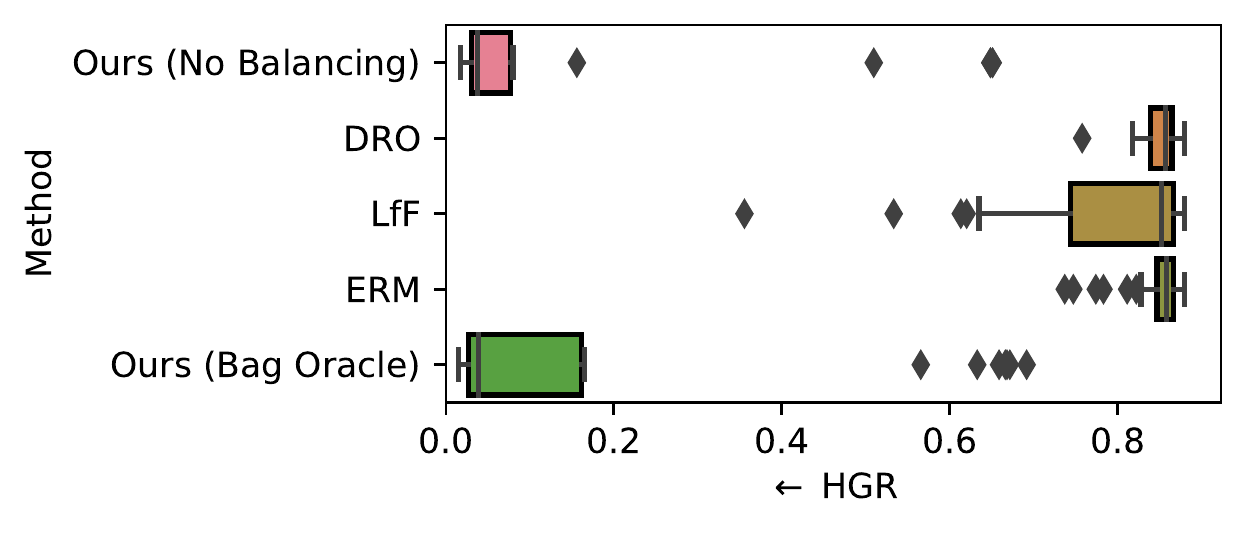}
  \caption{
    Results from \textbf{30 repeats} for the Colored MNIST dataset with three digits: `2', `4' and `6'.
    Four combinations of digit and color are missing: {\color{green}green} 2's, {\color{blue}blue} 2's, {\color{blue}blue} 4's and {\color{green}green} 6's.
    \textbf{Left}: Accuracy.
    \textbf{Right}: Hirschfeld-Gebelein-R\'enyi maximal correlation \cite{renyi1959measures} between $S$ and $Y$.
  }%
  \label{fig:cmnist-3dig-4miss}
\end{figure*}

\subsection{Extended Results for CelebA}\label{ssec:extended-results-celeba}

As alluded to in main text, for three out of four of the missing gender/smiling quadrants, the \texttt{GEORGE} baseline produced an extreme outlier for one out of the five total repeats - these outliers were omitted from the plots to ensure the discriminability of the other results. We reproduce the full, untruncated versions of these plots here in fig. \ref{fig:celeba-gender-smiling-full}. 
We have also included \texttt{Accuracy} metric in fig. \ref{fig:celeba-gender-smiling-full}.

\begin{figure*}[htp]
  \centering
 \includegraphics[width=0.49\textwidth]{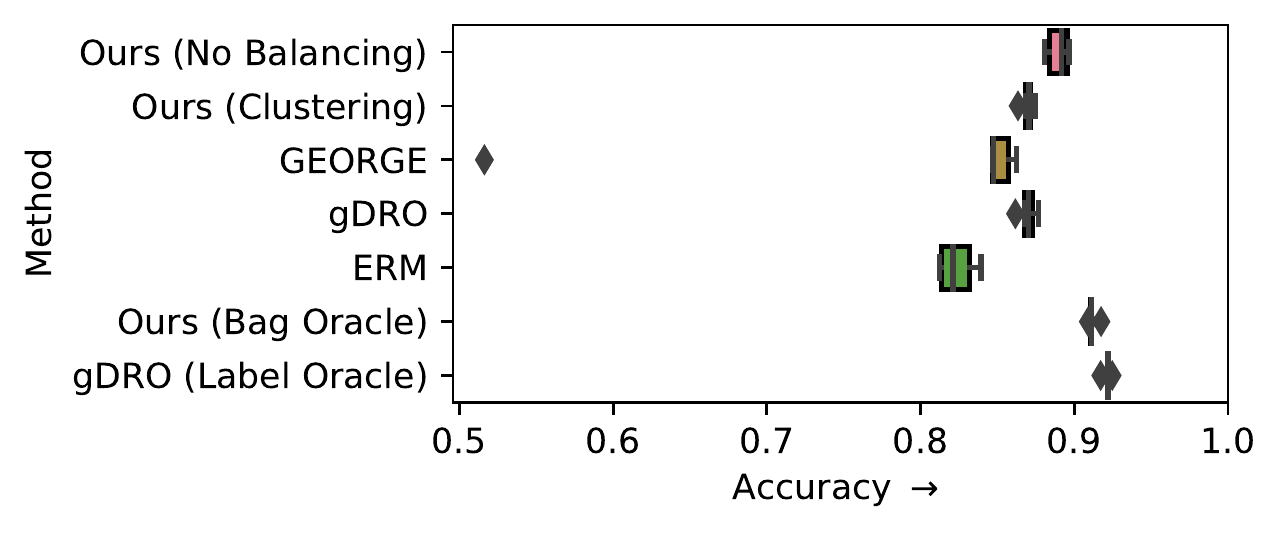}
  \includegraphics[width=0.49\textwidth]{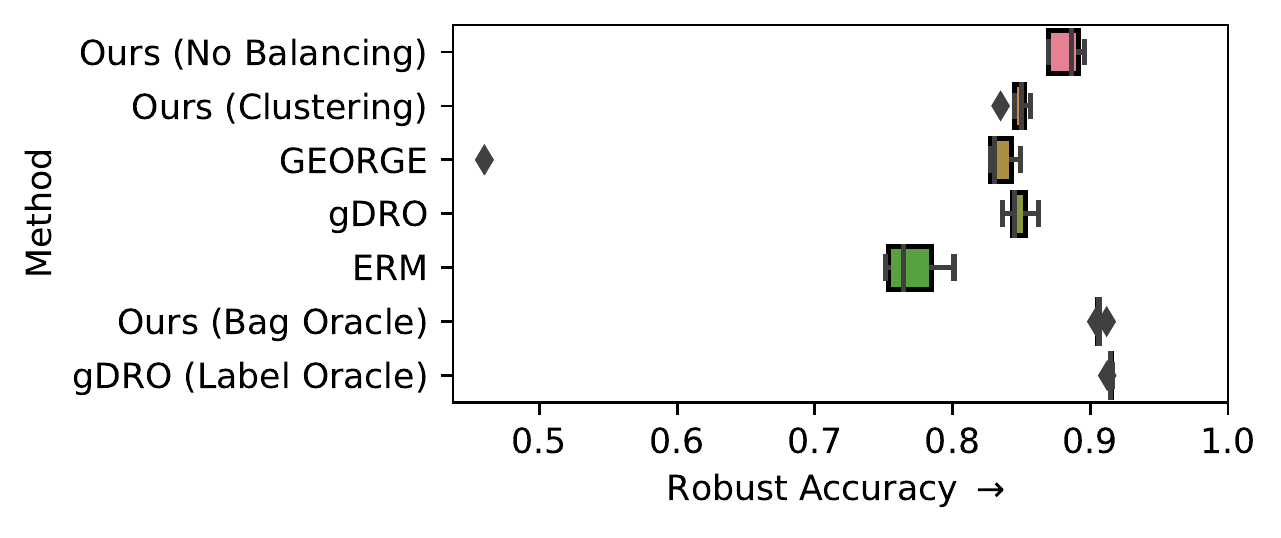}
 \includegraphics[width=0.49\textwidth]{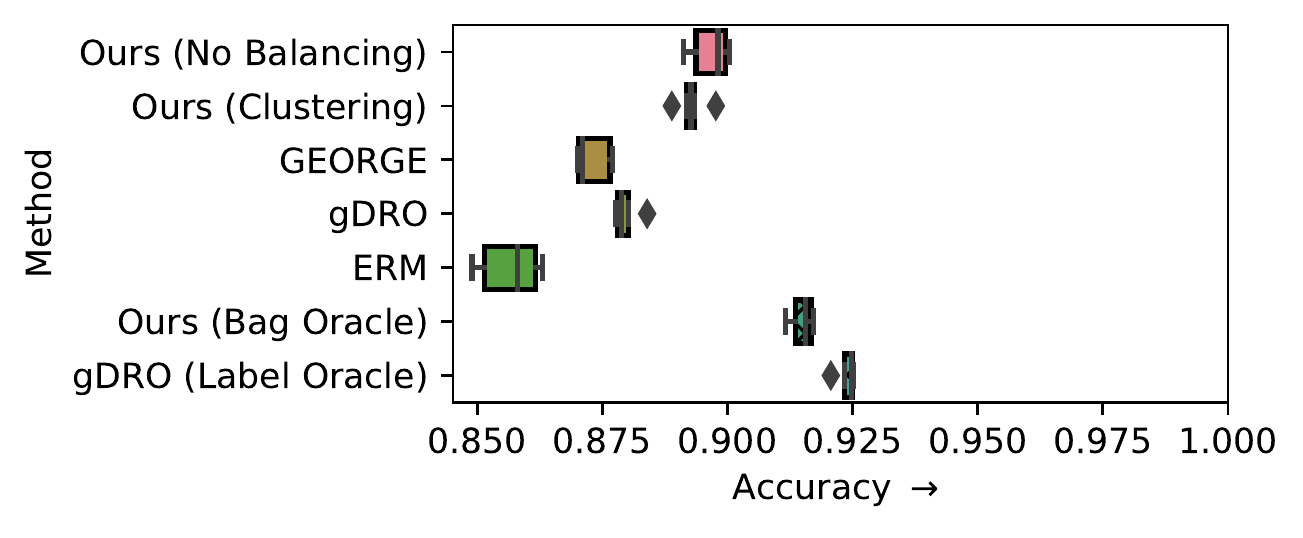}
 \includegraphics[width=0.49\textwidth]{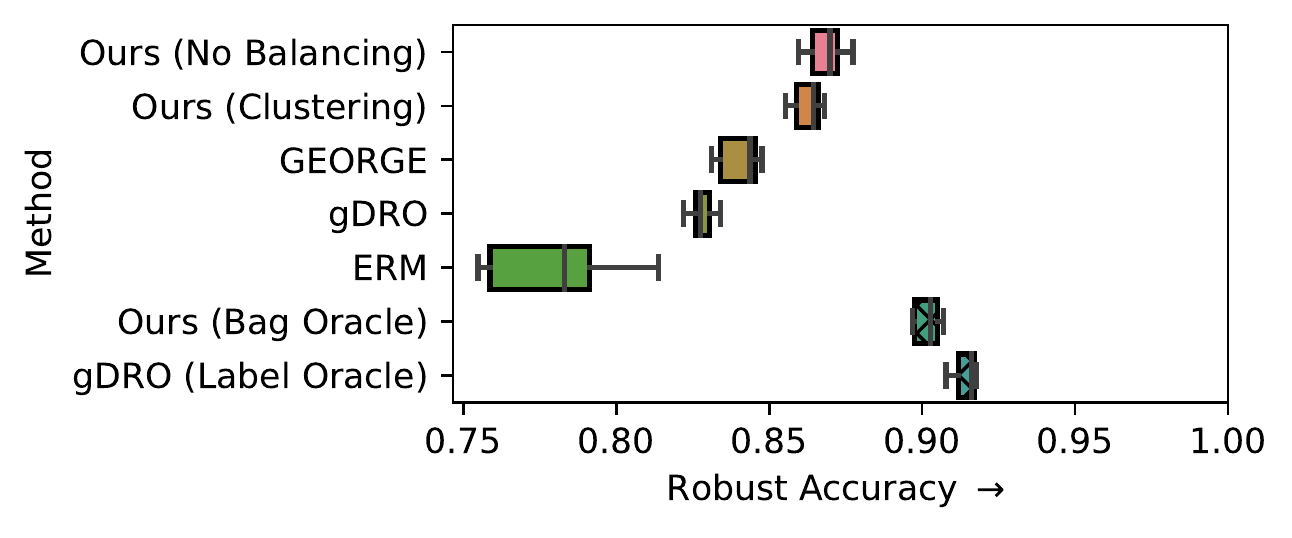}
 \includegraphics[width=0.49\textwidth]{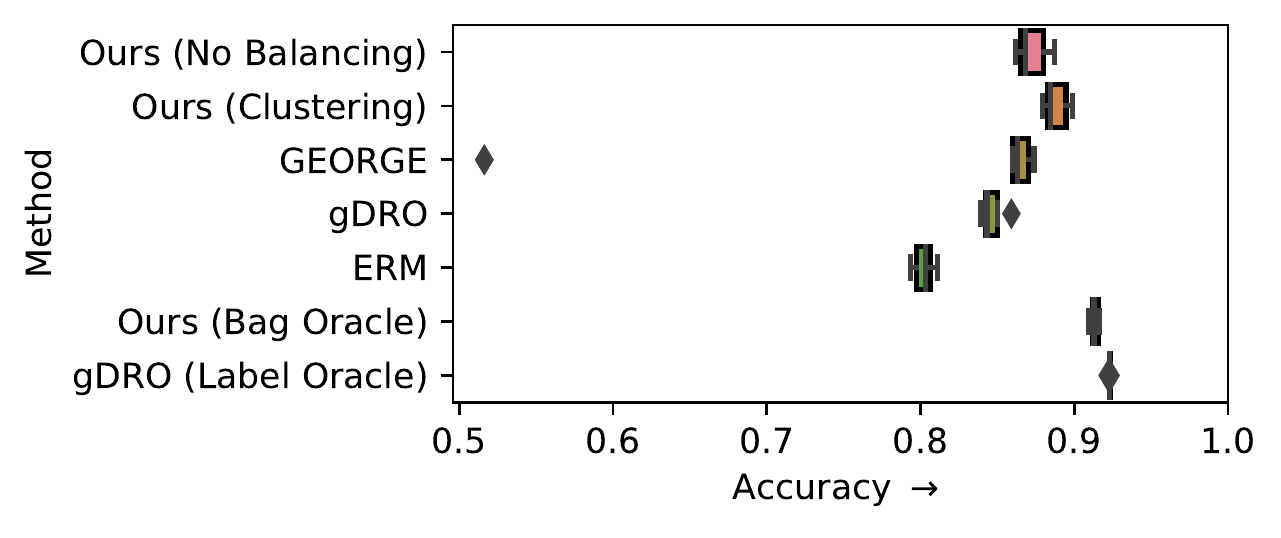}
  \includegraphics[width=0.49\textwidth]{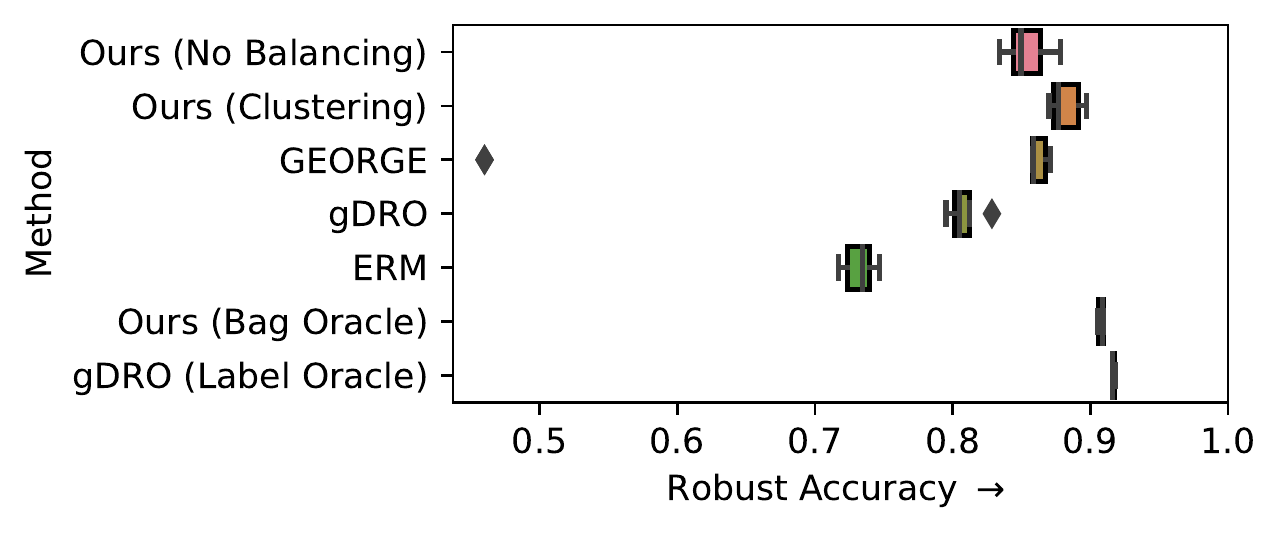}
 \includegraphics[width=0.49\textwidth]{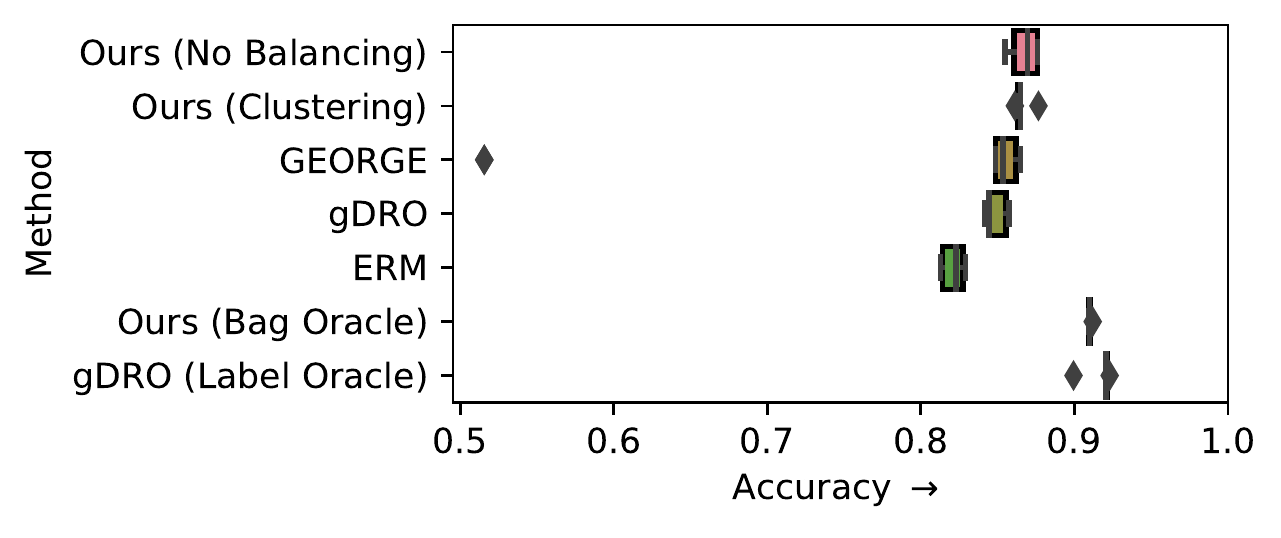}
  \includegraphics[width=0.49\textwidth]{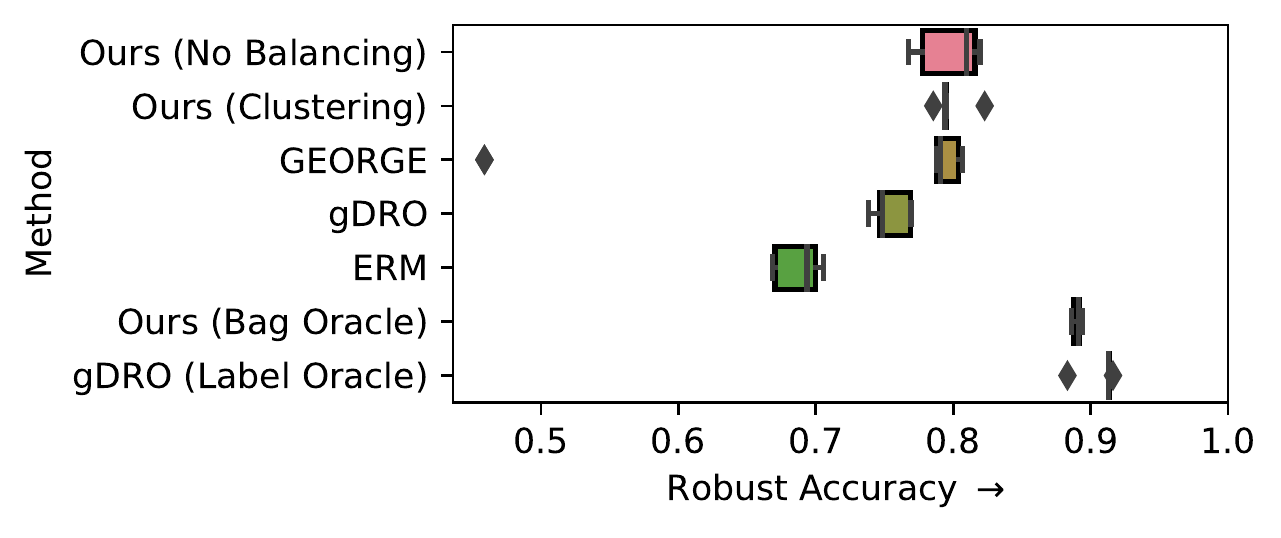}
\caption{
    Results from \textbf{5 repeats} for the CelebA dataset for the \emph{subgroup bias} scenario. The task is to predict ``smiling'' vs ``non-smiling'' and the subgroups are based on gender. The four sources are dropped one at a time from the training set
    (\textbf{first row}: smiling females; \textbf{second row}: smiling males; \textbf{third row}: non-smiling females, \textbf{fourth row}: non-smiling males), while the deployment set is kept fixed. ``Robust Accuracy'' refers to the minimum accuracy computed over the subgroups. 
  }%
 \label{fig:celeba-gender-smiling-full}
\end{figure*}

\section{Theoretical analysis}\label{sec:theoretical-analysis}
In this section, we present our theoretical results concerning the validity of our support-matching objective and the bound on the error introduced into it by clustering.
We use notation consistent with that used throughout the main text.

\subsection{Sampling function for the objective}%
\label{sub:sampling_function_for_objective}
The stated objective uses the following helper function:
\begin{align}
\Pi(s',y') = \begin{cases}
\{s'\}&\text{if }\,\mathcal{S}_{tr}(y=y')=\mathcal{S}\\
\mathcal{S}_{tr}(y=y')&\text{otherwise}~.
\end{cases}
\end{align}
This helper function determines which $s$ value in the training set an $s$-$y$ pair from the deployment set is mapped to.
(The $y$ value always stays \emph{the same} when mapping from deployment set to training set.)
To demonstrate the usage of this function,
we consider the example of binary Colored MNIST
with \(\gS=\{\text{\color{purple}{purple}}, \text{\color{green}green}\}\) and \(\gY=\{2, 4\}\)
where the training set is missing \((s=\text{purple}, y=4)\).
In this case, $\Pi$ takes on the following values:
\begin{align}
  \Pi(\text{{\color{purple}purple}}, 2) &= \{\text{\color{purple}purple}\}\\
  \Pi(\text{\color{green}green},  2) &= \{\text{\color{green}green} \}\\
  \Pi(\text{\color{purple}purple}, 4) &=\mathcal{S}_{tr}(y=4) = \{\text{\color{green}green} \}\\
  \Pi(\text{\color{green}green},  4) &=\mathcal{S}_{tr}(y=4) = \{\text{\color{green}green} \}
\end{align}
It is essential that \((s=\text{\color{purple}{purple}}, y=4)\) from the deployment set is mapped to \((s=\text{\color{green}{green}}, y=4)\)
from the training set, and not \((s=\text{\color{purple}{purple}}, y=2)\).
This procedure is illustrated in fig. \ref{fig:matching-repeated-correct}, and contrasted with an incorrect procedure based on balancing the bag according to $s$ in \ref{fig:matching-repeated-incorrect} -- such a procedure would result in invariance to $y$ instead of $s$, which is obviously undesirable.

In practice, we use the following sampling function $\pi$ to implement $\Pi$, sampling from it for all $(s,y) \in S \times Y$:
\begin{align}
  \pi(s',y') = \begin{cases} x\sim P_\mathit{train}(x|s=s',y'), \\
  \quad\quad\quad\quad\quad\quad\quad\quad\text{if }\,\mathcal{S}_{tr}(y=y')=\mathcal{S} \\
    x\sim P_\mathit{train}(x|s=\check{s},y'), \check{s}\sim \mathrm{uniform}(S_{tr}), \\
    \quad\quad\quad\quad\quad\quad\quad\quad\text{otherwise}~.
\end{cases}
\label{eq:functional-sampling}
\end{align}
With the assumption that our data follows a two-level hierarchy and all digits appear in the training set, the above sampling function $\pi$ traverses the first level which corresponds to the class-level information, and \emph{samples} the second level which corresponds to subgroup-level information when we have missing sources.

\begin{figure}[htp]
  \begin{subfigure}{0.49\textwidth}
    \centering
    \includegraphics[width=0.9\linewidth]{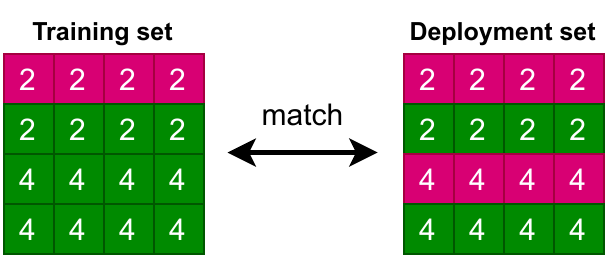}
    \caption{Correct matching procedure.}%
    \label{fig:matching-repeated-correct}
  \end{subfigure}
  \begin{subfigure}{0.49\textwidth}
    \centering
    \includegraphics[width=0.9\linewidth]{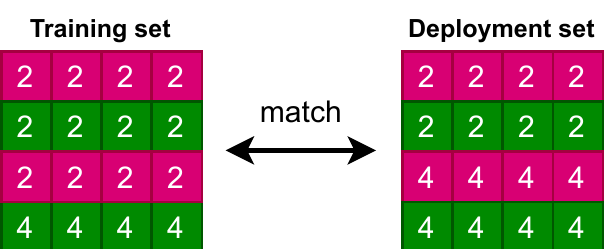}
    \caption{\emph{In}correct matching procedure.}%
    \label{fig:matching-repeated-incorrect}
  \end{subfigure}
  \caption{%
    The two natural matching procedures for one missing source in the training set.
    Only figure \ref{fig:matching-repeated-correct} (left) produces the desired invariance.
  }%
  \label{fig:matching-repeated}
\end{figure}

\subsection{Implication of the objective}\label{implication-of-the-objective}
We restate proposition 1 and present the proof.

We prove here that an encoding $f$ satisfying the objective is invariant to \(s\),
at least in those cases where the class does not have full \(s\)-support (which is exactly the case where it matters).

\begin{theorem}
If \(f\) is such that
\begin{align}
p_f|_{s\in \Pi(s',y'),y=y'} = q_f|_{s=s',y=y'}\quad\forall s'\in\mathcal{S}, y' \in\mathcal{Y}
\end{align}
and \(P_\mathit{train}\) and \(P_\mathit{dep}\) are data
distributions that correspond to the real data distribution \(P\),
except that some \(s\)-\(y\)-combinations are less prevalent, or, in the
case of \(P_\mathit{train}\), missing entirely, then, for every
\(y'\in\mathcal{Y}\), there is either full coverage of \(s\) for \(y'\)
in the training set (\(\mathcal{S}_{tr}(y=y')=\mathcal{S}\)), or the
following holds:
\begin{align}
P(s=s'|f(x)=z', y=y')=\frac{1}{n_s}~.
\end{align}
In other words: for \(y=y'\), \(f(x)\) is not predictive of \(s\).
\end{theorem}

\begin{proof}
If \(y'\) has full coverage of \(s\) in the training set, there is nothing to prove.
So, assume \(y'\) does not have full \(s\)-support.
That means \(\Pi(s',y')=\mathcal{S}_{tr}(y=y')\) for all \(s'\in \mathcal{S}\).
And so
\begin{align}
&P_\mathit{train}(f(x)=z'|s\in \mathcal{S}_{tr}(y=y'), y=y')\\
=\;&P_\mathit{dep}(f(x)=z'|s=s', y=y')\quad\quad\quad\quad\forall s'\in\mathcal{S} \nonumber
\end{align}
The left-hand side of this equation does not depend on \(s'\)
and so the right-hand side must have the same value for all \(s'\in\mathcal{S}\), which implies:
\begin{align}
&P_\mathit{dep}(f(x)=z'|s=s', y=y') \nonumber\\
=\;&P_\mathit{dep}(f(x)=z'|y=y')
\end{align}
Now, by assumption, the different data distributions \emph{train} and \emph{deployment}
only differ from the ``true'' distribution by the prevalence of the different \(s\)-\(y\)-combinations,
with the \emph{deployment} data distribution having all combinations but potentially not in equal quantity.
However, as we restrict ourselves to a certain combination (\(s=s',y=y'\)) in the above equation,
the equation also holds in the true data distribution:
\begin{align}
&P(f(x)=z'|s=s', y=y') \nonumber\\
=\;&P(f(x)=z'|y=y')
\end{align}
Then, using Bayes' rule, we get
\begin{align}
&P(s=s'|f(x)=z', y=y') \nonumber\\
=\;&\frac{P(f(x)=z'|s=s', y=y')P(s=s'|y=y')}{P(f(x)=z'|y=y')} \nonumber\\
=\;&P(s=s'| y=y')~.
\end{align}
Finally, in the true data distribution, we have a uniform prior:
\(P(s=s'|y=y')=(n_s)^{-1}\). This concludes the proof.
\end{proof}

\subsection{Bound on error introduced by clustering}\label{bound-on-error-introduced-by-clustering}

As previously stated, in practice, no labels are available for the deployment set.
Instead, we identify the relevant groupings by clustering.
Such clustering cannot be expected to be perfect.
So, how will clustering affect the calculation of our objective?

\begin{theorem}
If \(q_f(z)\) is a data distribution on
\(\mathcal{Z}\) that is a mixture of \(n_y\cdot n_s\) Gaussians, which
correspond to all unique combinations of \(y\in\mathcal{Y}\) and
\(s\in\mathcal{S}\), and \(p_f(z)\) is any data distribution on
\(\mathcal{Z}\), then without knowing \(y\) and \(s\) on \(q_f\), we can
estimate
\begin{align}
\sum\limits_{s'\in\mathcal{S}}\sum\limits_{y'\in\mathcal{Y}} TV(p_f|_{s\in \Pi(s',y'),y=y'}, q_f|_{s=s',y=y'})
\end{align}
with an error that is bounded by \(\tilde{O}(\sqrt{1/N})\) with high
probability, where \(N\) is the number of samples drawn from \(q_f\) for
learning.
\end{theorem}

\begin{proof}
First, we produce an estimate \(\hat{q}_f\) of \(q_f\) using the algorithm from \cite{ashtiani2020near},
which gives us a mixture-of-Gaussian distribution of \(n_y\cdot n_s\) components
with \(TV(q_f, \hat{q}_f)\leq \tilde{O}(\sqrt{1/N})\) with high probability,
where \(N\) is the number of data points used for learning the estimate.
Then, by Lemma 3 from \cite{SohDunAngGuetal20}, \emph{there exists} a
mapping \(i\) from the components \(k\) of the Gaussian mixture
\(\hat{q}_f\) to the \(s\)-\(y\)-combinations in \(q_f\) such that
\begin{align}
&TV(q_f(z|s=s',y=y'),\hat{q}_f(z|k=i(s',y'))) \\
&\quad\quad\quad\quad\quad\quad\quad\quad\quad\quad\quad\quad\quad\quad\leq \tilde{O}\left(\frac{1}{\sqrt{N}}\right)~. \nonumber
\end{align}
Now, consider the element of the sum in the objective that corresponds
to \((s',y')\):
\begin{align}
&TV(p_f(z|s\in \Pi(s',y'),y=y'), q_f(z|s=s',y=y'))\nonumber\\
\leq \;&TV(p_f(z|s\in \Pi(s',y'),y=y'), \hat{q}_f(z|k=i(s',y')))\nonumber\\
&\quad\quad+TV(\hat{q}_f(z|k=i(s',y')), q_f(z|s=s',y=y'))\nonumber\\
\leq \;&TV(p_f(z|s\in \Pi(s',y'),y=y'), \hat{q}_f(z|k=i(s',y'))) \nonumber\\
&\quad\quad+\tilde{O}(1/\sqrt{N})
\end{align}
Thus, for the whole sum over \(s\) and \(y\), the error is bounded by
\begin{align}
\sum\limits_{s'\in\mathcal{S}}\sum\limits_{y'\in\mathcal{Y}}\tilde{O}(\sqrt{1/N})
\leq n_sn_y \max\limits_{(s',y')\in\gS\times\gY}\tilde{O}(\sqrt{1/N})
\end{align}
which is equivalent to just \(\tilde{O}(\sqrt{1/N})\).
\end{proof}

\section{Dataset Construction}\label{sec:dataset-construction}
\subsection{Colored MNIST and biasing parameters}
The MNIST dataset \cite{lecun1998gradient} consists of 70,000 (60,000 designated for training, 10,000 for testing) images of gray-scale hand-written digits. We color the digits following the procedure outlined in \cite{KehBarThoQua20}, randomly assigning each sample one of ten distinct RGB colors. Each source is then a combination of digit-class (class label) and color (subgroup label). We use no data-augmentation aside from symmetrically zero-padding the images to be of size 32x32.

To simulate a more real-world setup where the data, labeled or otherwise, is not naturally balanced, we bias the Colored MNIST training and deployment sets by downsampling certain color/digit combinations. The proportions of each such combination \emph{retained} in the \emph{subgroup bias} (in which we have one source missing from the training set) and \emph{missing subgroup} (in which we have two sources missing from the training set) are enumerated in table~\ref{color_mnist_biasing_po} and \ref{color_mnist_biasing_id}, respectively.
For the 3-digit-3-color variant of the problem, no biasing is applied to either the deployment set or the training set (the missing combinations are specified in the caption accompanying figure~\ref{fig:cmnist-3dig-4miss-add}); this variant was experimented with only under the subgroup-bias setting.

\begin{table}[ht]
\caption{Biasing parameters for the training (left) and deployment (right) sets of Colored MNIST in the \emph{subgroup bias} setting.}
\label{color_mnist_biasing_po}
\centering
\begin{tabular}{lcc}
\toprule
Combination   & \multicolumn{2}{c}{Proportion retained} \\ \cmidrule(lr){2-3}
  & training set & deployment set \\ \midrule
(y = 2, s = {\color{purple}purple}) & 1.0  & 0.7 \\
(y = 2, s = {\color{green}green})   & 0.3  & 0.4 \\
(y = 4, s = {\color{purple}purple}) & 0.0  & 0.2 \\
(y = 4, s = {\color{green}green})   & 1.0  & 1.0 \\
\bottomrule
\end{tabular}
\end{table}

\begin{table}[ht]
\caption{Biasing parameters for the training (left) and deployment (right) sets of Colored MNIST in the \emph{missing subgroup} setting.}
\label{color_mnist_biasing_id}
\centering
\begin{tabular}{lcc}
\toprule
Combination   & \multicolumn{2}{c}{Proportion retained} \\ \cmidrule(lr){2-3}
  & training set & deployment set \\ \midrule
(y = 2, s = {\color{purple}purple}) & 0.0  & 0.7 \\
(y = 2, s = {\color{green}green})   & 0.85 & 0.6 \\
(y = 4, s = {\color{purple}purple}) & 0.0  & 0.4 \\
(y = 4, s = {\color{green}green})   & 1.0  & 1.0 \\
\bottomrule
\end{tabular}
\end{table}

\subsection{Adult Income biasing parameters}\label{ssec:dataset-construction-adult}
For the Adult Income dataset, we do not need to apply any synthetic biasing as the dataset naturally contains some bias w.r.t. $s$. Thus, we instantiate the deployment set as just a random subset of the original dataset. However, explicit balancing of the test set \emph{is} needed to yield meaningful evaluation (namely through the penalizing of biased classifiers) but care needs to be taken in doing so. Balancing the test set such that
\begin{align}
    |\{x \in X |s=0, y=0\}| &= |\{x \in X |s=1, y=0\}|    \nonumber\\
    \text{and}~|\{x \in X |s=0, y=1\}| &= |\{x \in X |s=1, y=1\}|
\end{align}
where for both target classes, $y=0$ and $y=1$, the proportions of the groups $s=0$ and $s=1$ are made to be the same, is intuitive, yet at the same time precludes sensible comparison of the accuracy/fairness trade-off of the different classifiers.
Indeed, with the above conditions, a majority classifier (predicting all 1s or 0s) achieves comparable accuracy to the fairness-unaware baselines, while also being perfectly fair by construction.
This observation motivated us to devise an alternative scheme, where we balance the test set according to the following constraints
\begin{align}
    & |\{x \in X |s=0, y=0\}| 
    = |\{x \in X |s=0, y=1\}|  \nonumber \\
    = &|\{x \in X |s=1, y=1\}|
    = |\{x \in X |s=1, y=0\}|~.
 \end{align}
That is, all subsets of $\gS \times \gY$ are made to be equally sized. Under this new scheme the accuracy of the the majority classifier is 50\% for the binary-classification task.

\section{Model details and optimization}
\subsection{Overview of model architecture}\label{sec:model-arch}
\begin{figure*}[htp]
    \centering
    \includegraphics[width=\textwidth]{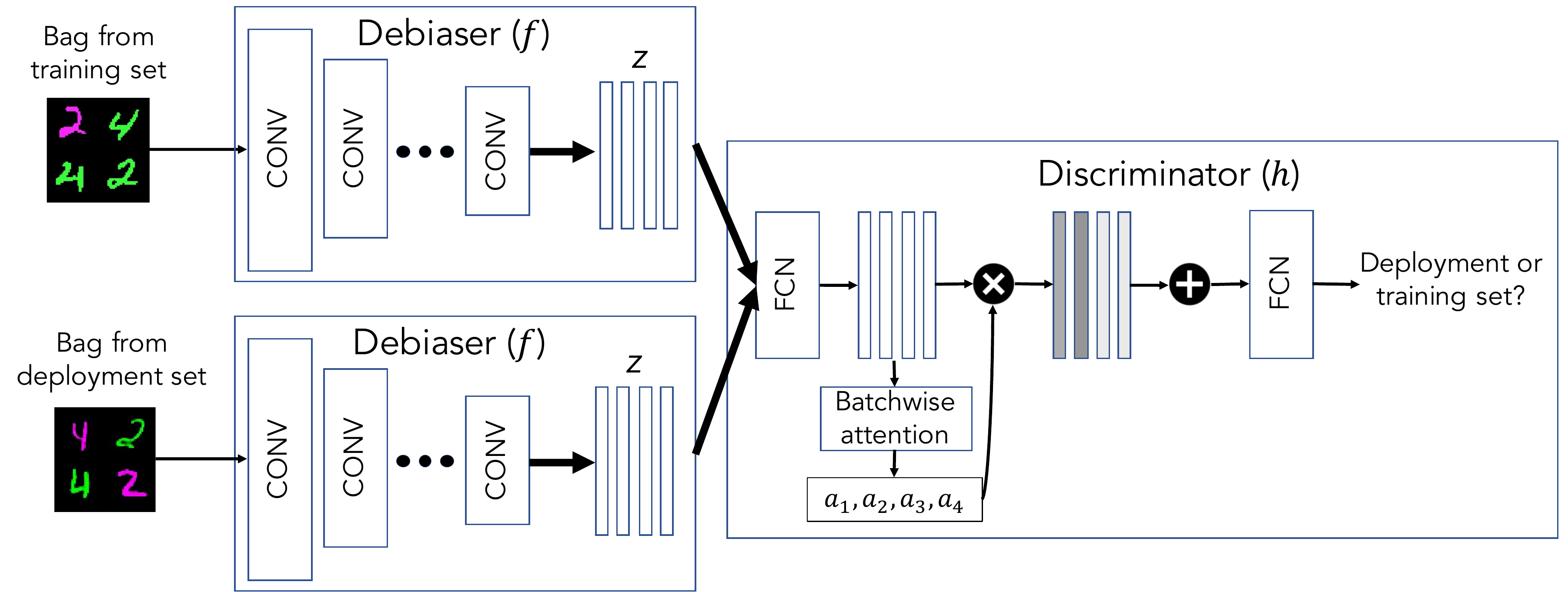}
    \caption{%
    The main components of our support-matching algorithm, $f$ (debiaser) and $h$ (discriminator). The debiaser is trained to produce encodings, $z$ of the data that are invariant to the subgroups differences. %
    In order to determine whether a bag of encodings originates from the training set or the deployment set, the discriminator performs an attention-weighted aggregation over the bag dimension to model interdependencies between the samples. In the case of Colored MNIST where {\color{purple}purple} fours constitute the missing subgroup, the discriminator can identify an encoding of a bag from the training set by the absence of such samples as long as color information is detectable in $z$. %
    By learning a subgroup invariant representation, the debiaser can hide the origin of the bags from the discriminator.}
    \label{fig:architecture}
\end{figure*}%
We give a more detailed explanation of the model used in our method.
Fig.~\ref{fig:architecture} shows the core of our method:
the debiaser, $f$, which produces bags of encodings, $z$
-- on both the training and the deployment set --
which are then fed to a discriminator that tries to identify the origin of the bags.
The discriminator uses batch-wise attention in order to consider a bag as a whole,
which allows cross-comparisons.

\subsection{Details of the attention mechanism}\label{ssec:attention-mechanism}
The \emph{discriminator} function $h$ that predicts which dataset a bag of samples embedded in $z$ was sampled from should have the following property: $h((f(x) | x \in B)) = h((f(x) | x \in \pi (B)))$ for all permutations $\pi$, and $f: x \rightarrow z$.
For the entirety of function $h$ -- composed of sub-functions $h_1(h_2(h_3...)))$ -- to have this property, it suffices that only the innermost sub-function, $\rho$, does.
While there are a number of choices when it comes to defining $\rho$, we choose a weighted average $\rho = \frac{1}{|\mathcal{B}|} \sum_{x \in B}\mathrm{attention}(f(x), B) \cdot f(x)$, with weights computed according to a learned attention mechanism.
The idea of using an attention mechanism for set-wise classification has been previously explored to great success by, e.g., \cite{lee2019set}. %
We experiment with two different types of attention mechanisms for the bag-wise pooling layer of our discriminator, finding both to work well: 1) the gated attention mechanism proposed by \cite{ilse2018attention}; 2) the scaled dot-product attention per \cite{vaswani2017attention}.
For the latter, we define $K$ and $V$ to be  $z$, and $Q$ to be the mean of $z$ over $\mathcal{B}$.
The output of $\rho$ is then further processed by a series of fully-connected layers and the final output is the binary prediction for a given bag of samples.

\subsection{Training procedure and hyperparameters}
\begin{table*}[tp]
 \centering
 \caption{Selected hyperparameters for experiments with Colored MNIST, Adult and CelebA datasets.}
 \label{tab:hparams}
 \scalebox{0.8}{
 \begin{tabular}{llll}
 \toprule
 & \textbf{Colored MNIST} & \textbf{Adult} & \textbf{CelebA}       
 \\ & 2-dig SB / 2-dig MS / 3-dig SB
 \\ \midrule
 Input size  &   $3 \times 32 \times 32$ & $61$ & $3 \times 64 \times 64$ \\  \midrule
 \multicolumn{4}{c}{AutoEncoder}                     \\ \midrule
 Levels                      & $4$         & $1$    & $5$\\
 Level depth                 & $2$         & $1$    & $2$\\
 Hidden units / level        & $[32, 64, 128, 256]$ & $[61]$ & $[32, 64, 128, 256, 512]$\\
 Activation                  & GELU        & GELU   & SiLU  \\
 Layer-wise Normalization               & -           & -      & LayerNorm \\
 Downsampling op.  & Strided Convs. & -- & Strided Convs.\\
 Reconstruction loss         & MSE         & Mixed$^1$  & MSE \\
 Learning rate               & $1 \times 10^{-3}$   & $1 \times 10^{-3}$  & $1 \times 10^{-3}$ \\ \midrule
 \multicolumn{4}{c}{Clustering}                      \\ \midrule
 Batch size                  & $256$      & $1000$  & $256$ \\
 AE pre-training epochs      & $150$        & $100$ & $10$  \\
 Clustering epochs           & $100$       & $300$  & $20$ \\
 Self-supervised loss & Cosine + BCE & Cosine + BCE & Cosine \\
 U (for ranking statistics)             & $5$         & $3$     & $8$    \\   \midrule
 \multicolumn{4}{c}{Support-Matching}                   \\ \midrule
 Batch size & $1$/$32$/$14$  & $64$   & $32$ \\
 Bag size   & $256$/$8$/$18$ & $32$ & $8$ \\
 Training iterations    & $8\text{k}/8\text{k}/20\text{k}$ & $5\text{k}$ & $2\text{k}$ \\
 Encoding ($z$) size$^2$  & $128$   & $35$  & $128$ \\
 Binarized $\tilde{s}$ & \xmark\, / \cmark\, / \cmark & \xmark & \xmark \\
 $y$-predictor weight ($\lambda_1$) & $1$ & $0$ & $1$  \\ 
 $s$-predictor weight ($\lambda_2$) & $1$ & $0$ & $1$  \\ 
 Adversarial weight ($\lambda_3$)   & $1 \times
 10^{-3}$   & $1$   & $1$\\ 
 Stop-gradient $\left(\nabla_\theta h_\psi(f_\theta(X^\mathit{dep}))=0\right)$ & \xmark & \cmark & \xmark \\
 \midrule
 \multicolumn{4}{c}{Predictors}   \\ \midrule
 Learning rate  & $3 \times 10^{-4}$ &   $1 \times 10^{-3}$  $ 1 \times 10^{-3}$\\
 \midrule
 \multicolumn{4}{c}{Discriminator}                   \\ \midrule
 Attention mechanism$^3$    & Gated   & Gated & Gated \\
 Hidden units pre-aggregation  & $[256, 256]$  & $[32]$ & $[256, 256]$\\
 Hidden units post-aggregation & $[256, 256]$ & --  & $[256, 256]$ \\
 Embedding dim (for attention) & $32$ & $128$ & $128$ \\
 Activation & GELU & GELU & GELU \\
 Learning rate  & $3 \times 10^{-4}$    & $1 \times 10^{-3}$ & $1 \times 10^{-3}$\\
 Updates / AE update    & $1$  & $3$    & $1$    \\
 \bottomrule
 \addlinespace
 \multicolumn{4}{p{17cm}}{\footnotesize $^1$ Cross-entropy is used for categorical features, MSE for continuous features.} \\
 \multicolumn{4}{p{17cm}}{\footnotesize $^2$ $|z|$ denotes the combined size of $\tilde{s}$ and $z$, with the former occupying $\ceil{\text{log}_2(\mathcal{S})}$ dimensions, the latter the remaining dimensions.} \\
 \multicolumn{4}{p{17cm}}{
 \footnotesize $^3$ 
 The attention mechanism used for computing the sample-weights within a bag.
 \emph{Gated} refers to gated attention  proposed by \cite{ilse2018attention}, while \emph{SDP} refers to the scaled dot-product attention proposed by \cite{vaswani2017attention}.
 }
 \end{tabular}
 }
 ~\\
 ~\\
 ~\\
\end{table*}

The hyperparameters and architectures for the AutoEncoder (\texttt{AE}), Predictor and Discriminator subnetworks are detailed in Table \ref{tab:hparams} for all three datasets.We train all models using \texttt{Adam} \cite{KinBa15}.

For the Colored MNIST and CelebA datasets, the baseline \texttt{ERM}, \texttt{DRO}, \texttt{LfF} (in the case of the former) and \texttt{gDRO} (in the case of the latter) models use a convolutional backbone consisting of one Conv-BN-LReLU block per ''stage``, with each stage followed by max-pooling operation to spatially downsample by a factor of two to produce the subsequent stage. This backbone consists of 4 and 5 stages for Colored MNIST and CelebA, respectively.
The output of the backbone is flattened and fed to a  single fully-connected layer of size $|Y|$ in order to obtain the class-prediction, $\hat{y_i}$, for a given instance. To evaluate our method, we simply train a linear classifier on top of $z$; this is sufficient due to linear-separability being encouraged during training by the $y$-predictor.
For the Adult Income dataset, we use an MLP composed of a single hidden layer 35 units in size, followed by a SELU activation \cite{klambauer2017self}, as both the downstream classifier for our method, and as the network architecture of the baselines. 
All baselines and downstream classifiers alike were trained for $60$ epochs with a learning rate of $1 \times 10^{-3}$ and a batch size of $256$.

Since, by design, we do not have labels for all subgroups the model will be tested on, and bias against these missing subgroups is what we aim to combat,
properly validating, and thus conducting hyperparameter selection for models generally, is not straightforward.
Indeed, performing model-selection for domain generalization problems is well-known to be a difficult problem \cite{gulrajani2021search}.
We can use estimates of the mutual information between the learned-representation and $s$ and $y$ (which we wish to minimize w.r.t.\ to the former, maximize w.r.t.\ the latter) to guide the process, though optimizing the model w.r.t.\ to these metrics obtained from only the training set does not guarantee generalization to the missing subgroups.
We can, however, additionally measure the entropy of the predictions on the encoded test set and seek to maximize it across all samples, or alternatively train a discriminator of the same kind used for distribution matching as a measure the shift in the latent space between datasets.
We use the latter approach (considering the combination of the learned distance between subspace distributions, accuracy, and reconstruction loss) to inform an extensive grid-search over the hyperparameter space for our method.

For the \texttt{DRO} baseline, we allowed access to the labels of the test set for the purpose of hyperparameter selection, performing a grid-search over multiple splits to avoid overfitting to any particular instantiation.
Specifically, the threshold ($\eta$) parameter for \texttt{DRO} was determined by a grid-search over the space $\{0.01, 0.1, 0.3, 1.0\}$.

In addition to the losses stated in the support-matching objective, $\mathcal{L}$, in the main text, we also regularize the encoder by the $\ell^2$ norm of its embedding, multiplied by a small pre-factor, finding this to work better than more complex regularization methods, such as spectral normalization \cite{miyato2018spectral}, for stabilizing adversarial training.
\section{Additional analysis of results}\label{sec:additional-analysis}
\subsection{Visualizations of results}\label{sec:qual-results}
\begin{figure}[tp]
  \centering
  \begin{subfigure}[b]{0.49\columnwidth}
    \centering
    \includegraphics[width=\textwidth]{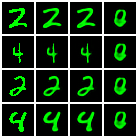}
    \caption{
    Different reconstructions on the training set.
    Corresponding to: original, full reconstruction, reconstruction of $z$ ($\tilde{s}$ zeroed out), reconstruction of $\tilde{s}$ ($z$ zeroed out).
    }%
    \label{fig:cmnist-recon-training}
  \end{subfigure}
   \hfill
  \begin{subfigure}[b]{0.49\columnwidth}
    \centering
    \includegraphics[width=\textwidth]{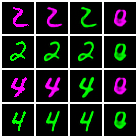}
    \caption{
    Different reconstructions on the deployment set.
    Corresponding to: original, full reconstruction, reconstruction of $z$ ($\tilde{s}$ zeroed out), reconstruction of $\tilde{s}$ ($z$ zeroed out).
    }%
    \label{fig:cmnist-recon-deployment}
  \end{subfigure}
  \caption{
   Visualization of our method's solutions for the Colored MNIST dataset, with {\color{purple}purple} as the missing subgroup.
   In each of the subfigures \ref{fig:cmnist-recon-training} and \ref{fig:cmnist-recon-deployment}:
   Column 1 shows the original images from $x$ from the respective set.
   Column 2 shows plain reconstructions generated from $x_\textit{recon}=g(f(x), t(x))$.
   Column 3 shows reconstruction with zeroed-out $\tilde{s}$: $g(f(x), 0)$, which effectively visualizes $z$.
   Column 4 shows the result of an analogous process where $z$ was zeroed out instead.
  }%
  \label{fig:cmnist-recon}
\end{figure}%
\begin{figure}[tp]
  \centering
  \begin{subfigure}[b]{0.49\columnwidth}
    \centering
    \includegraphics[width=\textwidth]{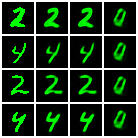}
    \caption{
    Different reconstructions on the training set.
    Corresponding to: original, full reconstruction, reconstruction of $z$ ($\tilde{s}$ zeroed out), reconstruction of $\tilde{s}$ ($z$ zeroed out).
    }%
    \label{fig:cmnist-recon-training-failure}
  \end{subfigure}
   \hfill
  \begin{subfigure}[b]{0.49\columnwidth}
    \centering
    \includegraphics[width=\textwidth]{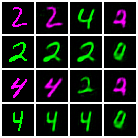}
    \caption{
    Different reconstructions on the deployment set.
    Corresponding to: original, full reconstruction, reconstruction of $z$ ($\tilde{s}$ zeroed out), reconstruction of $\tilde{s}$ ($z$ zeroed out).
    }%
    \label{fig:cmnist-recon-deployment-failure}
  \end{subfigure}
  \caption{
   Visualization of a failure of our method for the Colored MNIST dataset, with {\color{purple}purple} as the missing subgroup.
   In each of the subfigures \ref{fig:cmnist-recon-training-failure} and \ref{fig:cmnist-recon-deployment-failure}:
   Column 1 shows the original images from $x$ from the respective set.
   Column 2 shows plain reconstructions generated from $x_\textit{recon}=g(f(x), t(x))$.
   Column 3 shows reconstruction with zeroed-out $\tilde{s}$: $g(f(x), 0)$, which effectively visualizes $z$.
   Column 4 shows the result of an analogous process where $z$ was zeroed out instead.
  }%
  \label{fig:cmnist-recon-failure}
\end{figure}%
\begin{figure}[htp]
     \centering
     \includegraphics[width=0.5\textwidth]{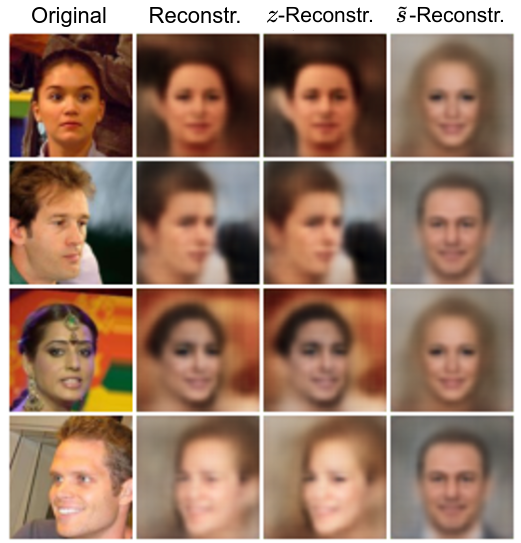}
     \caption{%
     Visualization of our method's solutions for the CelebA dataset, with ``smiling females'' as the missing subgroup.
     Column 1 shows the original images from $x$ from the deployment set of CelebA.
     Column 2 shows plain reconstructions generated from $x_\textit{recon}=g(f(x), t(x))$.
     Column 3 shows reconstruction with zeroed-out $\tilde{s}$: $g(f(x), 0)$, which effectively visualizes $z$.
     Column 4 shows the result of an analogous process where $z$ was zeroed out instead.
     }%
     \label{fig:celeba-recons}
\end{figure}%
\begin{figure*}[htp]
  \centering
    \begin{subfigure}[b]{0.49\textwidth}
    \includegraphics[width=\textwidth]{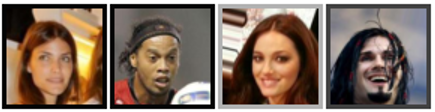}
    \end{subfigure}
    \hfill
    \begin{subfigure}[b]{0.49\textwidth}
    \includegraphics[width=\textwidth]{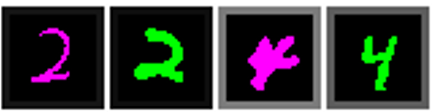}
    \end{subfigure}
  \caption{
    Example sample-wise attention maps for bags of CelebA (left) and CMNIST (right) images sampled from a balanced deployment set. The training set is biased according to the \emph{subgroup bias} setting where for CelebA ``smiling females'' constitute the missing source and for Colored MNIST {\color{purple}purple} fours constitute the missing source. The attention weights are used during the discriminator's aggregation step to compute a weighted sum over the bag. The attention-weight assigned to each sample is proportional to the lightness of its frame, with black signifying a weight of 0, white a weight of 1. Those samples belonging to the missing subgroup are assigned the highest weight as they signal from which dataset (training vs. deployment) the bag containing them was drawn from.
  }%
  \label{fig:attn_maps}
\end{figure*}%
We show qualitative results of the disentangling in figures \ref{fig:cmnist-recon}, \ref{fig:cmnist-recon-failure} (both Colored MNIST), and \ref{fig:celeba-recons} (CelebA).
Fig.~\ref{fig:cmnist-recon} shows successful disentangling (from a run that achieved close to 100\% accuracy);
in the deployment set the representation $z$ has lost all coloring (see column 3 in the figures).
Fig.~\ref{fig:cmnist-recon-failure} on the other hand, shows a visualization from a \emph{failed} run;
Instead of encoding purple 2's and green 2's with the same representation,
the model here encoded purple 2's and green 4's as similar.
This is a valid solution of the given optimization problem
-- the representation is invariant to training set vs deployment set --
but it is definitely not the intended solution.

Fig.\ref{fig:celeba-recons} shows visualizations for CelebA.
With a successful disentangling, column 3 (visualization of $z$) should show a version of the image that is ``gender-neutral'' (i.e., invariant to gender).
Furthermore, column 4 (visualization of $\tilde{s}$) should be invariant to the class label (i.e., ``smiling''), so the images should either be all with smiles or all without smiles.

Fig.~\ref{fig:attn_maps} shows attention maps for bags from the deployment set. We can see that the model pays special attention to those samples that are not included in the training set.
For details, see the captions.

\subsection{Additional metrics}\label{sec:additional-metrics}
Figures~\ref{fig:cmnist-2v4-partial-add}, \ref{fig:cmnist-2v4-miss-s-add},  and \ref{fig:celeba-gender-smiling-add} show the true positive rate (TPR) ratio and the true negative rate (TNR) ratio as additional metrics for Colored MNIST (2 digits) and CelebA.
These are computed as the ratio of TPR (or TNR) on subgroup $s=0$ over the TPR (or TNR) on subgroup $s=1$; if this gives a number greater than 1, the inverse is taken.
Similarly to the PR ratio reported in the main paper, these ratios give an indication of how much the prediction of the classifier depends on the subgroup label $s$.

Figure~\ref{fig:cmnist-3dig-4miss-add} shows metrics specific to multivariate $s$ (i.e., non-binary $s$).
We report the minimum (i.e. farthest away from 1) of the pairwise ratios (TPR/TNR ratio min) as well as the largest difference between the raw values (TPR/TNR diff max).
Additionally, we compute the Hirschfeld-Gebe\-lein-R\'enyi (HGR) maximal correlation \cite{renyi1959measures} between $S$ and $Y$, serving as a measure of dependence defined between two variables with arbitrary support.
\begin{figure*}[htp]
  \centering
  \includegraphics[width=0.49\textwidth]{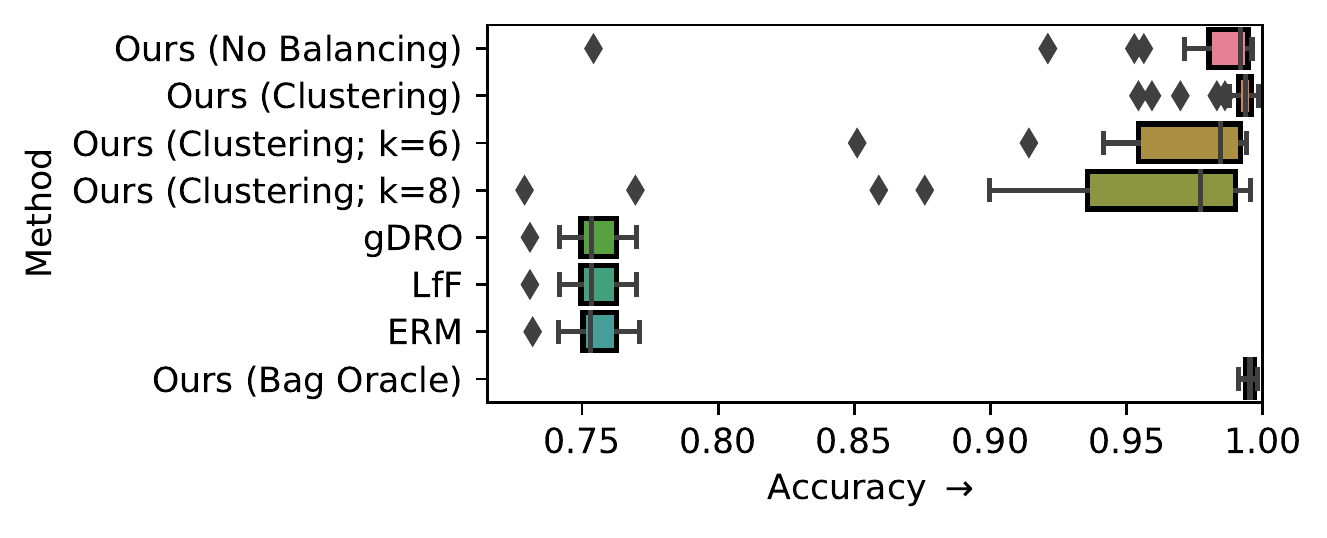}
  \includegraphics[width=0.49\textwidth]{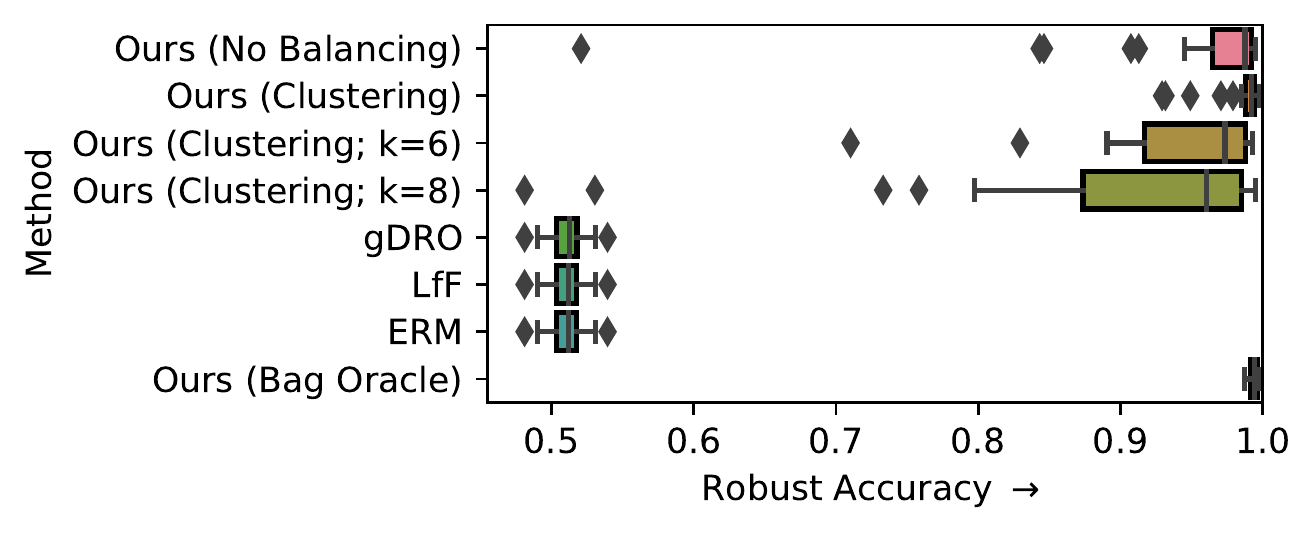}
  \includegraphics[width=0.49\textwidth]{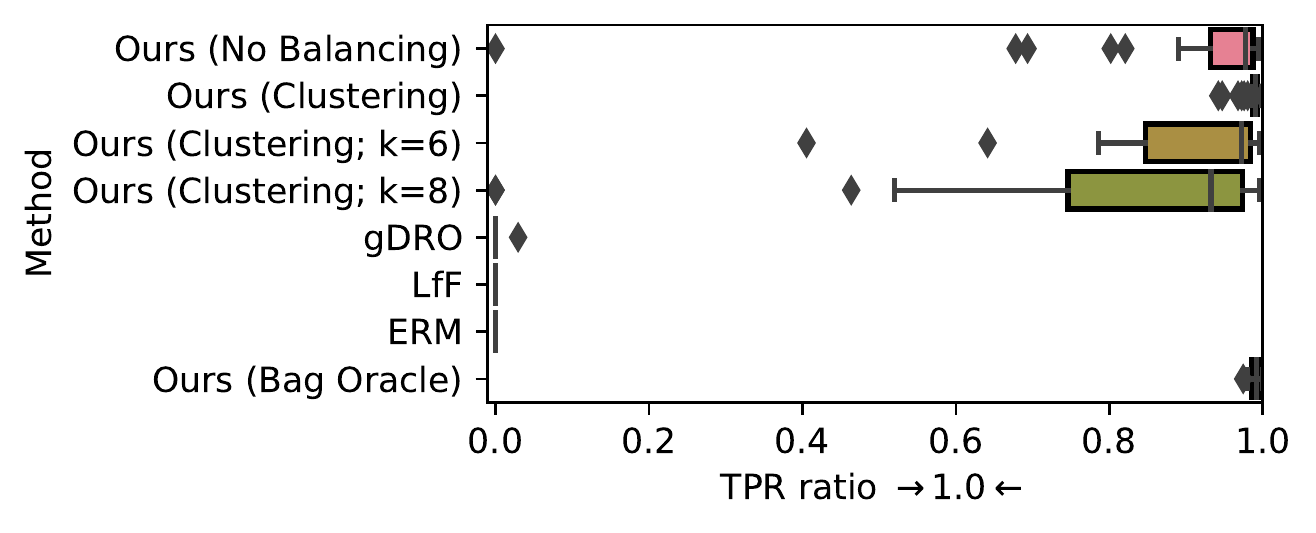}
    \includegraphics[width=0.49\textwidth]{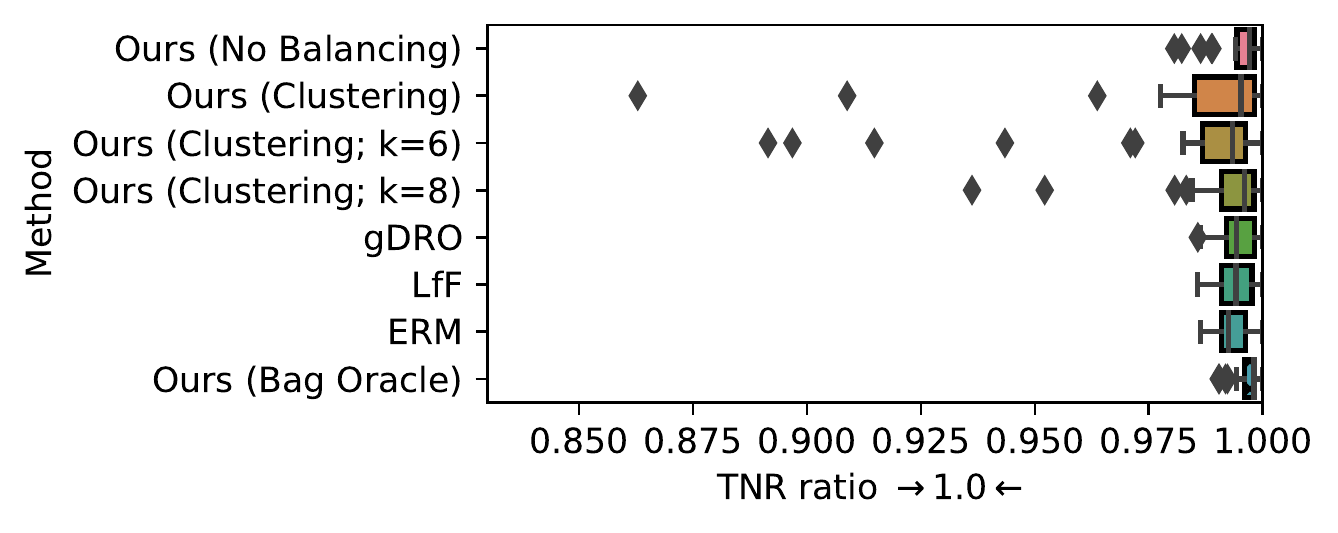}
  \caption{
    Results from \textbf{30 repeats} for the Colored MNIST dataset with two digits, 2 and 4, with \emph{subgroup bias} for the color `{\color{purple}purple}': for {\color{purple}purple}, only the digit class `2' is present.
    \textbf{Top left}: Accuracy.
    \textbf{Top right}: Positive rate ratio.
    \textbf{Bottom left}: True positive rate ratio.
    \textbf{Bottom right}: True negative rate ratio.
    For \texttt{Ours (Clustering)}, the clustering accuracy was 96\% $\pm$ 6\%.
    For an explanation of \texttt{Ours (Clustering; k=6/8)} see section~\ref{sec:overclustering}.
  }%
  \label{fig:cmnist-2v4-partial-add}
\end{figure*}
\begin{figure*}[htp]
  \centering
  \includegraphics[width=0.49\textwidth]{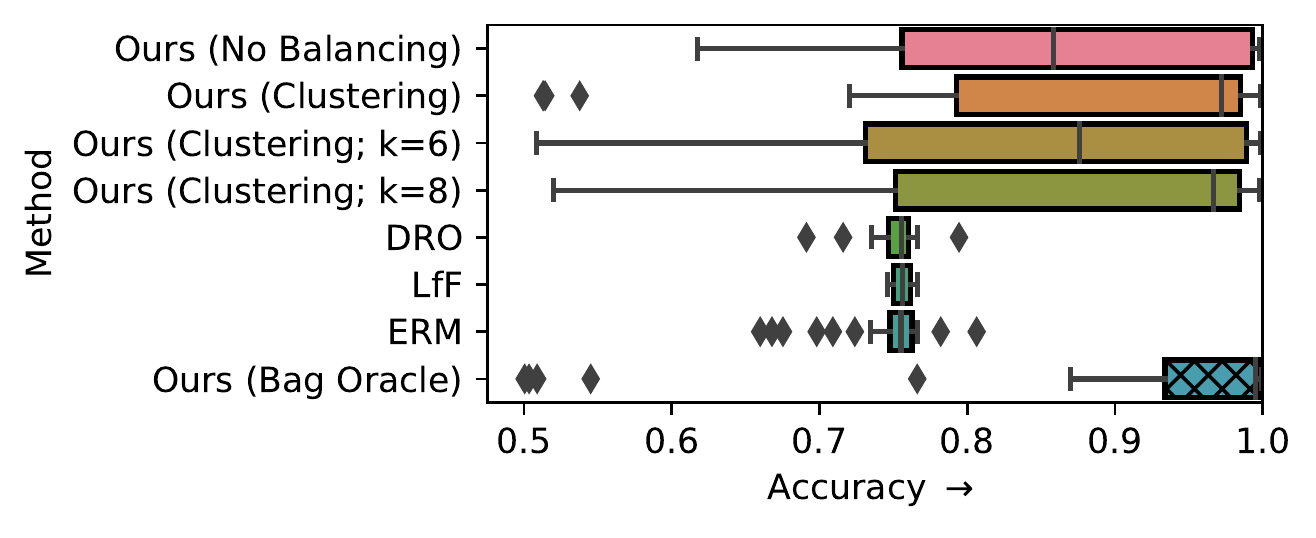}
  \includegraphics[width=0.49\textwidth]{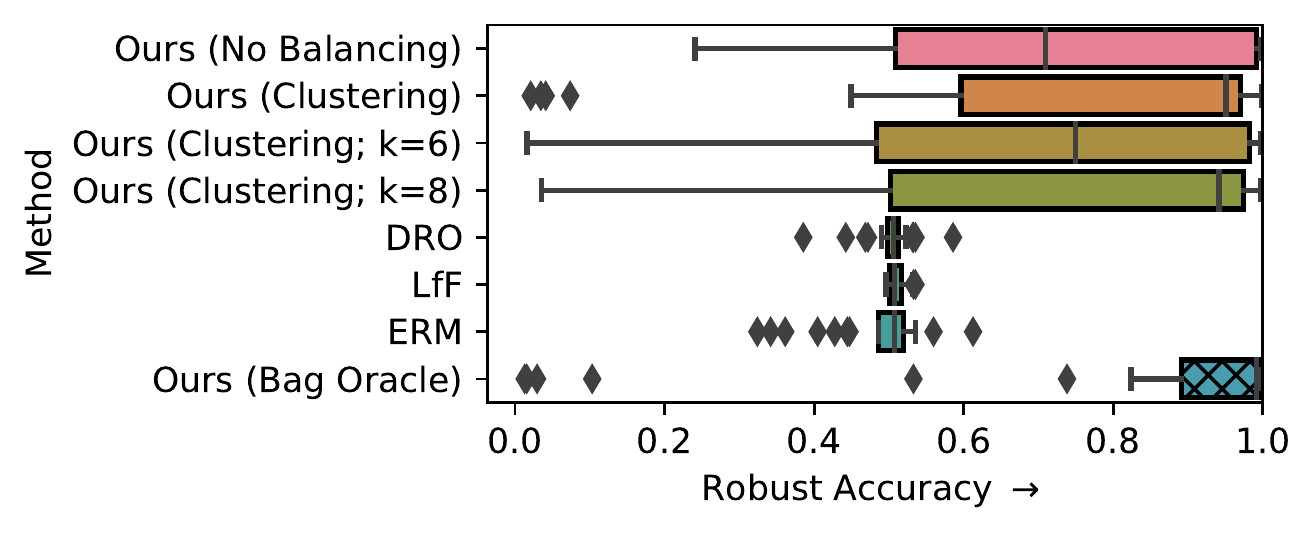}
  \includegraphics[width=0.49\textwidth]{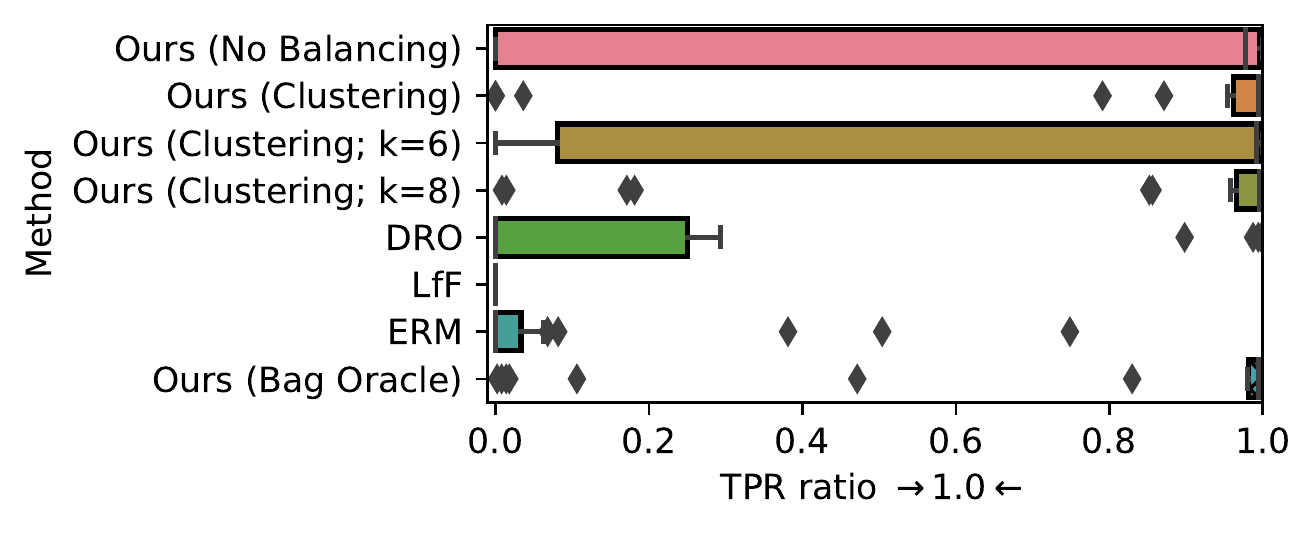}
  \includegraphics[width=0.49\textwidth]{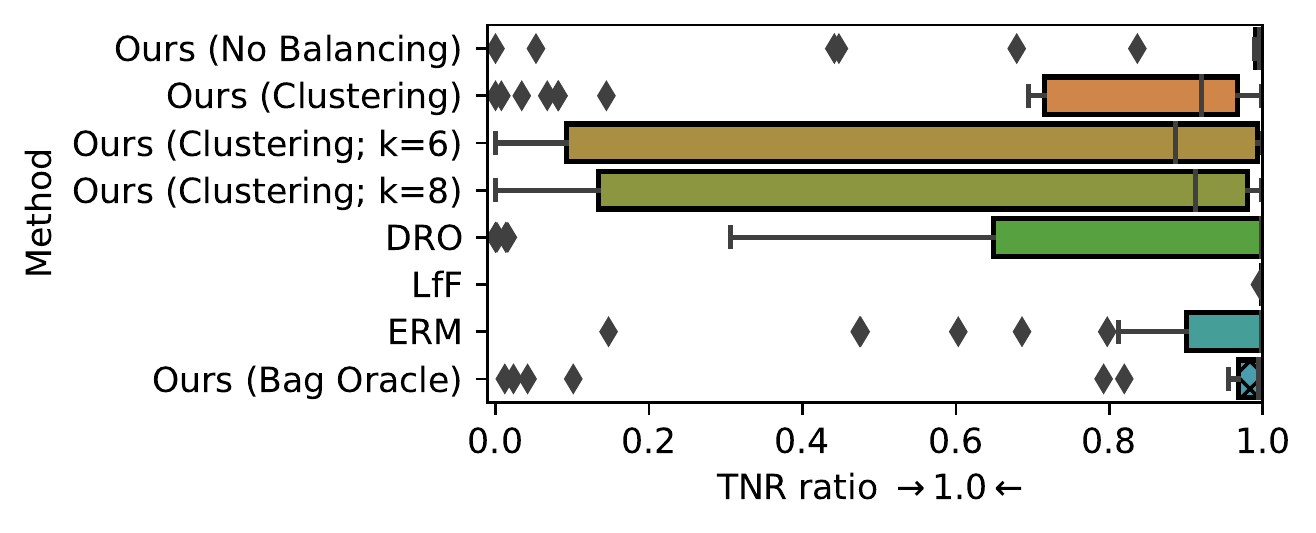}

  \caption{
    Results from \textbf{30 repeats} for the Colored MNIST dataset with two digits, 2 and 4, with a \emph{missing subgroup}: the training dataset only has {\color{green}green} digits.
    \textbf{Top left}: Accuracy.
    \textbf{Top right}: Robust Accuracy.
    \textbf{Bottom left}: True positive rate ratio.
    \textbf{Bottom right}: True negative rate ratio.
    For \texttt{Ours (Clustering)}, the clustering accuracy was 88\% $\pm$ 5\%.
    For an explanation of \texttt{Ours (Clustering; k=6/8)} see section~\ref{sec:overclustering}.
  }%
  \label{fig:cmnist-2v4-miss-s-add}
\end{figure*}

\begin{figure*}[htp]
  \centering
  \includegraphics[width=0.49\textwidth]{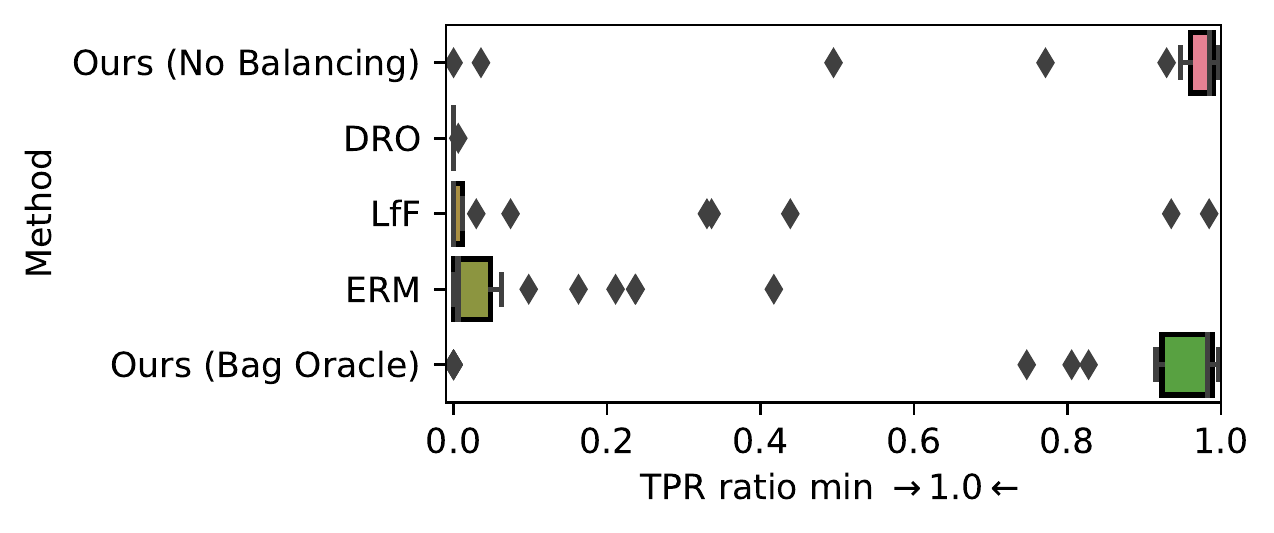}
  \includegraphics[width=0.49\textwidth]{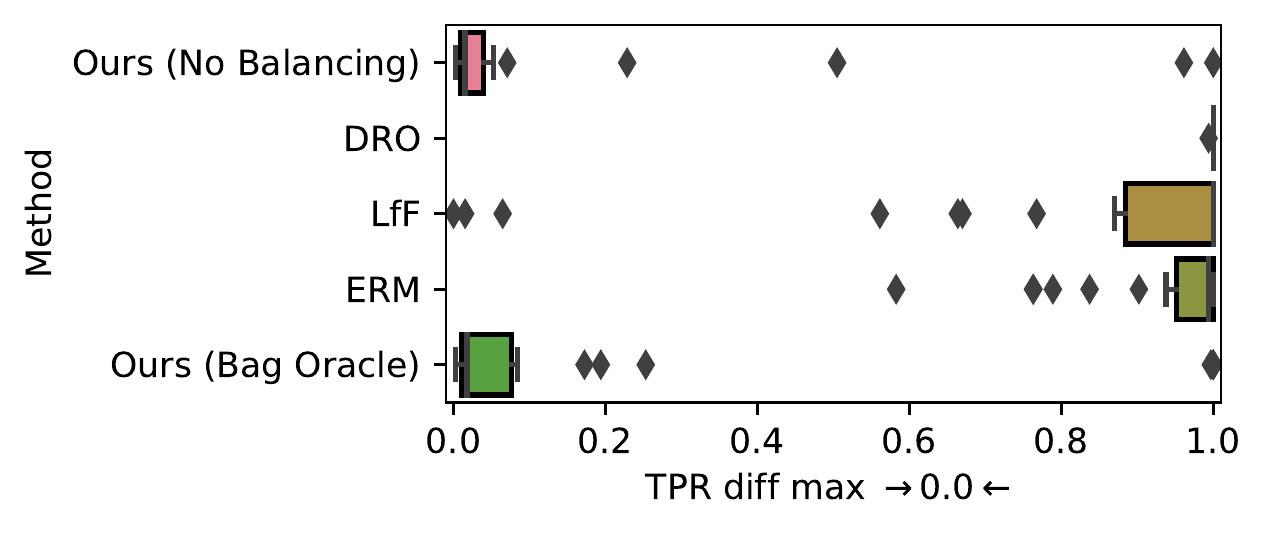}
  \includegraphics[width=0.49\textwidth]{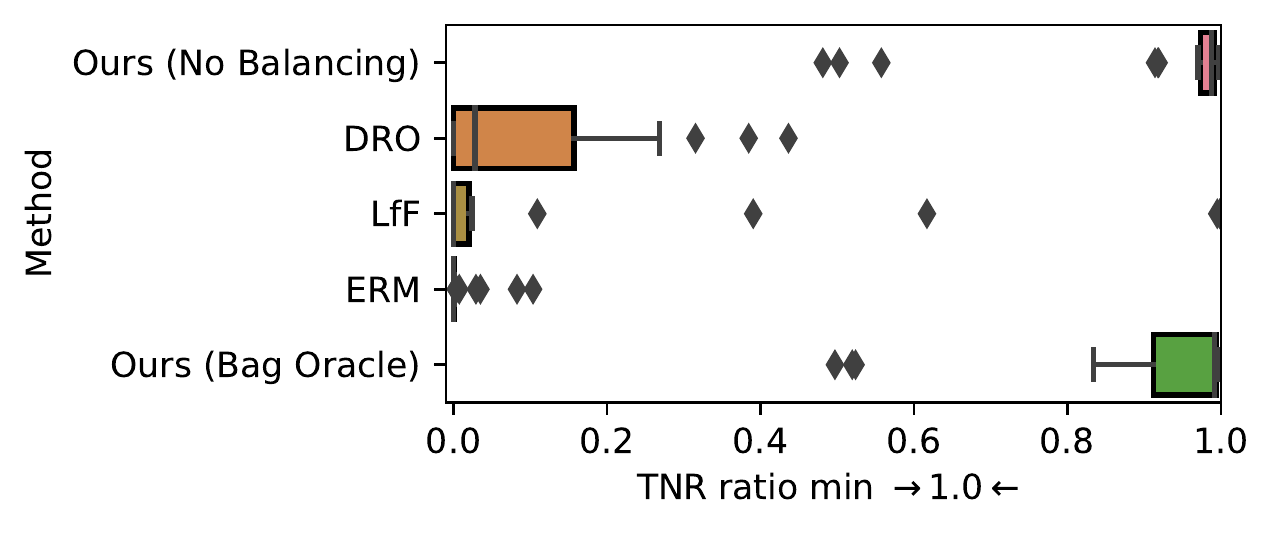}
  \includegraphics[width=0.49\textwidth]{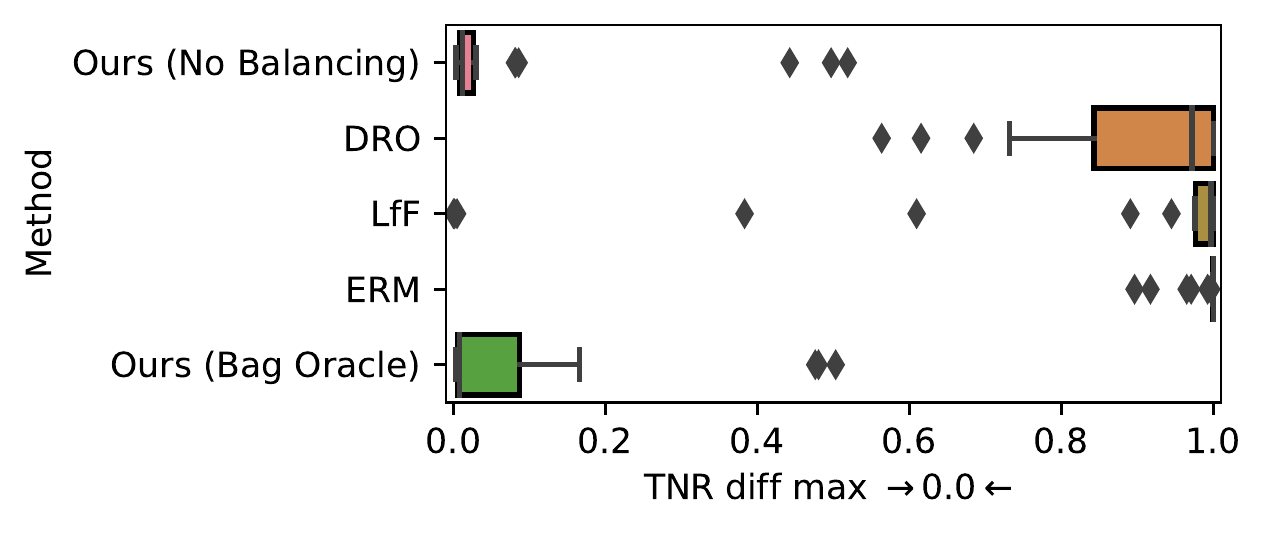}
  \caption{
    Results from \textbf{30 repeats} for the Colored MNIST dataset with three digits: `2', `4' and `6'.
    Four combinations of digit and color are missing: {\color{green}green} 2's, {\color{blue}blue} 2's, {\color{blue}blue} 4's and {\color{green}green} 6's.
    \textbf{First row, left}: minimum of all true positive rate ratios.
    \textbf{First row, right}: maximum of all true positive rate differences.
    \textbf{Second row, left}: minimum of all true negative rate ratios.
    \textbf{Second row, right}: maximum of all true negative rate differences.
  }%
  \label{fig:cmnist-3dig-4miss-add}
\end{figure*}
  
\begin{figure*}[t]
  \centering
 \includegraphics[width=0.49\textwidth]{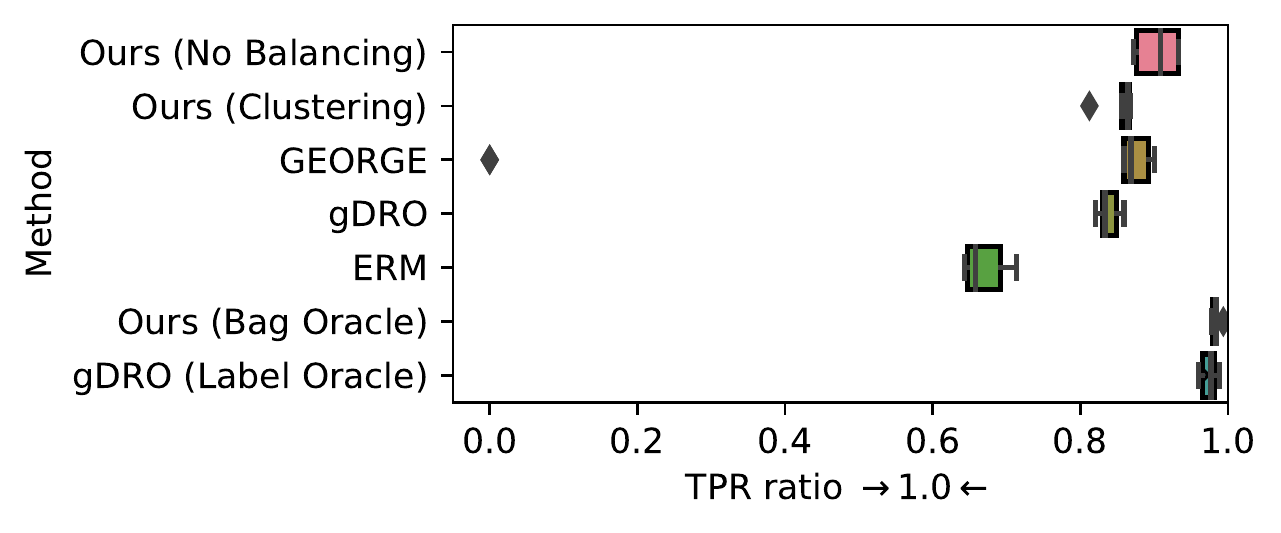}
 \includegraphics[width=0.49\textwidth]{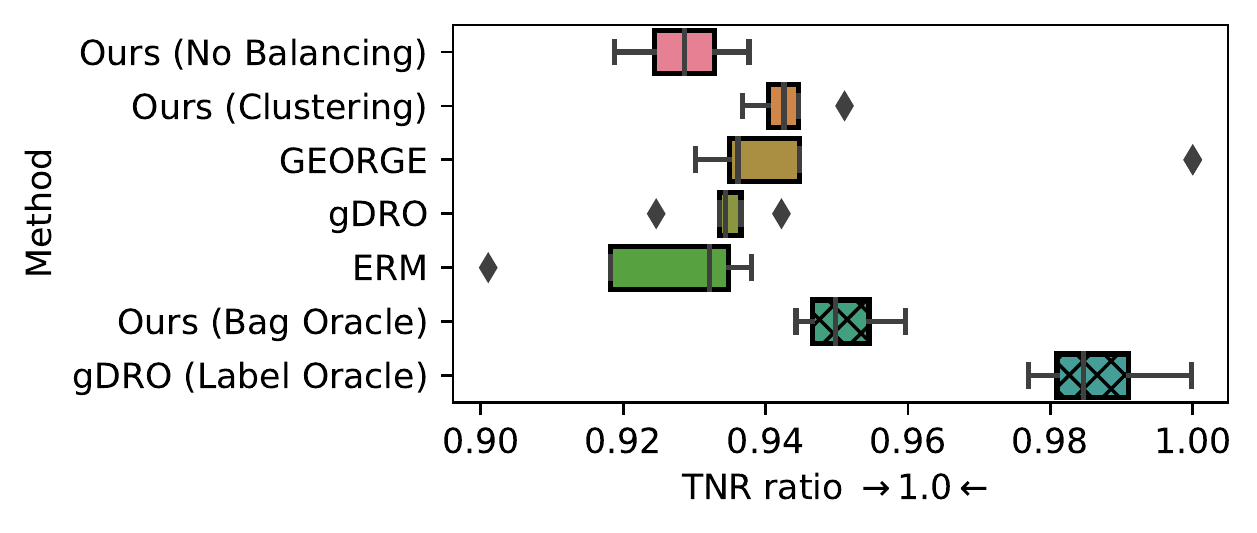}
 \includegraphics[width=0.49\textwidth]{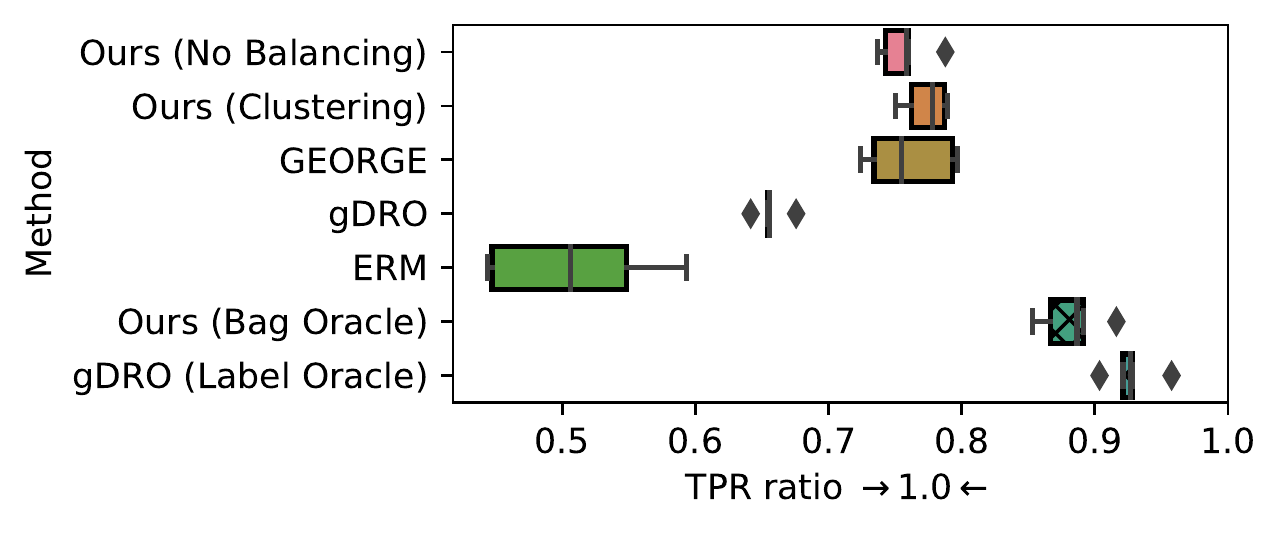}
 \includegraphics[width=0.49\textwidth]{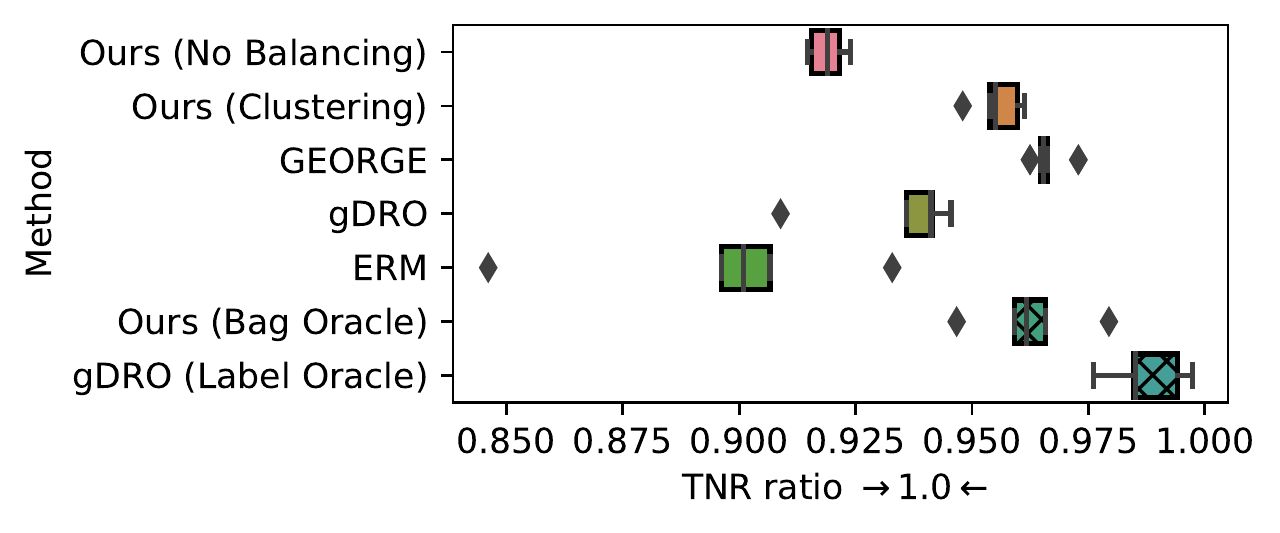}
 \includegraphics[width=0.49\textwidth]{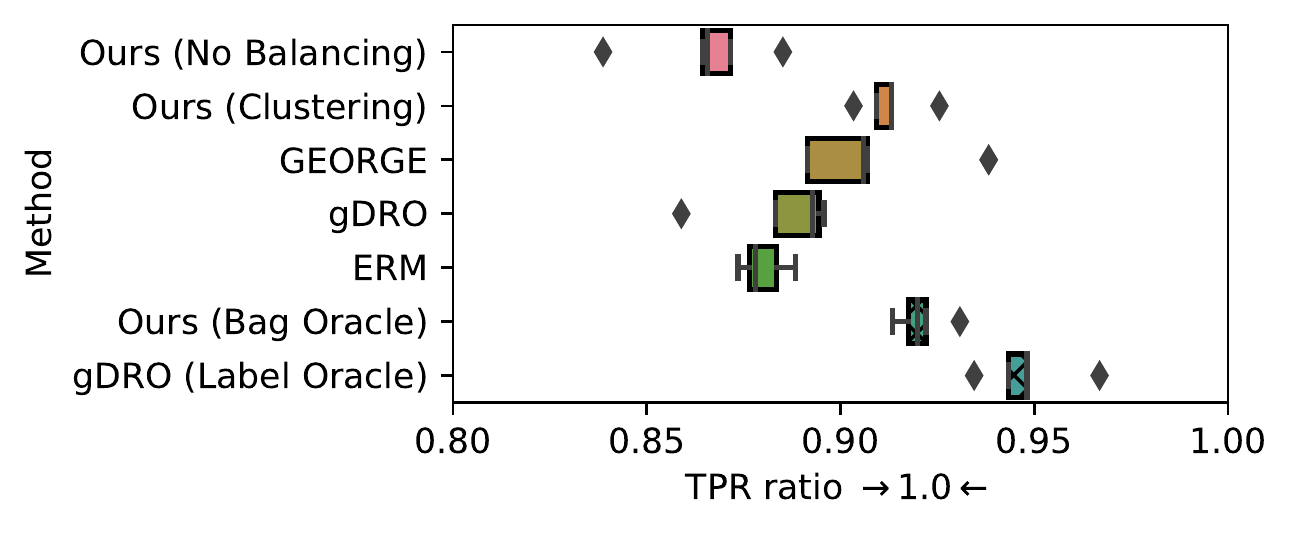}
 \includegraphics[width=0.49\textwidth]{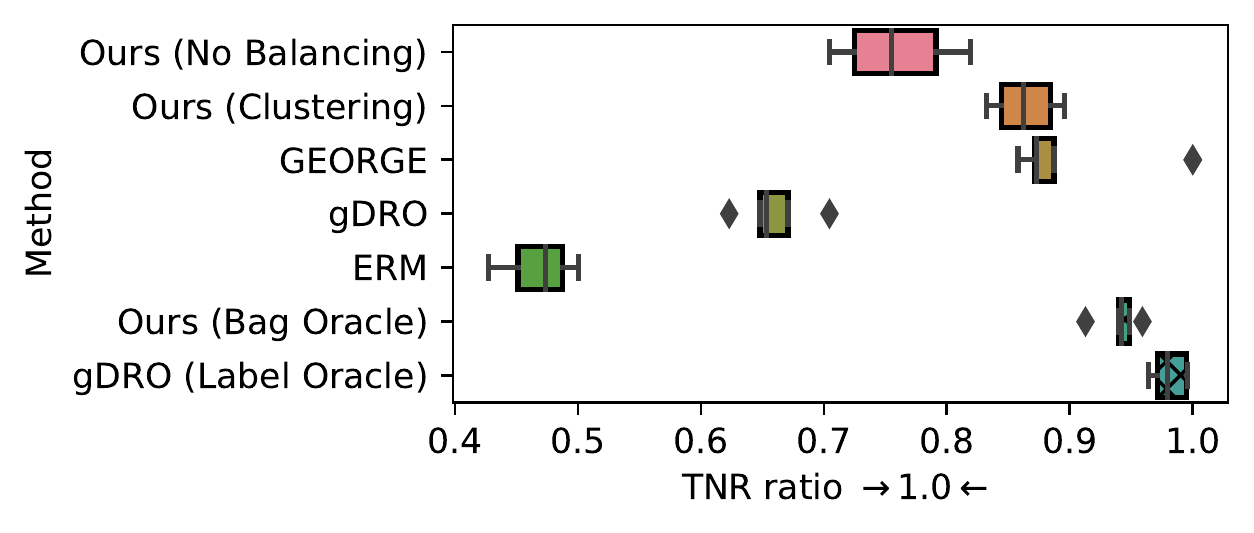}
 \includegraphics[width=0.49\textwidth]{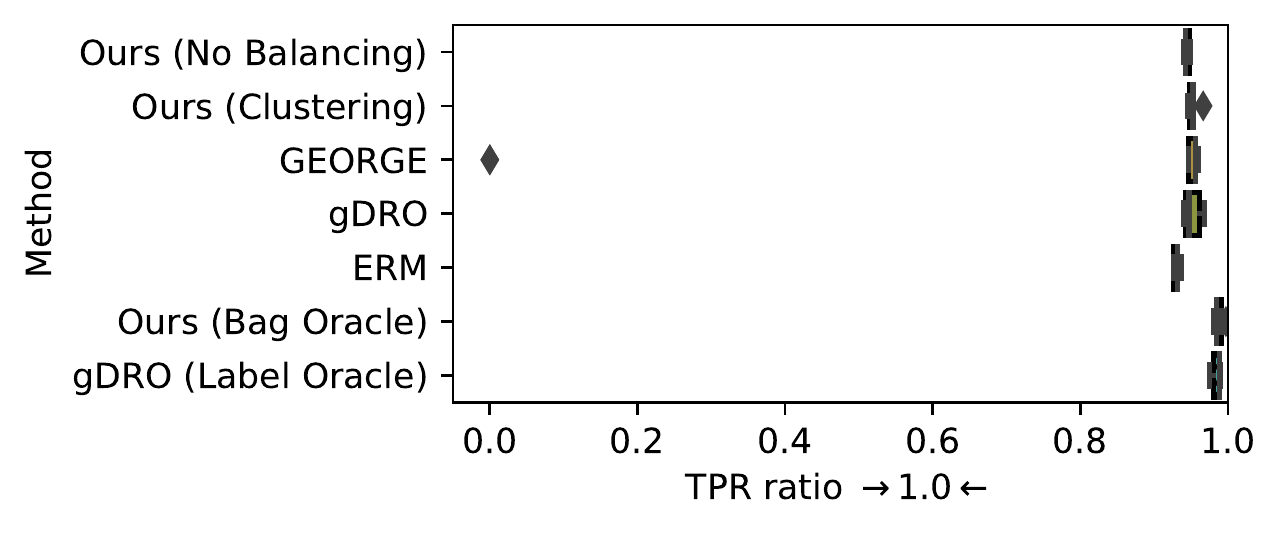}
 \includegraphics[width=0.49\textwidth]{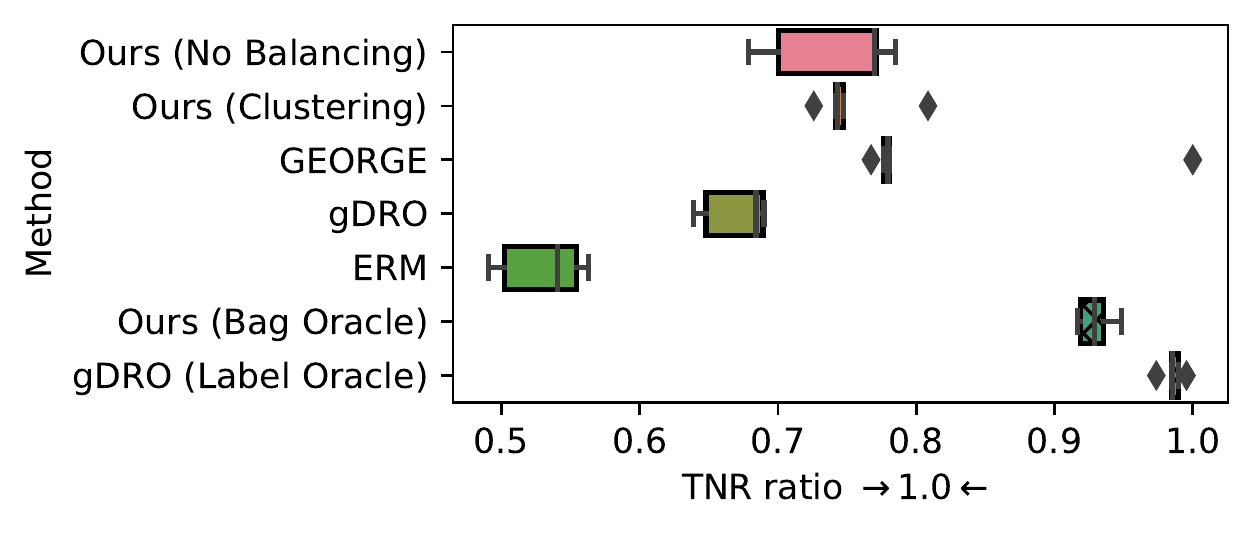}
 \caption{%
   Plots of additional metrics for CelebA under the SB setting, where ''smiling'' is the class label and ''gender'' is the subgroup label. These metrics are ratios computed between the \emph{Male} and \emph{Female} subgroups with the largest of the two values involved always selected as the denominator. \textbf{Left:} TNR ratio. \textbf{Right}: TNR ratio.
 }%
 \label{fig:celeba-gender-smiling-add}
\end{figure*}

\section{Ablation studies}\label{sec:ablations}
\subsection{Using an instance-wise loss instead of a set-wise loss}\label{ssec:no-mil}
\begin{figure*}[htp]
  \centering
  \includegraphics[width=0.49\textwidth]{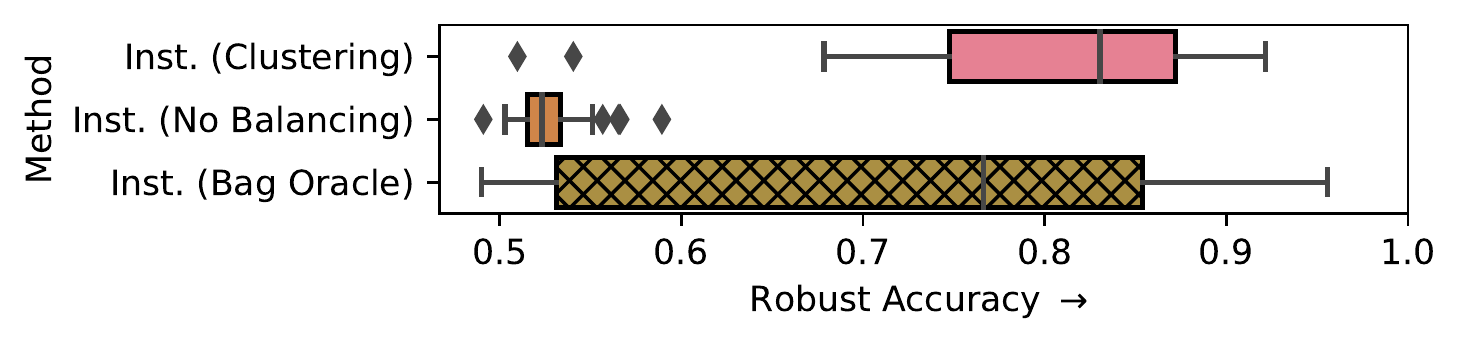}
  \includegraphics[width=0.49\textwidth]{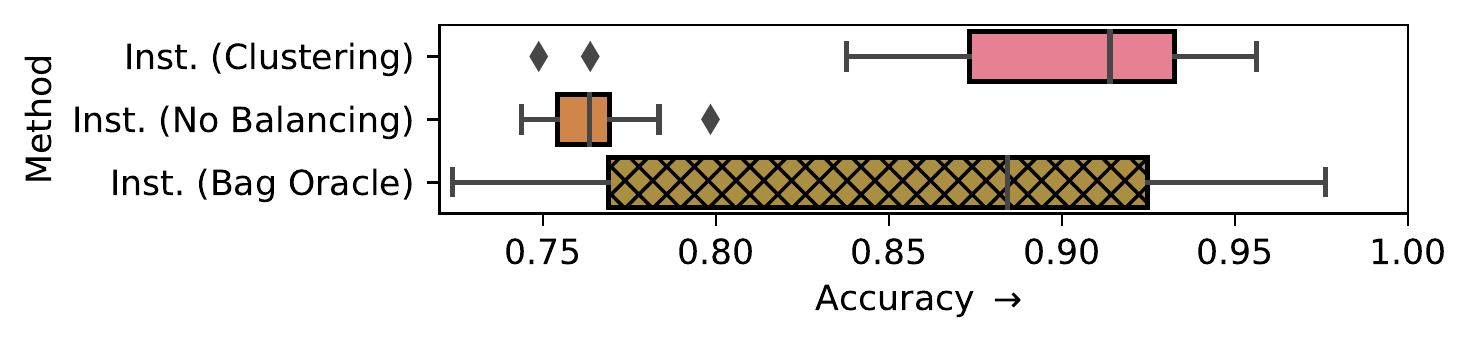}
  \caption{
    Results from \textbf{30 repeats} with an \emph{instance-wise} loss for the Colored MNIST dataset with two digits, 2 and 4, with \emph{subgroup bias} for the color `{\color{purple}purple}': for {\color{purple}purple}, only the digit class `2' is present.
    \textbf{Left}: Accuracy.
    \textbf{Right}: Positive rate ratio.
    For \texttt{Inst.\ (Clustering)}, the clustering accuracy was 96\% $\pm$ 6\%.
  }%
  \label{fig:cmnist-2v4-partial-add-nomil}
\end{figure*}
\begin{figure*}[htp]
  \centering
  \includegraphics[width=0.49\textwidth]{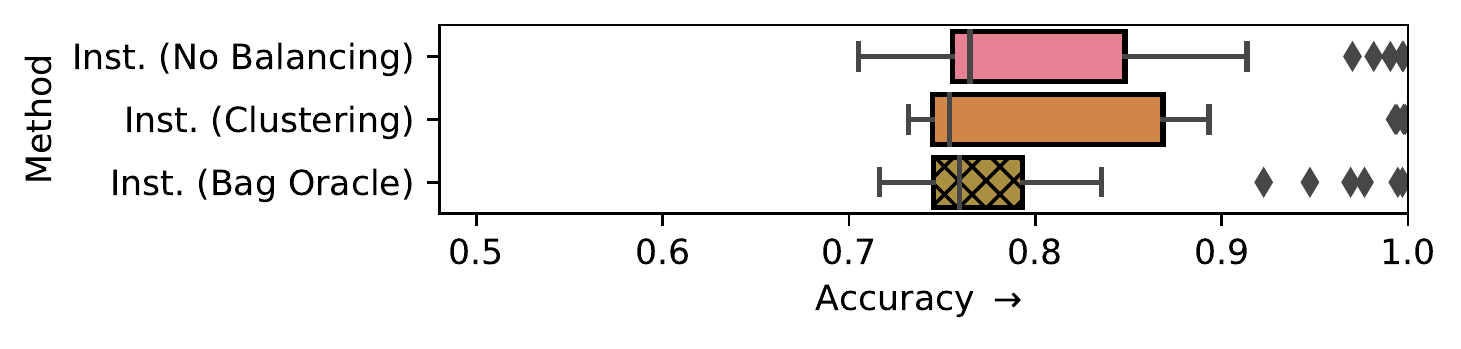}
  \includegraphics[width=0.49\textwidth]{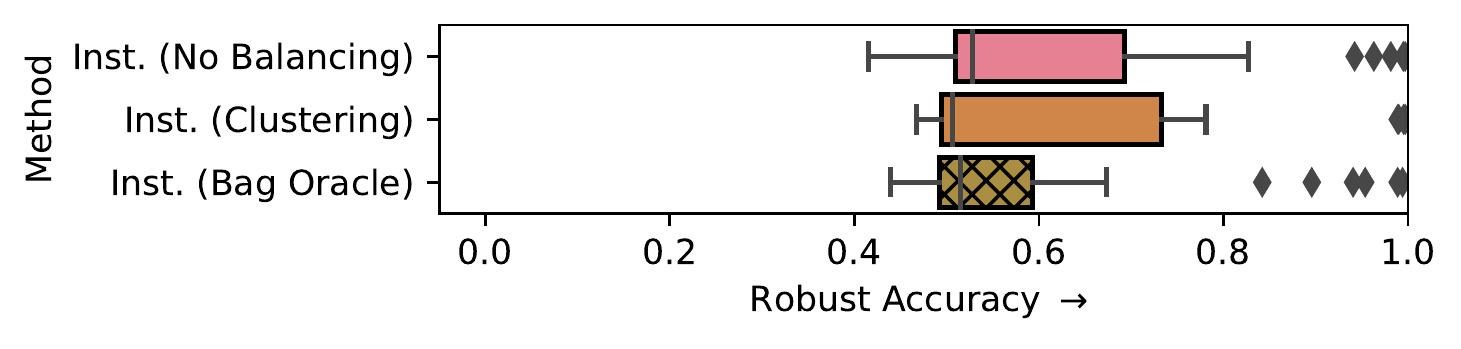}

  \caption{
    Results from \textbf{30 repeats} with an \emph{instance-wise} loss for the Colored MNIST dataset with two digits, 2 and 4, with a \emph{missing subgroup}: the training dataset only has {\color{green}green} digits.
    \textbf{Left}: Accuracy.
    \textbf{Right}: Robust Accuracy.
    For \texttt{Inst.\ (Clustering)}, the clustering accuracy was 88\% $\pm$ 5\%.
  }%
  \label{fig:cmnist-2v4-miss-s-add-nomil}
\end{figure*}
See Fig.~\ref{fig:cmnist-2v4-partial-add-nomil} and Fig.~\ref{fig:cmnist-2v4-miss-s-add-nomil}
for results on 2-digit Colored MNIST (under the \emph{subgroup bias} and \emph{missing subgroup} settings, respectively) for our method but with the loss computed instance-wise (\texttt{Inst.})\ as opposed to set-wise, as is typical of adversarial unsupervised domain adaptation methods (e.g. \cite{ganin2016domain}).
All aspects of the method, other than those directly involved in the loss-computation, were kept constant -- this includes the use of hierarchical balancing, despite the necessary removal of the aggregation layer meaning the discriminator is no longer sensitive to the bagging.
It is clear that the aforementioned change to the loss drastically increases the variance (IQR) of the results for both settings and, at the same time, drastically reduces the median \texttt{Robust Accuracy} to the point of being only marginally above that of the \texttt{ERM} baseline, regardless of the chosen balancing scheme.

\subsection{Clustering with an incorrect number of clusters}\label{sec:overclustering}
We also investigate what happens when the number of clusters is set incorrectly.
For 2-digit Colored MNIST, we expect 4 clusters, corresponding to the 4 possible combinations of the binary class label $y$ and the binary subgroup label $s$.
However, there might be circumstances where the correct number of clusters is not known; how does the batch balancing work in this case?
We run experiments with the number of clusters set to 6 and to 8, with all other aspects of the pipeline kept the same.
It should be noted that this is a very na\"ive way of dealing with an unknown number of clusters.
There are methods specifically designed for identifying the right number of clusters \cite{hamerly2004learning,chazal2013persistence},
and that is what would be used if this situation arose in practice.

The results can be found in figures~\ref{fig:cmnist-2v4-partial-add} and \ref{fig:cmnist-2v4-miss-s-add}.
Bags and batches are constructed by drawing an equal number of samples from each cluster.
Unsurprisingly, the method performs worse than with the correct number of clusters.
When investigating how the clustering methods deal with the larger number of clusters,
we found that it is predominantly those samples that do not appear in the training set
which get spread out among the additional clusters.
This is most likely due to the fact that the clustering is semi-supervised,
with those clusters that occur in the training set having supervision.
The overall effect is that the samples which are not appearing in the training set are over-represented in the drawn bags,
which means it is easier for the adversary to identify where the bags came from,
and the encoder cannot properly learn to produce an invariant encoding.

\section{Adapting GEORGE}\label{adapting_g}

As discussed in the main text,
GEORGE \cite{SohDunAngGuetal20} was originally developed to address the uneven performance resulting from hidden stratification, though hidden stratification of a different kind to the one we consider.
In \cite{SohDunAngGuetal20} the training set comes with (super-)class labels but without subclass (or \emph{subgroup} in our terminology) labels.
The training set is unlabeled with respect to the subclass, but all superclass-subclass combinations (or ``sources'') are assumed to be present in the training data and therefore discoverable via clustering.
(Note that the clustering in \cite{SohDunAngGuetal20} is -- in contrast to our method -- completely without supervision and there is nothing to guide the clustering towards discovering the subgroups of interest, apart from the assumption that they are the most salient.)
On the other hand, in our setting, we do have access to all sources expected at deployment time, but not all of them are present in the training data -- some are exclusively found in the \emph{unlabeled} deployment set.

This necessitates propagating the labels from the training set to the deployment set, which can be done within the clustering step to ensure consistency between the cluster labels and the propagated superclass labels.
Doing so requires us to modify the clustering algorithm such that instead of predicting each source independently of one another, we factorize the joint distribution of the super- and subclasses, $P(Y, S)$ into their respective marginal distributions, $P(Y)$ and $P(S)$.
In practice, this is achieved by applying two separate cluster-prediction heads to the image representation, $z$: one, $\mu_y$, predicting the superclass, $y$, the other, $\mu_s$ predicting the subclass, $s$.
This allows us to decouple the supervised loss for the two types of label and to always be able to recover $y$ due to having full supervision in terms of its ground-truth labels from the training set -- this means we can identify all the $y$ clusters with the right $y$ labels.
This is not necessarily possible for $s$, because some $s$ values might be completely missing from the labeled training set (missing subgroup setting).

With the outputs structured as just described, we can obtain the prediction for a given sample's source (which is needed to compute the unsupervised clustering loss and for balancing the deployment set), by taking the argmax of the vectorized outer product of the softmaxed outputs of the two heads:
\begin{align}
&\omega_i = \argmax\limits_{k} \: \mathrm{vec}(\mu_y(z_i) \otimes \mu_s(z_i))_k\;,\\
&\quad\quad\quad\quad\quad\quad\quad\quad\quad\quad k = 1, ..., |S \times Y| \nonumber
\end{align}
where $\mu_y(z_i)$ and $\mu_s(z_i)$ are vectors, $\otimes$ is the outer product, and $\mathrm{vec}(A)$ is the vectorization of matrix $A$.
After training the clustering model, we can then use it to generate predictions $\hat{Y}^{dep}$, as well as the cluster labels $\hat{\Omega}^{dep}$, for the deployment set, and use them together to train a robust classifier with gDRO \cite{sagawa2019distributionally}, as in \cite{SohDunAngGuetal20}.

\section{Code}
The code can be found here: \url{https://github.com/wearepal/missing-sources}.

\end{document}